\documentclass{article}

\usepackage{etoolbox}
\usepackage{comment}

\newtoggle{neurips}
\togglefalse{neurips}
\newcommand{\neurips}[1]{\iftoggle{neurips}{#1}{}}
\newcommand{\arxiv}[1]{\iftoggle{neurips}{}{#1}}

\neurips{
\usepackage[nonatbib]{neurips_2022}
}

\newcommand{\loose}{\neurips{\looseness=-1}}

\neurips{
  \usepackage[numbers]{natbib}
  \bibliographystyle{abbrvnat}
}
\arxiv{
\usepackage{natbib}
\bibliographystyle{plainnat}
\bibpunct{(}{)}{;}{a}{,}{,}
}

\usepackage[utf8]{inputenc} %
\usepackage[T1]{fontenc}    %
\usepackage[colorlinks,citecolor=blue!70!black,linkcolor=blue!70!black,
urlcolor=blue!70!black,breaklinks=true]{hyperref}
\usepackage{url}            %
\usepackage{booktabs}       %
\usepackage{amsfonts}       %
\usepackage{nicefrac}       %
\usepackage{microtype}      %
\usepackage[dvipsnames]{xcolor}         %

\usepackage{enumitem}
\usepackage{breakcites}

\usepackage{transparent}

\usepackage{caption} 
\usepackage{algorithm,algorithmicx}
\usepackage[noend]{algpseudocode}      %
\usepackage{amsthm}

\usepackage{mathtools}
\usepackage{amsmath}
\usepackage{bbm}
\usepackage{amsfonts}
\usepackage{amssymb}
\usepackage{nicefrac}
\usepackage{comment}

\usepackage{thmtools}
\usepackage{thm-restate}
\declaretheorem[name=Theorem,parent=section]{theorem}
\declaretheorem[name=Lemma,parent=section]{lemma}
\declaretheorem[name=Assumption, parent=section]{assumption}
\declaretheorem[name=Condition, parent=section]{condition}
\declaretheorem[qed=$\triangleleft$,name=Example,style=definition, parent=section]{example}
\declaretheorem[name=Remark,style=definition, parent=section]{remark}
\declaretheorem[name=Proposition, parent=section]{proposition}

\usepackage{mathrsfs}
\usepackage{graphicx}
\usepackage{xspace}
\usepackage{setspace}

\neurips{\usepackage[scaled=.88]{helvet}}
\arxiv{\usepackage[scaled=.88]{helvet}}
\usepackage{inconsolata}

\arxiv{
\usepackage[letterpaper, left=1in, right=1in, top=1in, bottom=1in]{geometry}
\usepackage{parskip}
\PassOptionsToPackage{hypertexnames=false}{hyperref}  %
}

\usepackage{microtype}
\usepackage{hhline}

\makeatletter
\newcommand{\neutralize}[1]{\expandafter\let\csname c@#1\endcsname\count@}
\makeatother

\newenvironment{thmmod}[2]
  {\renewcommand{\thetheorem}{\ref*{#1}#2}%
   \neutralize{theorem}\phantomsection
   \begin{theorem}}
  {\end{theorem}}

\usepackage{algorithm}

\usepackage{amsthm}
\usepackage{mathtools}
\usepackage{amsmath}
\usepackage{bbm}
\usepackage{amsfonts}
\usepackage{amssymb}

\usepackage{xpatch}

\theoremstyle{definition}  %

\newtheorem{corollary}{Corollary}[section]

\theoremstyle{plain}
\newtheorem{definition}{Definition}[section]

\xpatchcmd{\proof}{\itshape}{\normalfont\proofnameformat}{}{}
\newcommand{\proofnameformat}{\bfseries}

\usepackage[nameinlink,capitalize]{cleveref}

\newcommand{\pref}[1]{\cref{#1}}
\newcommand{\pfref}[1]{Proof of \preftitle{#1}}

\crefformat{equation}{#2(#1)#3}
\Crefformat{equation}{#2(#1)#3}

\Crefformat{figure}{#2Figure #1#3}
\Crefformat{assumption}{#2Assumption #1#3}

\usepackage{crossreftools}
\newcommand{\preftitle}[1]{\crtcref{#1}}

\usepackage{xparse}

\DeclarePairedDelimiter{\abs}{\lvert}{\rvert} %
\DeclarePairedDelimiter{\brk}{[}{]}
\DeclarePairedDelimiter{\crl}{\{}{\}}
\DeclarePairedDelimiter{\prn}{(}{)}
\DeclarePairedDelimiter{\nrm}{\|}{\|}
\DeclarePairedDelimiter{\tri}{\langle}{\rangle}

\DeclareMathOperator{\En}{\mathbb{E}}

\DeclareMathOperator*{\argmin}{arg\,min} %
\DeclareMathOperator*{\argmax}{arg\,max}

\newcommand{\mb}[1]{\boldsymbol{#1}}
\newcommand{\wt}[1]{\widetilde{#1}}
\newcommand{\wh}[1]{\widehat{#1}}
\newcommand{\wb}[1]{\widebar{#1}}

\def\ddefloop#1{\ifx\ddefloop#1\else\ddef{#1}\expandafter\ddefloop\fi}
\def\ddef#1{\expandafter\def\csname bb#1\endcsname{\ensuremath{\mathbb{#1}}}}
\ddefloop ABCDEFGHIJKLMNOPQRSTUVWXYZ\ddefloop
\def\ddefloop#1{\ifx\ddefloop#1\else\ddef{#1}\expandafter\ddefloop\fi}
\def\ddef#1{\expandafter\def\csname b#1\endcsname{\ensuremath{\mathbf{#1}}}}
\ddefloop ABCDEFGHIJKLMNOPQRSTUVWXYZ\ddefloop
\def\ddef#1{\expandafter\def\csname sf#1\endcsname{\ensuremath{\mathsf{#1}}}}
\ddefloop ABCDEFGHIJKLMNOPQRSTUVWXYZ\ddefloop
\def\ddef#1{\expandafter\def\csname c#1\endcsname{\ensuremath{\mathcal{#1}}}}
\ddefloop ABCDEFGHIJKLMNOPQRSTUVWXYZ\ddefloop
\def\ddef#1{\expandafter\def\csname h#1\endcsname{\ensuremath{\widehat{#1}}}}
\ddefloop ABCDEFGHIJKLMNOPQRSTUVWXYZ\ddefloop
\def\ddef#1{\expandafter\def\csname hc#1\endcsname{\ensuremath{\widehat{\mathcal{#1}}}}}
\ddefloop ABCDEFGHIJKLMNOPQRSTUVWXYZ\ddefloop
\def\ddef#1{\expandafter\def\csname t#1\endcsname{\ensuremath{\widetilde{#1}}}}
\ddefloop ABCDEFGHIJKLMNOPQRSTUVWXYZ\ddefloop
\def\ddef#1{\expandafter\def\csname tc#1\endcsname{\ensuremath{\widetilde{\mathcal{#1}}}}}
\ddefloop ABCDEFGHIJKLMNOPQRSTUVWXYZ\ddefloop
\def\ddefloop#1{\ifx\ddefloop#1\else\ddef{#1}\expandafter\ddefloop\fi}
\def\ddef#1{\expandafter\def\csname scr#1\endcsname{\ensuremath{\mathscr{#1}}}}
\ddefloop ABCDEFGHIJKLMNOPQRSTUVWXYZ\ddefloop

\newcommand{\ind}{\mathbbm{1}}    %

\newcommand{\veps}{\varepsilon}

\newcommand{\ldef}{\vcentcolon=}

\newcommand{\dlike}{divergence-like\xspace}

\newcommand{\subc}{sub-Chebychev\xspace}

\newcommand{\irtextp}{parameterized Information Ratio\xspace}

\newcommand{\gammau}{\gamma_{\mathrm{u}}}
\newcommand{\gammal}{\gamma_{\mathrm{\ell}}}

\newcommand{\Mnot}{M_0}
\newcommand{\fmnot}{f\sups{\Mnot}}

\newcommand{\Mbarnu}{\Mbar_{\nu}}%
\newcommand{\Mbarkap}{\Mbar_{\kappa}}%

\newcommand{\expthree}{\textsf{EXP3}\xspace}

\newcommand{\irtext}{Information Ratio\xspace}

\newcommand{\irgen}[1][D]{\mathsf{inf}^{#1}_{\gamma}}
\newcommand{\irgens}[1][D]{\mathsf{inf}^{#1}}

\newcommand{\irHels}{\irgens[\mathsf{H}]}

\newcommand{\mainalg}{\textsf{ExO}$^{+}$}

\newcommand{\exotext}{Exploration-by-Optimization\xspace}

\newcommand{\eostar}{\mathsf{exo}}
\newcommand{\exo}{\mathsf{exo}}
\newcommand{\exostar}{\mathsf{exo}}

\newcommand{\exoval}{\Gamma}

\newcommand{\fmt}{f\sups{M\ind{t}}}
  
\newcommand{\ppr}{\mu_{\mathrm{pr}}}
\newcommand{\ppo}{\mu_{\mathrm{po}}}
\newcommand{\mupr}{\mu_{\mathrm{pr}}}
\newcommand{\mupo}{\mu_{\mathrm{po}}}

\newcommand{\var}{\mathbb{V}}
\newcommand{\Var}{\var}
\newcommand{\varplus}{\mathbb{V}_{+}}
\newcommand{\Varplus}{\varplus}

\newcommand{\Varm}[1][M]{\Var\sups{#1}}

\renewcommand{\pm}[1][M]{p_{\sss{#1}}}
\newcommand{\Ct}{C(T)}

\newcommand{\gmbar}{g\sups{\Mbar}}

\newcommand{\pmbar}{p_{\sMbar}}

\newcommand{\MinimaxReg}{\mathfrak{M}(\cM,T)}

\renewcommand{\emptyset}{\varnothing}

\newcommand{\NullObs}{\crl{\emptyset}}

\newcommand{\filt}{\mathscr{F}}

\newcommand{\hist}{\mathcal{H}}

\newcommand{\Asig}{\mathscr{P}}
\newcommand{\Rsig}{\mathscr{R}}
\newcommand{\Osig}{\mathscr{O}}

\newcommand{\Hspace}{\Omega}
\newcommand{\Hsig}{\filt}

\newcommand{\abscont}{V(\cM)}

\newcommand{\vepslowg}{\veps_{\gamma}}

\newcommand{\Framework}{Decision Making with Structured Observations\xspace}
\newcommand{\FrameworkShort}{DMSO\xspace}

\newcommand{\learner}{learner\xspace}

\newcommand{\act}{\pi}
\newcommand{\Act}{\Pi}

\newcommand{\obs}{o}
\newcommand{\Obs}{\mathcal{\cO}}
\newcommand{\ObsSpace}{\mathcal{\cO}}

\newcommand{\RewardSpace}{\cR}

\newcommand{\Rspace}{\RewardSpace}

\newcommand{\compbasic}{\mathsf{dec}}
\newcommand{\comp}[1][\gamma]{\mathsf{dec}_{#1}}

\newcommand{\comploc}[2][\gamma]{\mathsf{dec}_{#1,#2}}

\newcommand{\CompText}{Decision-Estimation Coefficient\xspace}
\newcommand{\CompAbbrev}{DEC\xspace}
\newcommand{\CompShort}{\CompAbbrev}

\newcommand{\compgen}[1][D]{\comp^{#1}}

\newcommand{\M}[1]{^{{\scriptscriptstyle M}}}  %

\newcommand{\sMbar}{\sss{\Mbar}}
 
\newcommand{\sups}[1]{^{{\scriptscriptstyle#1}}}
\newcommand{\subs}[1]{_{{\scriptscriptstyle#1}}}

\newcommand{\sss}[1]{{\scriptscriptstyle#1}}

\newcommand{\Enm}[1][M]{\En^{\sss{#1}}}
\newcommand{\Empi}[1][M]{\En^{\sss{#1},\pi}}
\newcommand{\Enmbar}[1][\Mbar]{\En^{\sss{#1}}}

\newcommand{\bbPm}[1][M]{\bbP\sups{#1}}
\newcommand{\bbPmbar}[1][\Mbar]{\bbP\sups{#1}}

\newcommand{\fm}[1][M]{f\sups{#1}}
\newcommand{\pim}[1][M]{\pi_{\sss{#1}}}

\newcommand{\gm}{g\sups{M}}

\newcommand{\cFm}{\cF_{\cM}}

\newcommand{\cMf}{\cM_{\cF}}

\newcommand{\fmbar}{f\sups{\Mbar}}
\newcommand{\pimbar}{\pi\subs{\Mbar}}

\newcommand{\fmstar}{f\sups{\Mstar}}
\newcommand{\pimstar}{\pi\subs{\Mstar}}

\newcommand{\fstar}{f^{\star}}
\newcommand{\pistar}{\pi^{\star}}

\newcommand{\cMloc}[1][\veps]{\cM_{#1}}

\newcommand{\Mbar}{\wb{M}}

\newcommand{\PiNS}{\Pi_{\mathrm{NS}}} %
\newcommand{\Reg}{\mathrm{\mathbf{Reg}}}

\newcommand{\RegDM}{\Reg_{\mathsf{DM}}}

\newcommand{\Mstar}{M^{\star}}

\newcommand{\algcommentlight}[1]{\textcolor{blue!70!black}{\transparent{0.5}\small{\texttt{\textbf{//\hspace{2pt}#1}}}}}

\newcommand{\midsem}{\,;}

\newcommand{\approxleq}{\lesssim}
\newcommand{\approxgeq}{\gtrsim}

\newcommand{\fhat}{\wh{f}}

\renewcommand{\ind}[1]{^{{\scriptscriptstyle(#1)}}}

\newcommand{\bigoh}{O}
\newcommand{\bigoht}{\wt{O}}
\newcommand{\bigom}{\Omega}
\newcommand{\bigomt}{\wt{\Omega}}
\newcommand{\bigthetat}{\wt{\Theta}}

\newcommand{\indic}{\mathbb{I}}

\newcommand{\poly}{\mathrm{poly}}
\newcommand{\polylog}{\mathrm{polylog}}

\newcommand{\Dbreg}[2]{D_{\cR}\prn*{#1\,\|\,#2}}

\newcommand{\Dkl}[2]{D_{\mathsf{KL}}\prn*{#1\,\|\,#2}}

\newcommand{\Dhel}[2]{D_{\mathsf{H}}\prn*{#1,#2}}
\newcommand{\Df}[2]{D_{f}\prn*{#1\dmid{}#2}}
\newcommand{\Dgen}[2]{D\prn*{#1\dmid{}#2}}

\newcommand{\Dhels}[2]{D^{2}_{\mathsf{H}}\prn*{#1,#2}}

\newcommand{\Dchis}[2]{D_{\chi^2}\prn*{#1\dmid{}#2}}

\newcommand{\Dtv}[2]{D_{\mathsf{TV}}\prn*{#1,#2}}

\newcommand{\DhelsX}[3]{D^{2}_{\mathsf{H}}\prn[#1]{#2,#3}}

\newcommand{\Ber}{\mathrm{Ber}}

\newcommand{\dmid}{\;\|\;}

\newcommand{\conv}{\mathrm{co}}
\newcommand{\diam}{\mathrm{diam}}

\newcommand{\unif}{\mathrm{unif}}

\newcommand{\mathand}{\quad\text{and}\quad}

\def\multiset#1#2{\ensuremath{\left(\kern-.3em\left(\genfrac{}{}{0pt}{}{#1}{#2}\right)\kern-.3em\right)}}

\newcommand{\grad}{\nabla}

\renewcommand{\emptyset}{\varnothing}

\newcommand{\phat}{\wh{p}}

 \let\underbar\undefined
\makeatletter
\let\save@mathaccent\mathaccent
\newcommand*\if@single[3]{%
  \setbox0\hbox{${\mathaccent"0362{#1}}^H$}%
  \setbox2\hbox{${\mathaccent"0362{\kern0pt#1}}^H$}%
  \ifdim\ht0=\ht2 #3\else #2\fi
  }
\newcommand*\rel@kern[1]{\kern#1\dimexpr\macc@kerna}
\newcommand*\widebar[1]{\@ifnextchar^{{\wide@bar{#1}{0}}}{\wide@bar{#1}{1}}}
\newcommand*\underbar[1]{\@ifnextchar_{{\under@bar{#1}{0}}}{\under@bar{#1}{1}}}
\newcommand*\wide@bar[2]{\if@single{#1}{\wide@bar@{#1}{#2}{1}}{\wide@bar@{#1}{#2}{2}}}
\newcommand*\under@bar[2]{\if@single{#1}{\under@bar@{#1}{#2}{1}}{\under@bar@{#1}{#2}{2}}}
\newcommand*\wide@bar@[3]{%
  \begingroup
  \def\mathaccent##1##2{%
    \let\mathaccent\save@mathaccent
    \if#32 \let\macc@nucleus\first@char \fi
    \setbox\z@\hbox{$\macc@style{\macc@nucleus}_{}$}%
    \setbox\tw@\hbox{$\macc@style{\macc@nucleus}{}_{}$}%
    \dimen@\wd\tw@
    \advance\dimen@-\wd\z@
    \divide\dimen@ 3
    \@tempdima\wd\tw@
    \advance\@tempdima-\scriptspace
    \divide\@tempdima 10
    \advance\dimen@-\@tempdima
    \ifdim\dimen@>\z@ \dimen@0pt\fi
    \rel@kern{0.6}\kern-\dimen@
    \if#31
      \overline{\rel@kern{-0.6}\kern\dimen@\macc@nucleus\rel@kern{0.4}\kern\dimen@}%
      \advance\dimen@0.4\dimexpr\macc@kerna
      \let\final@kern#2%
      \ifdim\dimen@<\z@ \let\final@kern1\fi
      \if\final@kern1 \kern-\dimen@\fi
    \else
      \overline{\rel@kern{-0.6}\kern\dimen@#1}%
    \fi
  }%
  \macc@depth\@ne
  \let\math@bgroup\@empty \let\math@egroup\macc@set@skewchar
  \mathsurround\z@ \frozen@everymath{\mathgroup\macc@group\relax}%
  \macc@set@skewchar\relax
  \let\mathaccentV\macc@nested@a
  \if#31
    \macc@nested@a\relax111{#1}%
  \else
    \def\gobble@till@marker##1\endmarker{}%
    \futurelet\first@char\gobble@till@marker#1\endmarker
    \ifcat\noexpand\first@char A\else
      \def\first@char{}%
    \fi
    \macc@nested@a\relax111{\first@char}%
  \fi
  \endgroup
}
\newcommand*\under@bar@[3]{%
  \begingroup
  \def\mathaccent##1##2{%
    \let\mathaccent\save@mathaccent
    \if#32 \let\macc@nucleus\first@char \fi
    \setbox\z@\hbox{$\macc@style{\macc@nucleus}_{}$}%
    \setbox\tw@\hbox{$\macc@style{\macc@nucleus}{}_{}$}%
    \dimen@\wd\tw@
    \advance\dimen@-\wd\z@
    \divide\dimen@ 3
    \@tempdima\wd\tw@
    \advance\@tempdima-\scriptspace
    \divide\@tempdima 10
    \advance\dimen@-\@tempdima
    \ifdim\dimen@>\z@ \dimen@0pt\fi
    \rel@kern{0.6}\kern-\dimen@
    \if#31
      \underline{\rel@kern{-0.6}\kern\dimen@\macc@nucleus\rel@kern{0.4}\kern\dimen@}%
      \advance\dimen@0.4\dimexpr\macc@kerna
      \let\final@kern#2%
      \ifdim\dimen@<\z@ \let\final@kern1\fi
      \if\final@kern1 \kern-\dimen@\fi
    \else
      \underline{\rel@kern{-0.6}\kern\dimen@#1}%
    \fi
  }%
  \macc@depth\@ne
  \let\math@bgroup\@empty \let\math@egroup\macc@set@skewchar
  \mathsurround\z@ \frozen@everymath{\mathgroup\macc@group\relax}%
  \macc@set@skewchar\relax
  \let\mathaccentV\macc@nested@a
  \if#31
    \macc@nested@a\relax111{#1}%
  \else
    \def\gobble@till@marker##1\endmarker{}%
    \futurelet\first@char\gobble@till@marker#1\endmarker
    \ifcat\noexpand\first@char A\else
      \def\first@char{}%
    \fi
    \macc@nested@a\relax111{\first@char}%
  \fi
  \endgroup
}
\makeatother

\usepackage{accents}

\makeatletter
  \renewenvironment{proof}[1][Proof]%
  {%
   \par\noindent{\bfseries\upshape {#1.}\ }%
  }%
  {\qed}
\makeatother

\makeatletter
\let\OldStatex\Statex
\renewcommand{\Statex}[1][3]{%
  \setlength\@tempdima{\algorithmicindent}%
  \OldStatex\hskip\dimexpr#1\@tempdima\relax}
\makeatother

\let\oldparagraph\paragraph
\arxiv{\renewcommand{\paragraph}[1]{\oldparagraph{#1.}}}
\neurips{\renewcommand{\paragraph}[1]{\textbf{#1.}~~}}

\addtocontents{toc}{\protect\setcounter{tocdepth}{0}}

\newcommand{\citebasic}[1]{{\protect\NoHyper\citet{#1}\protect\endNoHyper}}

\title{On the Complexity of Adversarial Decision Making}

\arxiv{
\date{}
  \author{%
Dylan J. Foster\\%
{\small\texttt{dylanfoster@microsoft.com}}
\and
Alexander Rakhlin\\%
{\small\texttt{rakhlin@mit.edu}}
\and
Ayush Sekhari\\%
{\small\texttt{as3663@cornell.edu}}
\and
Karthik Sridharan\\%
{\small\texttt{ks999@cornell.edu}}
}
}

\neurips{
\author{%
}
}

\begin{document}

\maketitle

\begin{abstract}

A central problem in online learning and decision making---from
bandits to reinforcement learning---is to understand what
modeling assumptions lead to sample-efficient learning
guarantees.
We consider a general \emph{adversarial decision
  making} framework that encompasses (structured) bandit
problems with adversarial rewards and reinforcement learning
problems with adversarial dynamics. Our main result is to show---via new upper and lower bounds---that the \CompText, a complexity measure introduced by \citebasic{foster2021statistical} in the stochastic
counterpart to our setting, is necessary and sufficient to obtain low
regret for adversarial decision making. However, compared to the stochastic
setting, one must apply the \CompText to the
\emph{convex hull} of the class of models (or, hypotheses) under
consideration. This establishes that the price of accommodating
adversarial rewards or dynamics is governed by the behavior of the
model class under convexification, and recovers a number of existing
results---both positive and negative. 
En route to obtaining these guarantees, we provide new structural results that
connect the \CompText to variants of other well-known complexity measures,
including the
Information Ratio of \citebasic{russo2018learning} and the
Exploration-by-Optimization objective of \citebasic{lattimore2021mirror}.

 \end{abstract}

\section{Introduction}
\label{sec:intro}

To reliably deploy data-driven decision making methods in
real-world systems where safety is critical, such methods should satisfy two desiderata: (i) provable
robustness in the face of dynamic or even adversarial environments,
and (ii) ability to effectively take
advantage of problem structure as modeled by the practitioner. In high-dimensional problems, this entails efficiently generalizing across states and actions while delicately exploring new decisions.

For decision making in static, stochastic environments, recent years have seen
extensive investigation into optimal sample complexity and algorithm
design principles, and the foundations are beginning to take shape. With an emphasis on reinforcement learning, a burgeoning body
of research identifies specific modeling assumptions under which sample-efficient
interactive decision making is possible
\citep{dean2020sample,yang2019sample,jin2020provably,modi2020sample,ayoub2020model,krishnamurthy2016pac,du2019latent,li2009unifying,dong2019provably,zhou2021nearly},
as well as general structural conditions that aim to unify these
assumptions
\citep{russo2013eluder,jiang2017contextual,sun2019model,wang2020provably,du2021bilinear,jin2021bellman,foster2021statistical}. For
dynamic or adversarial settings, however, comparatively little is
known outside of (i) positive results for special cases such as
adversarial bandit problems
\citep{auer2002non,audibert2009minimax,hazan2011better,dani2007price,abernethy2008competing,bubeck2012towards,kleinberg2004nearly,flaxman2005online,bubeck2017kernel,lattimore2020improved},
and (ii) a handful of negative results suggesting that online reinforcement learning in
agnostic or adversarial settings can actually be statistically
intractable \citep{sekhari2021agnostic,liu2022learning}. These developments
raise the following questions: (a) what are the underlying phenomena
that govern the statistical complexity of decision making in
adversarial settings? (b) what are the corresponding algorithmic
design principles that attain optimal statistical complexity?

  \paragraph{Contributions}
  We consider an adversarial variant of the \emph{\Framework}
(\FrameworkShort) framework introduced in
\citet{foster2021statistical}, where a learner or decision-maker
interacts with a sequence of \emph{models} (reward distributions in the
case of bandits, or MDPs in the case of reinforcement learning) chosen by an adaptive adversary, and aims to minimize regret
against the best decision in hindsight. Models are assumed to
belong to a known \emph{model class},
which reflects the learner's prior knowledge
about the problem. 
The main question we investigate is: \emph{How does the structure of the model class determine the
    minimax regret for adversarial decision making?}
We show:\loose
\begin{enumerate}
  \item For \emph{any} model class, one can obtain high-probability
    regret bounds
    that scale with a \emph{convexified} version of the \emph{\CompText}
    (\CompShort), a complexity measure introduced by \citet{foster2021statistical}.\loose
  \item For any algorithm with ``reasonable'' tail behavior, the optimal regret for adversarial decision making
    is lower bounded by (a suitably localized version of) the convexified \CompShort.
  \end{enumerate}
  In the process of obtaining these results, we draw new connections to several existing
  complexity measures.

  \subsection{Problem Setting}
We adopt an adversarial variant of the %
\FrameworkShort 
framework of \citet{foster2021statistical} consisting of $T$
rounds, where at each round $t=1,\ldots,T$:
\begin{enumerate}
\item The \learner selects a \emph{decision} $\act\ind{t}\in\Act$,
  where $\Act$ is the \emph{decision space}.
  \item Nature selects a \emph{model} $M\ind{t}\in\cM$, where $\cM$ is
    a \emph{model class}.
  \item The learner receives a reward $r\ind{t}\in\cR\subseteq\bbR$
    and observation $o\ind{t}\in\cO$ sampled via
    $(r\ind{t},o\ind{t})\sim{}M\ind{t}(\pi\ind{t})$, where
    $\cO$ is the
    \emph{observation space}. We abbreviate
    $z\ind{t}\ldef{}(r\ind{t},o\ind{t})$ and $\cZ\ldef\cR\times\cO$.
  \end{enumerate}
  Here, each model $M=M(\cdot,\cdot\mid\cdot)\in\cM$ is a conditional distribution
  $M:\Pi\to\Delta(\cR\times\cO)$ that maps the learner's decision to a
  distribution over rewards and observations. This setting subsumes
  (adversarial) bandit problems, where models
  correspond to reward functions (or distributions), as well as adversarial reinforcement
  learning, where models correspond to Markov decision
  processes (MDPs). In both cases, the model class
  $\cM$ encodes prior knowledge about the decision making problem, such
  as structure of rewards or dynamics (e.g., linearity or
  convexity). The model class might be parameterized by linear models, neural
  networks, or other rich function approximators depending on the problem domain.\loose

  For a model $M\in\cM$, $\Empi[M]\brk*{\cdot}$ denotes expectation
  under the process $(r,\obs)\sim{}M(\pi)$. We define
  $\fm(\pi)\ldef{}\Empi[M]\brk*{r}$ as the mean reward function and $\pim\ldef{}\argmax_{\act\in\Act}\fm(\act)$ as the decision
  with greatest reward for $M$. We let $\cFm=\crl*{\fm\mid{}M\in\cM}$
  denote the induced class of reward functions. We measure performance via
  \emph{regret} to the best fixed decision in hindsight:\footnote{The
    results in this paper immediately extend to the regret
    $\sup_{\pistar\in\Pi}\sum_{t=1}^{T}r\ind{t}(\pistar)-r\ind{t}(\pi\ind{t})$
    through standard tail bounds.}
  \begin{equation}
    \label{eq:regret}
    \RegDM\ldef\sup_{\pistar\in\Pi}\sum_{t=1}^{T}\En_{\pi\ind{t}\sim{}p\ind{t}}\brk*{\fmt(\pistar)-\fmt(\pi\ind{t})}.
  \end{equation}
  This formulation---in which models are selected by a
  potentially adaptive adversary---generalizes \citet{foster2021statistical}, who
  considered a \emph{stochastic} setting where $M\ind{t}=\Mstar$ is fixed across
  all rounds. Examples include:
\begin{itemize}
  \item \textbf{Adversarial bandits.}
    With no observations ($\ObsSpace=\crl*{\emptyset}$), the
    adversarial \FrameworkShort framework
    is equivalent to the \emph{adversarial bandit} problem with
    structured rewards. In this context, $\act\ind{t}$ is typically referred to as an
\emph{action} or \emph{arm} and $\Act$
is referred to as the \emph{action space}. The most basic
example here is the adversarial finite-armed bandit problem with $A$ actions
\citep{auer2002non,audibert2009minimax,hazan2011better}, where $\Act=\crl{1,\ldots,A}$ and
$\cFm=\bbR^{A}$. Other well-studied examples include
adversarial linear bandits \citep{dani2007price,abernethy2008competing,bubeck2012towards},
    bandit convex optimization
    \citep{kleinberg2004nearly,flaxman2005online,bubeck2017kernel,lattimore2020improved},
    and nonparametric bandits
    \citep{kleinberg2004nearly,bubeck2011x,magureanu2014lipschitz}.\footnote{Typically,
      these examples are formulated with deterministic rewards,
      which we encompass by restricting models in $\cM$ to be
      deterministic. Our formulation is more general and allows for, e.g.,
      semi-stochastic adversaries.}
  \item \textbf{Reinforcement learning.}     The adversarial
    \FrameworkShort framework encompasses finite-horizon, episodic
    online reinforcement learning, with each round $t$ corresponding to
    a single episode: $\pi\ind{t}$ is a \emph{policy} (a mapping from
    state to actions) to play in the 
    episode, $r\ind{t}$ is the cumulative reward in the episode, and
    the observation $o\ind{t}$ is the episode's trajectory (sequence
    of observed states, actions, and rewards). Online reinforcement learning in the stochastic setting where
    $M\ind{t}=\Mstar$ is fixed has received extensive attention
    \citep{jiang2017contextual,sun2019model,jin2020provably,wang2020provably,du2021bilinear,jin2021bellman,foster2021statistical},
    but the adversarial setting we study has received less
    investigation. Examples include the adversarial MDP problem where an adversary
    chooses a sequence of tabular MDPs, which
    is known to be intractable \citep{liu2022learning}, and the easier
    problem in which there is a fixed (known) MDP but rewards are adversarial \citep{neu2010online,zimin2013online,neu2014online,jin2020simultaneously}.
    See \pref{sec:examples} for more details.
\end{itemize}
We refer to \pref{app:prelim} for additional measure-theoretic details and
background, and to
  \citet{foster2021statistical} for further examples and detailed
  discussion.\footnote{We mention in passing that the upper bounds in this paper encompass the more general setting where rewards are
  not observed by the learner (i.e., $z\ind{t}$ does not contain the reward), thus subsuming the partial monitoring problem. Our lower bounds, however,
  require that rewards are observed. See \pref{sec:related}.}

Understanding statistical complexity (i.e., minimax regret) for the \FrameworkShort setting at this level of generality is a
challenging problem. Even if one restricts only to
bandit-type problems with no observations, any complexity measure
must capture the role of structural assumptions such as convexity or
smoothness in determining the optimal rates. To go beyond bandit
problems and handle the general setting, one must accommodate problems with rich,
structured feedback such as reinforcement learning, where
observations (as well as subtle features of the noise distribution) can reveal information about the underlying model.\loose

\subsection{Overview of Results}
For a model class $\cM$, reference model $\Mbar\in\cM$, and scale
parameter $\gamma>0$, the \CompText \citep{foster2021statistical} is
defined via
\begin{align}
  \label{eq:dec}
  \comp(\cM,\Mbar)=
  \inf_{p\in\Delta(\Pi)}\sup_{M\in\cM}\En_{\pi\sim{}p}\brk*{
  \fm(\pim) - \fm(\pi)
  - \gamma\cdot\Dhels{M(\pi)}{\Mbar(\pi)}
  },
\end{align}
where we recall that for probability measures $\bbP$ and $\bbQ$ with a common
dominating measure $\nu$, (squared) Hellinger distance is given by \neurips{  $\Dhels{\bbP}{\bbQ}=\int\prn[\big]{\sqrt{\nicefrac{d\bbP}{d\nu}}-\sqrt{\nicefrac{d\bbQ}{d\nu}}}^{2}$.}
\arxiv{\begin{equation}
  \label{eq:hellinger}
  \Dhels{\bbP}{\bbQ}=\int\prn*{\sqrt{\frac{d\bbP}{d\nu}}-\sqrt{\frac{d\bbQ}{d\nu}}}^{2}.
\end{equation}}
We define $\comp(\cM)=\sup_{\Mbar\in\cM}\comp(\cM,\Mbar)$, and let
$\conv(\cM)$ denote the convex hull of $\cM$, which can be viewed as the set of all mixtures of models
in $\cM$. Our main results show that the \emph{convexified \CompText},
\[\comp(\conv(\cM)),\] leads to
upper and lower bounds on the optimal regret for adversarial decision making.
\newtheorem*{thm:informal1}{Theorem (informal)}
\begin{thm:informal1}
  For any model class $\cM$, \pref{alg:main} ensures that with high probability,\loose
    \begin{align}
      \label{eq:upper_informal}
      \RegDM \approxleq \comp(\conv(\cM))\cdot{}T,
    \end{align}
    where $\gamma$ satisfies the balance $\comp(\conv(\cM))\propto
    \frac{\gamma}{T}\log\abs{\Pi}$. Moreover, for any algorithm with
    ``reasonable'' tail behavior (\pref{sec:lower}), regret must scale
    with a localized version of the same quantity. 

    As a consequence, there exists an
  algorithm for which $\En\brk*{\RegDM} \leq \tilde{o}(T)$ if and only
  if $\comp(\conv(\cM))\propto \gamma^{-\rho}$ for some $\rho>0$.
\end{thm:informal1}

For the stochastic version of our setting, \citet{foster2021statistical}
give upper and lower bounds that scale with $\comp(\cM)$, without convexifying (under
appropriate technical assumptions;
cf. \pref{sec:learnability}). Hence, our results show that in general,
the gap in optimal regret for stochastic and
adversarial decision making (or, ``price of adversarial outcomes'') is governed by the behavior of the \CompShort under
convexification. For example, multi-armed bandits, linear bandits, and
convex bandits correspond to convex model classes (where $\conv(\cM)=\cM$), which gives a
post-hoc explanation for why these problems are tractable in the
adversarial setting. Finite state/action Markov decision processes do
not correspond to a convex model class, and have $\comp(\conv(\cM))$ exponentially large compared to
$\comp(\cM)$; in this case, our results recover lower bounds of \citet{liu2022learning}.

Beyond these results, we prove that the convexified \CompText is equivalent to:\loose
\begin{enumerate}[topsep=0pt]
\item a ``parameterized'' variant of the generalized \irtext of \citet{lattimore2021mirror}.
\item a novel high-probability variant of the \emph{\exotext} %
  objective of \citet{lattimore2020exploration,lattimore2021mirror}.\loose
\end{enumerate}

\paragraph{Our techniques}
On the lower bound side, we strengthen the approach
from \citet{foster2021statistical} with an improved change-of-measure
argument (leading to improved results even in the
stochastic setting), and combine this with the simple idea of constructing
adversaries based on static mixture models. On the upper bound side, we extend the
powerful \exotext machinery of \citet{lattimore2021mirror} to the
\FrameworkShort setting, and give a novel high-probability variant of
the technique which leads to regret bounds for adaptive adversaries. We show that the performance of this method is
controlled by a complexity measure whose value is equivalent to the
convexified \CompShort, as well as parameterized variant of the
\irtext (we present results in terms of the former to draw
comparison to the stochastic setting).\loose

Overall, our results heavily draw on the work of \citet{foster2021statistical} and
  \citet{lattimore2021mirror}, but we believe they play a valuable role in
  bridging these lines of research and formalizing connections.

\arxiv{\subsection{Organization}}
\neurips{\paragraph{Organization}}
\neurips{
\pref{sec:main} presents our main results, including upper and lower
bounds on regret and a characterization of learnability. In \pref{sec:connections}, we provide new structural results
 connecting the \CompShort to \exotext and the \irtext. We close
 with \arxiv{discussion of }future directions (\pref{sec:discussion}). Additional
 comparison to related work is deferred to \pref{sec:related}.
 The appendix also contains proofs and additional results, including examples (\pref{sec:examples}) and
 further structural results (\pref{app:structural}).\loose
}
\arxiv{
\pref{sec:main} presents our main results, including upper and lower
bounds on regret and a characterization of learnability. In \pref{sec:connections}, we provide new structural results
connecting the \CompShort to \exotext and the \irtext. Examples are
given in \pref{sec:examples}, and additional related work is discussed
in \pref{sec:related}. We close with \arxiv{discussion of }future
directions (\pref{sec:discussion}). Proofs and additional
structural results (\pref{app:structural}) are deferred to the appendix.\loose
 }
 
 \neurips{\vspace{-3pt}}

\arxiv{
\paragraph{Additional notation} For a set $\cX$, we let
        $\Delta(\cX)$ denote the set of all Radon probability measures
        over $\cX$. We let $\conv(\cX)$ denote the set of all finitely
        supported convex combinations of elements in $\cX$.

    We adopt non-asymptotic big-oh notation: For functions
	$f,g:\cX\to\bbR_{+}$, we write $f=\bigoh(g)$ (resp. $f=\bigom(g)$) if there exists a constant
	$C>0$ such that $f(x)\leq{}Cg(x)$ (resp. $f(x)\geq{}Cg(x)$)
        for all $x\in\cX$. We write $f=\bigoht(g)$ if
        $f=\bigoh(g\cdot\mathrm{polylog}(T))$, $f=\bigomt(g)$ if $f=\bigom(g/\polylog(T))$, and
        $f=\bigthetat(g)$ if $f=\bigoht(g)$ and $f=\bigomt(g)$. %
	We write $f\propto g$ if $f=\bigthetat(g)$.
        }

\section{Main Results}
\label{sec:main}

We now present our main results. First, using a new high-probability
variant of the \exotext technique \citep{lattimore2020exploration,lattimore2021mirror}, we provide an upper bound on regret
 based on the (convexified) \CompText (\pref{sec:upper}). Next, we present a lower
 bound that scales with a localized version of the same quantity
 (\pref{sec:lower}). Finally, we use
 these results to give a characterization for learnability
 (\pref{sec:learnability}), and discuss the gap between stochastic and adversarial decision making.\loose

To keep presentation as simple as possible, we make the following assumption.
\begin{assumption}
  The decision space $\Pi$ has $\abs{\Pi}<\infty$, and we have $\cR=\brk*{0,1}$.
\end{assumption}\neurips{\vspace{-5pt}}
This assumption only serves to facilitate the use of the minimax
theorem, and we expect that our results can be generalized (e.g., with covering numbers as in Section
3.4 of \citet{foster2021statistical}).\loose

\subsection{Upper Bound}
\label{sec:upper}

In this section we give regret bounds for adversarial decision making based on
the (convexified) \CompText. A-priori, it is not obvious
why the \CompShort should bear any relevance to the adversarial
setting we consider: The
algorithms and regret bounds based on the \CompShort that
\citet{foster2021statistical} introduce for the
stochastic setting heavily rely on the ability to estimate a static
underlying model, yet in the adversarial setting, the learner may only
interact with each model a single time. This renders
any sort of global estimation (e.g., for dynamics of an MDP) impossible. In spite of this difficulty, we show that
regret bounds can be achieved by building on
the 
\emph{\exotext} technique of
\citet{lattimore2020exploration,lattimore2021mirror}, which provides an
elegant approach to estimating rewards that exploits the structure of the model
class under consideration.

\exotext---introduced by \citet{lattimore2020exploration} and
substantially expanded in \citet{lattimore2021mirror}---can be thought of as a generalization of the classical \expthree
algorithm \citep{auer2002non} for finite-action bandits, which applies the exponential weights
method for full-information online learning to a sequence of unbiased importance-weighted
estimators for rewards. While \expthree is near-optimal for bandits,
it is unsuitable for general model classes because the reward
estimators the algorithm uses do not exploit the structure of the
decision space. Consequently, the regret scales linearly with
$\abs{\Pi}$ rather than with, e.g., dimension, as one might hope for
problems like linear bandits.
The idea behind \exotext is to solve an optimization
problem at each round to search for a (potentially biased) reward estimator and modified sampling
distribution that better exploit the structure of the model class
$\cM$, leading to information sharing and improved regret. \citet{lattimore2021mirror} showed
that for a general partial monitoring setting
(cf. \pref{sec:related}), the expected regret for this method---and for more general family of algorithms based on Bregman divergences---is bounded by a generalization of
the \irtext of \citet{russo2014learning,russo2018learning}.\loose

Our development builds on that of \citet{lattimore2021mirror}, but we
pursue \emph{high-probability} guarantees rather than in-expectation
guarantees. This allows us to provide
regret bounds that hold for \emph{adaptive adversaries}, rather than
oblivious adversaries as considered in prior work.\footnote{In general, in-expectation regret bounds do not
  imply high-probability bounds. For example, in adversarial bandits,
  the \expthree algorithm can experience linear regret with constant
  probability \citep{lattimore2020bandit}.} Beyond this basic
motivation, our interest in high-probability guarantees
comes from the lower bound in the sequel (\pref{sec:lower}), which shows
that the convexified \CompText lower bounds regret for algorithms
with ``reasonable'' tail behavior. To develop high-probability regret
bounds and complement this lower bound, we use a novel variant of the \exotext objective and a specialized analysis that goes beyond
the Bregman divergence framework.\loose
\begin{algorithm}[tp]
    \setstretch{1.3}
     \begin{algorithmic}[1]
       \State \textbf{parameters}: Learning rate $\eta>0$.
  \For{$t=1, 2, \cdots, T$}
  \State Define $q\ind{t}\in\Delta(\Pi)$
  via exponential weights update:%
    \neurips{\begin{equation}\textstyle q\ind{t}(\pi)=
        \exp\prn*{\eta\sum_{i=1}^{t-1}\fhat\ind{i}(\pi)}\text{\Large$/$}\sum_{\pi'\in\Pi}\exp\prn*{\eta\sum_{i=1}^{t-1}\fhat\ind{i}(\pi')}.
  \end{equation}}\arxiv{\begin{equation}q\ind{t}(\pi)=
    \frac{\exp\prn*{\eta\sum_{i=1}^{t-1}\fhat\ind{i}(\pi)}}{\sum_{\pi'\in\Pi}\exp\prn*{\eta\sum_{i=1}^{t-1}\fhat\ind{i}(\pi')}}.
  \end{equation}}%
\label{line:exponential_weights}
\vspace{-15pt}
  \State Solve \emph{high-probability exploration-by-optimization}
  objective:\hfill\algcommentlight{See Eq. \pref{eq:exo_val}}
  \begin{equation}
    (p\ind{t},g\ind{t})
    \gets
    \argmin_{p\in\Delta(\Pi),g\in\cG}\sup_{M\in\cM,\pistar\in\Pi}\exoval_{q\ind{t},\eta}(p,g\midsem
    \pistar,M).
  \end{equation}
  \label{line:exo}
  \neurips{\vspace{-10pt}}
\State{}Sample decision $\act\ind{t}\sim{}p\ind{t}$ and observe
$z\ind{t}=(r\ind{t},o\ind{t})$.\label{line:sample}
\State{}Form reward estimator:
\begin{equation}
\fhat\ind{t}(\pi) = \frac{g\ind{t}(\pi;
  \pi\ind{t},z\ind{t})}{p\ind{t}(\pi\ind{t})}.\label{eq:importance_weighting}
\end{equation}\neurips{\vspace{-10pt}}
\label{line:estimator}
\EndFor
\end{algorithmic}
\caption{High-Probability Exploration-by-Optimization (\mainalg)}
\label{alg:main}
\end{algorithm}

Our algorithm, \mainalg, is displayed in \pref{alg:main}. At each
round $t$, the algorithm computes a \emph{reference distribution}
$q\ind{t}\in\Delta(\Pi)$ by applying the standard exponential weights
update (with learning rate $\eta>0$) to a sequence of reward estimators
$\fhat\ind{1},\ldots,\fhat\ind{t-1}$ from previous rounds (\pref{line:exponential_weights}). For
the main step (\pref{line:exo}), the algorithm obtains a \emph{sampling distribution}
$p\ind{t}\in\Delta(\Pi)$ and an \emph{estimation function}
$g\ind{t}\in\cG\ldef(\Pi\times\Pi\times\cZ\to\bbR)$ by solving a
minimax optimization problem based on a new objective we term \emph{high-probability
  exploration-by-optmization}: Defining
\begin{align}
  \exoval_{q,\eta}(p,g\midsem \pistar,M) 
  &\ldef{} \En_{\pi\sim{}p}\brk*{\fm(\pistar)-\fm(\pi)} \label{eq:exo_val}  \\ 
  & \qquad \qquad + \frac{1}{\eta} \cdot
    \En_{\pi\sim{}p,z\sim{}M(\pi)}\En_{\pi'\sim{}q}\brk*{
    \exp\prn*{
  \frac{\eta}{p(\pi)}\prn*{
  g(\pi';\pi,z)
  - g(\pistar;\pi,z) 
  }
    }
    -1
  },\notag
\end{align}
we solve
\begin{equation}
  \label{eq:exo_minimax}
    (p\ind{t},g\ind{t})
    \gets
    \argmin_{p\in\Delta(\Pi),g\in\cG}\sup_{M\in\cM,\pistar\in\Pi}\exoval_{q\ind{t},\eta}(p,g\midsem
    \pistar,M).
\end{equation}
Finally (\pref{line:sample,line:estimator}), \arxiv{given $p\ind{t}$ and $g\ind{t}$, }the algorithm samples
$\pi\ind{t}\sim{}p\ind{t}$, observes $z\ind{t}=(r\ind{t},o\ind{t})$,
and then forms an importance-weighted reward estimator via
$\fhat\ind{t}(\pi)\ldef{}g\ind{t}(\pi; \pi\ind{t},z\ind{t})\text{\large$/$}p\ind{t}(\pi\ind{t})$.\loose

The interpretation of the high-probability \exotext objective \pref{eq:exo_val} is
as follows: For a given round $t$, the model $M\in\cM$ and decision $\pistar\in\Pi$ should be
thought of as a proxy for the true model $M\ind{t}$ and
optimal decision, respectively. By solving the
minimax problem in \pref{eq:exo_minimax}, the min-player aims to---in
the face of an unknown, worst-case model---find a sampling distribution
that minimizes instantaneous regret, yet ensures good tail behavior
for the importance-weighted estimator $g(\cdot;\pi,z)/p(\pi)$. Tail
behavior is captured by the moment generating function-like
term in \pref{eq:exo_val}, which penalizes the learner for
over-estimating rewards under the reference distribution $q$ or under-estimating rewards under $\pistar$.\loose

We show that this approach leads to a bound on
regret that scales with the convexified \CompShort.
  \begin{restatable}[Main upper bound]{theorem}{uppermain}
    \label{thm:upper_main}
    For any choice of $\eta>0$, \pref{alg:main} ensures that for all
    $\delta>0$, with probability at
    least $1-\delta$,
    \begin{align}
      \label{eq:upper_main}
      \RegDM
      \leq{} \comp[1/8\eta](\conv(\cM))\cdot{}T +  \frac{2}{\eta} \cdot\log(\abs{\Pi}/\delta).
    \end{align}
    In particular, for any $\delta>0$, with appropriate $\eta$, the algorithm ensures that with probability at least $1-\delta$,
    \begin{align}
            \RegDM
      \leq{} \bigoh(1)\cdot\inf_{\gamma>0}\crl*{\comp[\gamma](\conv(\cM))\cdot{}T + \gamma\cdot\log(\abs{\Pi}/\delta)}.       \label{eq:upper_main2}
    \end{align}%
  \end{restatable}
  This regret bound holds for arbitrary, potentially adaptive
  adversaries. The result should be compared to the upper bound for the stochastic setting in
\citet{foster2021statistical} (e.g., Theorem 3.3), which takes a
similar form, but scales with the weaker quantity
$\sup_{\Mbar\in\conv(\cM)}\comp(\cM,\Mbar)$.\footnote{If a proper
  estimation algorithm (i.e., an algorithm producing estimators that lie in $\cM$) is available, \citet{foster2021statistical} (\arxiv{Theorem}\neurips{Thm.} 4.1) gives tighter bounds scaling with
  $\comp(\cM)$.} See \pref{sec:related} for comparison to \citet{lattimore2020exploration,lattimore2021mirror}.  

\paragraph{Equivalence of \exotext and \CompText}
We now discuss a deeper connection between \exotext and the
\CompShort. Define the minimax value of the high-probability \exotext objective via
\begin{align}
  \exo_{\eta}(\cM,q)\ldef\inf_{p\in\Delta(\Pi),g\in\cG}\sup_{M\in\cM,\pistar\in\Pi}  \exoval_{q,\eta}(p,g\midsem \pistar,M),
\end{align}
and let $\exo_{\eta}(\cM) \ldef
\sup_{q\in\Delta(\Pi)}\exo_{\eta}(\cM,q)$. This quantity can be
interpreted as a complexity measure for $\cM$ whose, value reflects the
difficulty of exploration. The following structural result (\pref{thm:equivalence} in \pref{sec:connections}), which is
critical to the proof of \pref{thm:upper_main}, shows that this
complexity measure is equivalent to the convexified \CompText:
\begin{equation}
  \label{eq:exo_dec_teaser}
\comp[(4\eta)^{-1}](\conv(\cM))\leq\eostar_{\eta}(\cM)\leq\comp[(8\eta)^{-1}](\conv(\cM)),\quad\forall{}\eta>0.
\end{equation}
As we show, the regret of \pref{alg:main} is controlled by the value of
$\exo_{\eta}(\cM)$, and thus \pref{thm:upper_main} follows. In the process of proving
  \pref{eq:exo_dec_teaser}, we also establish equivalence of the
  \exotext objective and a \emph{parameterized} version of the
  \irtext, which is of independent interest (cf. \pref{sec:connections}).
Both results build on, but go beyond the Bregman divergence-based framework in
\citet{lattimore2021mirror}, and exploit a somewhat obscure
connection between Hellinger distance and the moment generating function (MGF) for the logarithmic
loss. In particular, we use a technical lemma (proven in
\pref{app:technical}), which shows that up to constants, the Hellinger
distance between two probability distributions can be expressed as
variational problem based on the associated MGFs.
  \begin{restatable}{lemma}{hellingerexp}
    \label{lem:hellinger_exp}
    Let $\bbP$ and $\bbQ$ be probability distributions over a measurable space
    $(\cX,\filt)$. Then
    \begin{align}
      \label{eq:hellinger_exp}
      \frac{1}{2}\Dhels{\bbP}{\bbQ}  \leq{} \sup_{g:\cX\to\bbR}\crl*{1-\En_{\bbP}\brk[\big]{e^{g}}\cdot\En_{\bbQ}\brk[\big]{e^{-g}}}
      \leq{} \Dhels{\bbP}{\bbQ}.
    \end{align}\neurips{\vspace{-15pt}}
  \end{restatable}
  The lower inequality in \pref{lem:hellinger_exp} is proven using a
  trick similar to one used by \citet{zhang2006from} to prove
  high-probability bounds for maximum likelihood estimation based on
  Hellinger distance. To prove the upper bound in
  \pref{eq:exo_dec_teaser}, we apply the lower inequality in \pref{eq:hellinger_exp} with the
  test function $g$ taking the role of the estimation function in the
  \exotext objective.

  \paragraph{Further remarks}
  The main focus of this work is statistical complexity (in
  particular, minimax regret), and the
  runtime and memory requirements of \pref{alg:main}, which are
  linear in $\abs{\Pi}$, are not practical for large
  decision spaces. Improving the
  computational efficiency is an interesting question for future work.
    We mention in passing that \pref{thm:upper_main} answers a question
  raised by \citet{foster2021statistical} of obtaining in the frequentist
  setting a regret bound matching
  the Bayesian regret bound in their Theorem 3.6.

\subsection{Lower Bound} 
\label{sec:lower}

We now complement the regret bound in the prequel with a lower bound based on the convexified
\CompShort. Our most general result shows that for any algorithm,
either the expected regret or its (one-sided) second moment must
scale with a localized version of the convexified \CompShort. %

To state the result, we define the \emph{localized model
  class} around a model $\Mbar$ via
\[
  \cMloc[\veps](\Mbar) = \crl*{
      M\in\cM: \fmbar(\pimbar) \geq{} \fm(\pim) - \veps
    },
  \]
  and define
  $\comploc{\veps}(\cM)\ldef\sup_{\Mbar\in\cM}\comp(\cM_{\veps}(\Mbar),\Mbar)$
  as the \emph{localized \CompText}. We let
  $(x)_{+}\ldef\max\crl*{x,0}$ and define
$\abscont\ldef\sup_{M,M'\in\cM}\sup_{\act\in\Act}\sup_{A\in{}\Rsig\otimes\Osig}\crl[\big]{\tfrac{M(A\mid\act)}{M'(A\mid{}\act)}}\vee{}e$;\footnote{Recall
    (\pref{app:prelim}) that $M(\cdot,\cdot\mid{}\pi)$ is the
    conditional distribution given $\pi$.} finiteness of $\abscont$ is not
    necessary, but removes a $\log(T)$ factor from \pref{thm:lower_main}.
    \loose
  \begin{restatable}[Main lower bound]{theorem}{lowermain}
    \label{thm:lower_main}
          Let $\Ct \ldef c\cdot\log(T\wedge{}V(\cM))$ for a
          sufficiently large numerical constant $c>0$. Set
          $\vepslowg\ldef\tfrac{\gamma}{4\Ct{}T}$.
          For any algorithm, there exists an oblivious adversary for which\loose
          \begin{align}
            \label{eq:lower_main}
            \En\brk*{\RegDM} + \sqrt{\En\prn{\RegDM}_{+}^2} \geq\bigom(1)\cdot\sup_{\gamma>\sqrt{2\Ct{}T}}\comploc{\vepslowg}(\conv(\cM))\cdot{}T-\bigoh(T^{1/2}).
  \end{align}%
\end{restatable}
\pref{thm:lower_main} implies that for any algorithm (such as \pref{alg:main})
with tail behavior beyond what is granted by
control of the first moment, the
regret in \pref{thm:upper_main} cannot be substantially improved. In more detail, consider
the notion of a \emph{sub-Chebychev} algorithm.
\begin{definition}[Sub-Chebychev Algorithm]
  \label{def:subc}
  A regret minimization algorithm is said to be sub-Chebychev with
  parameter $R$ if for all $t>0$,
  \begin{align}
    \arxiv{\bbP((\RegDM)_{+} \geq{} t) \leq \frac{R^{2}}{t^2}.}
        \neurips{\bbP((\RegDM)_{+} \geq{} t) \leq R^{2}/t^2.}
  \end{align}
\end{definition}
For \subc algorithms, both the mean and (root) second moment of
regret are bounded by the parameter $R$ (cf. \pref{sec:subc}), which
has the following consequence.
\begin{restatable}{corollary}{subcregret}
  \label{cor:subc}
  Any regret minimization algorithm with sub-Chebychev parameter $R>0$ must have
  \begin{equation}
    R
    \geq\bigomt(1)\cdot\sup_{\gamma>\sqrt{2\Ct{}T}}\comploc{\vepslowg}(\conv(\cM))\cdot{}T
    - \bigoh(T^{1/2}).     \label{eq:subc}
  \end{equation}
\end{restatable}
To interpret this result, suppose for simplicity that
$\comp(\conv(\cM))$ and $\comploc{\vepslowg}(\conv(\cM))$ are
continuous with respect to $\gamma>0$, and that $\comploc{\vepslowg}(\conv(\cM))\approxgeq\gamma^{-1}$, which is satisfied
for non-trivial classes.\footnote{The dominant term
  $\comploc{\vepslowg}(\conv(\cM))\cdot{}T$ in \pref{eq:lower_main}
  scales with $T^{1/2}$ for any class that is non-trivial in the sense
  that it embeds the
  two-armed bandit problem, so that the $-\bigoh(T^{1/2})$ term
  can be discarded.} In this
case, it follows from \pref{thm:upper_main}
(cf. \pref{prop:subc_high_prob} for details) that by setting $\delta=1/T^2$, \pref{alg:main} is sub-Chebychev with
parameter %
\begin{equation}
  R = \bigoht\prn[\Big]{
    \inf_{\gamma>0}\crl*{\comp[\gamma](\conv(\cM))\cdot{}T + \gamma\cdot\log(\abs{\Pi})}
  } = \bigoht\prn*{
    \comp[\gammau](\conv(\cM))\cdot{}T
  },
  \label{eq:subc_upper}
\end{equation}
where $\gammau$ satisfies the balance
$\comp[\gammau](\conv(\cM))\propto\frac{\gammau}{T}\log\abs{\Pi}$. On
the other hand, the lower bound in \pref{eq:subc} can be
shown to scale with
\begin{equation}
R \geq{}  \bigomt\prn*{
  \compbasic_{\gammal,\veps_{\gammal}}(\conv(\cM))\cdot{}T
  },\label{eq:subc_lower}
\end{equation}
where $\gammal$ satisfies the balance
$\compbasic_{\gammal,\veps_{\gammal}}(\conv(\cM))\propto\frac{\gammal}{T}$. We
conclude that the upper bound from \pref{thm:upper_main} cannot be
improved beyond (i) localization and (ii) dependence on
$\log\abs{\Pi}$.

As an example, we show in \pref{app:examples} that for the
multi-armed bandit problem with $\Pi=\crl*{1,\ldots,A}$, the upper
bound in \pref{eq:subc_upper} yields $R=\bigoht(\sqrt{AT\log{}A})$,
while the lower bound in \pref{eq:subc_lower} yields
$R=\bigom(\sqrt{AT})$. See \pref{sec:examples} for additional
examples which further illustrate the scaling in the upper and lower bounds.\loose

The dependence on $\log\abs{\Pi}$ cannot
be removed from the upper bound or made to appear in the lower bound
in general (cf. Section 3.5 of \citet{foster2021statistical}). As shown in \citet{foster2021statistical}, localization is
inconsequential for most model classes commonly studied in
the literature. The same is true for the examples we consider here
(\pref{sec:examples}), where \pref{thm:lower_main} leads to the
correct rate up to small polynomial factors. However, improving the
upper bound to achieve localization, which
\citet{foster2021statistical} show is possible in the stochastic
setting, is an interesting future direction. 

\neurips{See \pref{sec:related} for further discussion and for comparison to a related
lower bound in \citet{lattimore2022minimax}.\loose}

\arxiv{\oldparagraph{Why convexity?}}
\neurips{\textbf{Why convexity?}~~}
At this point, a natural question is \emph{why} the convex hull
$\conv(\cM)$ plays a fundamental role in the adversarial setting. For the lower bound, the intuition is simple: Given a model class
$\cM$, the adversary can pick any mixture distribution
$\mu\in\Delta(\cM)$, then choose the sequence of models
$M\ind{1},\ldots,M\ind{T}$ by sampling $M\ind{t}\sim\mu$ independently
at each round. This is equivalent to playing a static mixture model
$\Mstar=\En_{M\sim\mu}\brk*{M}\in\conv(\cM)$, which is what allows us
to prove a lower bound based on the \CompShort for the set $\conv(\cM)$ of all such
models. In view of the fact that the lower bound is obtained
through this static (and stochastic) adversary, we believe the more surprising result
here is that good behavior of the convexified \CompShort is also
\emph{sufficient} for low regret for fully adversarial decision making.

\subsection{Learnability and Comparison to Stochastic Setting}
\label{sec:learnability}

Building on the upper and lower bounds in the prequel, we give a
characterization for \emph{learnability} (i.e., when non-trivial
regret is possible) for adversarial decision making. This extends the
learnability characterization for the stochastic setting in
\citet{foster2021statistical}, 
and follows a long tradition in learning
theory
\citep{vapnik1995nature,alon1997scale,shalev2010learnability,rakhlin2010online,daniely2011multiclass}. To
state the result, we define the minimax regret for model class $\cM$ as
\newcommand{\bp}{\mb{p}}%
\renewcommand{\bM}{\mb{M}}%
\[
\MinimaxReg = \inf_{\bp\ind{1},\ldots,\bp\ind{T}}\sup_{\bM\ind{1},\ldots,\bM\ind{T}}\En\brk*{\RegDM},
\]
where $\bp\ind{t}:(\Pi\times\cZ)^{t-1}\to\Delta(\Pi)$ and $\bM\ind{t}:
(\Pi\times\cZ)^{t-1}\to\cM$ are policies for the learner and
adversary, respectively. Our characterization is as follows.
\begin{restatable}{theorem}{learnability}%
  \label{thm:learnability}
  Suppose there exists $\Mnot\in\cM$ such that $\fmnot$ is a constant
  function, and that $\abs{\Pi}<\infty$.
  \begin{enumerate}
  \item If there exists $\rho>0$ \arxiv{such that}\neurips{s.t.}
    $\lim_{\gamma\to\infty} \comp(\conv(\cM))\cdot\gamma^{\rho}=0$, then
    $\lim_{T\to\infty}\frac{\MinimaxReg}{T^{p}}=0$ for \arxiv{some} $p<1$.\loose
  \item If $\lim_{\gamma\to\infty}\comp(\conv(\cM))\cdot\gamma^{\rho}>0$ for
    all $\rho>0$, then 
    $\lim_{T\to\infty}\frac{\MinimaxReg}{T^{p}}=\infty$ for all
    $p < 1$.\loose
  \end{enumerate}
  \arxiv{In addition, the}\neurips{The} same conclusion holds when $\Pi=\Pi_{T}$ grows
  with $T$, but has $\log\abs{\Pi_T}=\bigoh(T^{q})$ for any
  $q<1$.\footnote{Allowing $\Pi$ to grow with $T$ is useful when
    considering infinite decision spaces, because it facilitates covering arguments.}
\end{restatable}
\pref{thm:learnability} shows that polynomial decay of the convexified \CompShort
is necessary and sufficient for low regret. We emphasize that this
result is complementary to \pref{thm:lower_main}, and does not require localization or
any assumption on the tail behavior of the algorithm. This is a
consequence of the coarse, asymptotic nature of the result, which
allows us the use of rescaling arguments to remove these conditions.\loose

\paragraph{Comparison to stochastic setting}
Having shown that the convexified \CompText leads to upper and lower
bounds on the optimal regret for the adversarial \FrameworkShort
setting, we now contrast with the %
stochastic setting. There,
\citet{foster2021statistical} obtain upper bounds on regret that have
the same form as \pref{eq:upper_main2}, but scale with the weaker
quantity $\max_{\Mbar\in\conv(\cM)}\comp(\cM,\Mbar)$.\footnote{Theorem 3.1 of \citet{foster2021statistical} attains $\RegDM
\approxleq{}
\inf_{\gamma>0}\crl[\big]{\max_{\Mbar\in\conv(\cM)}\comp(\cM,\Mbar) +
  \gamma\cdot\log\abs{\cM}}$ with high probability.} For classes that are not convex, but
where ``proper'' estimators are available (including finite-state/action MDPs), the
upper bounds in \citet{foster2021statistical} can further be improved
to scale with $\comp(\cM)$. Hence, our results show that in general,
the price of adversarial outcomes can be as large as
$\comp(\conv(\cM))/\comp(\cM)$. Examples \neurips{(see
  \pref{sec:examples} for details and more) }include:
\begin{itemize}
\item For tabular (finite-state/action) MDPs with horizon $H$, $S$ states, and $A$ actions,
  \citet{foster2021statistical} show that
  $\comp(\cM)\leq{}\mathrm{poly}(H,S,A)/\gamma$, and use this to obtain
  regret $\sqrt{\poly(H,S,A)\cdot{}T}$. Tabular MDPs are \emph{not} a
  convex class, and $\conv(\cM)$ is equivalent to the class of so-called
  \emph{latent MDPs}, which are known to be intractable
  \citep{kwon2021rl,liu2022learning}. Indeed, we show
  (\pref{sec:examples}) that
  $\comp(\conv(\cM))\geq\bigom(A^{\min\crl{S,H}})$, which implies an
  exponential lower bound on regret through \pref{thm:lower_main}. This
  highlights that in general, the gap between stochastic and
  adversarial outcomes can be quite large.
\item For many common bandit problems, one has $\conv(\cM)=\cM$,
  leading to polynomial bounds on regret in the adversarial
  setting. For example, the multi-armed bandit problem with $A$ actions
  has $\comp(\conv(\cM))\leq\bigoh(A/\gamma)$, leading to
  $\sqrt{AT\log{}A}$ regret (via \pref{thm:upper_main}), and the
  linear bandit problem in $d$ dimensions has
  $\comp(\conv(\cM))\leq\bigoh(d/\gamma)$, leading to regret $\sqrt{dT\log\abs{\Pi}}$.
\end{itemize}
\arxiv{We refer to \pref{sec:examples} for further details and examples.}

\section{Connections Between Complexity Measures}
\label{sec:connections}
The \CompText bears a resemblance to the \emph{generalized
  Information Ratio} introduced by \citet{lattimore2021mirror,lattimore2021minimax} which
extends the original Information Ratio of
\citet{russo2014learning,russo2018learning}. In what follows, we
establish deeper connections between these complexity measures. All of
the results in this section are proven in \pref{app:structural}.

Let us call any function $\Dgen{\cdot}{\cdot}:\Delta(\Pi)\times\Delta(\Pi)\to\bbR^{+}$ a
\emph{divergence-like function}. We restate the generalized Information Ratio from
\citet{lattimore2022minimax}. For a distribution $\mu\in\Delta(\cM\times\Pi)$ and
decision distribution $p\in\Delta(\Pi)$, let $\bbP$
  be the law of the process
  $(M,\pistar)\sim\mu,\pi\sim{}p,z\sim{}M(\pi)$, and define $\ppr(\pi')\ldef\bbP(\pistar=\pi')$ and
  $\ppo(\pi';\pi,z)\ldef\bbP(\pistar=\pi'\mid{}(\pi,z))$. The distribution $\mupr$ should be
  thought of as the prior over $\pistar$, and $\mupo$ should be
  thought of as the posterior over $\pistar$ after observing $(z,\pi)$; note that
    the law $\ppo$ does not depend on the distribution $p$. For
  parameter $\lambda>1$, \citet{lattimore2022minimax} defines the generalized \irtext for a class $\cM$ via\footnote{\citet{lattimore2021mirror} give
  a slightly different but essentially equivalent definition; cf. \pref{app:structural}.}
  \begin{align}
    \label{eq:info_ratio_psi}
\Psi_{\lambda}(\cM) = \sup_{\mu\in\Delta(\cM\times\Pi)}\inf_{p\in\Delta(\Pi)}\crl*{
    \frac{(\En_{(M,\pistar)\sim\mu}\En_{\pi\sim{}p}\brk*{\fm(\pistar)-\fm(\pi)})^{\lambda}}
    {\En_{\pi\sim{}p}\En_{z\mid{}\pi}\brk*{\Dgen{\ppo(\cdot;\pi,z)}{\ppr}}}
    }.
  \end{align}
  Here, we have slightly generalized the original definition in
  \citet{lattimore2022minimax} by incorporating models in $\cM$ rather than placing
an arbitrary prior over observations $z$ directly. We also use a
general divergence-like function, while \citet{lattimore2022minimax}
uses KL divergence and \citet{lattimore2021mirror} use Bregman divergences.\loose

To understand the connection to the \CompText, it will be helpful to introduce another
variant of the \irtext, which we call the \emph{\irtextp}.
\begin{definition}
  \label{def:info_ratio_p}
  \arxiv{For a divergence-like function
  $\Dgen{\cdot}{\cdot}:\Delta(\Pi)\times\Delta(\Pi)\to\bbR_{+}$, the
  \emph{\irtextp} is given by}
\neurips{For a divergence
  $\Dgen{\cdot}{\cdot}$, the \emph{\irtextp} is given by}
  \begin{align}
    &\hspace{- 1cm} \irgens_{\gamma}(\cM)\label{eq:info_ratio_p}\\
    &= \sup_{\mu\in\Delta(\cM\times\Pi)}\inf_{p\in\Delta(\Pi)}\En_{\pi\sim{}p}\brk*{
\En_{(M,\pistar)\sim\mu}\brk*{\fm(\pistar)-\fm(\pi)}
    -\gamma\cdot\En_{\pi\sim{}p}\En_{z\mid{}\pi}\brk*{\Dgen{\ppo(\cdot;\pi,z)}{\ppr}}
    }.\notag
  \end{align}
\end{definition}
The \irtextp is always bounded by the generalized Information Ratio in 
\pref{eq:info_ratio_psi}; in particular, we have
$\irgens_{\gamma}(\cM)\leq(\Psi_{\lambda}(\cM)/\gamma)^{\frac{1}{\lambda-1}}\;\forall{}\gamma>0$. All regret bounds based on the generalized Information Ratio
that we are aware of \citep{lattimore2021mirror,lattimore2022minimax} 
implicitly bound regret by the \irtextp, and then invoke the inequality
above to move to the generalized Information Ratio. In general though, it does
not appear that these notions are equivalent. Informally, this is because the notion in \pref{eq:info_ratio_psi} is equivalent to requiring that a single
distribution $p$ certify a certain bound on the value in
\pref{eq:info_ratio_p} for all values of the parameter $\gamma$ simultaneously,
while the \irtextp allows the distribution $p$ to vary as a function
of $\gamma>0$ (hence the name); see also \pref{app:structural}.\loose

Letting $\irHels_{\gamma}(\cM)$ denote the \irtextp with $D=\Dhels{\cdot}{\cdot}$,
we show that this notion is equivalent to the convexified \CompText.\neurips{\vspace{4pt}}
\begin{restatable}{theorem}{irdecbasic}
  \label{thm:ir_dec_basic}
  For all $\gamma>0$, 
      \neurips{$
\irHels_{\gamma}(\cM)\leq\comp(\conv(\cM))\leq\irHels_{\gamma/4}(\cM)$.
}
    \arxiv{\[
\irHels_{\gamma}(\cM)\leq\comp(\conv(\cM))\leq\irHels_{\gamma/4}(\cM).
  \]}
\end{restatable} 
This result is a special case of \pref{thm:dec_info_ratio} in
\pref{app:structural}, which shows that a similar equivalence
holds for a class of ``well-behaved'' $f$-divergences that includes KL
divergence, but not necessarily for general Bregman divergences. The
idea is to use Bayes' law to move from the \CompText, which
considers distance between \emph{distributions over observations}, to the
\irtext, which considers distance between \emph{distributions over decisions}.

In light of this characterization, the results in this paper could
have equivalently been presented in terms of the \irtextp. We chose to
present them in terms of the \CompText in order to draw parallels to
the stochastic setting, where guarantees that scale with $\comp(\cM)$
(without convexification) are available. It is unclear whether the
\irtext can accurately reflect the complexity for both
stochastic and adversarial settings in the same fashion, because---unlike the
\CompShort---it is invariant under convexification, as the following
proposition shows.\footnote{The
  variants in
  \citet{lattimore2021mirror,lattimore2022minimax} are also invariant
  under convexification.}
\begin{restatable}{proposition}{inforatioconvex}
  \label{prop:info_ratio_convex}
  For any divergence-like function $\Dgen{\cdot}{\cdot}$,
  we have
  \[
    \irgens_{\gamma}(\cM) = \irgens_{\gamma}(\conv(\cM)),\quad\forall{}\gamma>0.
  \]
\end{restatable}

For a final structural result, we show that up to constants, the
\irtextp is equivalent to the high-probability \exotext objective.
\neurips{\vspace{4pt}}
  \begin{restatable} 
    {theorem}{eotodec}
    \label{thm:eo_to_dec} 
    For all $\eta>0$,
        \neurips{$
  \irHels_{\eta^{-1}}(\cM)\leq\eostar_{\eta}(\cM)\leq\irHels_{(8\eta)^{-1}}(\cM)$.
}
    \arxiv{\[
  \irHels_{\eta^{-1}}(\cM)\leq\eostar_{\eta}(\cM)\leq\irHels_{(8\eta)^{-1}}(\cM)
\]}
\end{restatable}
This result is proven through a direct argument (cf. \pref{sec:upper}), and the equivalence of the
\CompShort and \exotext in \pref{eq:exo_dec_teaser} is
proven by combining\arxiv{ this result} with
\pref{thm:ir_dec_basic}.\neurips{ Summarizing the equivalence:\loose}
\arxiv{The following corollary summarizes the equivalence.}
\begin{corollary}
  \label{thm:equivalence}
  For all $\eta>0$, \[
\comp[(4\eta)^{-1}](\conv(\cM))\leq\irHels_{\eta^{-1}}(\cM)\leq\eostar_{\eta}(\cM)\leq\irHels_{(8\eta)^{-1}}(\cM)\leq\comp[(8\eta)^{-1}](\conv(\cM)).\]
\end{corollary}
Since this equivalence depends of the \arxiv{specific} value of the
parameter $\gamma>0$ in the \irtextp, it is not clear whether a similar equivalence can be
established using the generalized Information Ratio in 
\pref{eq:info_ratio_psi}. We note in passing that one can use similar arguments to lower 
bound the Bregman divergence-based \exotext objective in
\citet{lattimore2021mirror} by the \irtextp for the Bregman divergence
of interest, complementing their upper bound.\loose

\arxiv{
\section{Examples}
\label{sec:examples}
In this section, we instantiate our upper bounds for basic examples of interest.
\subsection{Structured Bandits}
We first consider adversarial (structured) bandit problems,
which correspond to the
special case of the adversarial \FrameworkShort setting in which there are no
observations (i.e., $\cO=\NullObs$). We consider three examples:
finite-armed bandits, linear bandits, and convex bandits. For each
example, we take $\cR=\brk*{0,1}$, fix a \emph{reward function class}
$\cF\subseteq(\Pi\to\brk*{0,1})$, and take
$\cMf=\crl*{M\mid{}\fm\in\cF}$ to be the induced model
class. The class $\cMf$ should be thought of as the set of all
reward distributions over $\brk*{0,1}$ with mean rewards in $\cF$.

\begin{example}[Finite-armed bandit] 
  \label{ex:mab}
In the finite-armed bandit problem, we take $\Pi=\crl*{1,\ldots,A}$ as
the decision space for $A\in\bbN$, let $\cF=\brk*{0,1}^{A}$, and
take $\cM=\cMf$ as the induced model class. For this setting, whenever
\(A \geq 2\), it holds that
\begin{align}
  \label{eq:mab}
  \comp[\gamma](\conv(\cM)) \leq \frac{A}{\gamma}\quad\forall{}\gamma>0,\mathand
  \comp[\gamma, \veps_\gamma](\conv(\cM)) \geq{} 2^{-6} \cdot \frac{A}{\gamma}\quad\forall\gamma\geq\frac{A}{3},
\end{align}
where $\veps_\gamma = \frac{A}{12 \gamma}$.
\end{example}
This result follows from \citet[Proposition 5.2 and
5.3]{foster2021statistical}, noting that \(\conv(\cM) =
\cM\). Plugging \pref{eq:mab} into \pref{thm:upper_main} yields a
$\bigoh(\sqrt{AT\log{}A})$ upper bound on regret, and plugging into \pref{thm:lower_main} gives a $\bigomt(\sqrt{AT})$ lower
bound for \subc algorithms.\footnote{For this example and \pref{ex:linear}, the lower bound on $\comp[\gamma,
  \veps_\gamma](\conv(\cM))$ in \citet{foster2021statistical} is witnessed by a subfamily
  $\cM'\subseteq\cM$ with $V(\cM')=\bigoh(1)$. As a result, we can
  take $\Ct=\bigoh(1)$ in \pref{thm:lower_main}.}

\begin{example}[Linear bandit]
  \label{ex:linear}
  In the linear bandit problem, we have $\Pi\subseteq\bbR^{d}$ and
  \[
    \cF=\crl*{f:\Pi\to\brk*{0,1}\mid{}\text{$f$ is linear}},
  \]
  and take $\cM=\cMf$ as the induced model class. For this setting, it
  holds that\footnote{The upper bound here holds for all $\Pi$, while
    the lower bound holds for a specific choice for $\Pi$.}
  \begin{align}
    \label{eq:linear_bandit}
 \comp[\gamma](\conv(\cM)) \leq
    \frac{d}{4\gamma}\quad\forall\gamma>0,\mathand
    \comp[\gamma, \veps_\gamma](\conv(\cM))\geq{}\frac{d}{12 \gamma}  \quad\forall\gamma\geq\frac{2d}{3},
  \end{align}
  where $\veps_\gamma \ldef \frac{d}{3 \gamma}$.
\end{example}

This result follows from \citet[Proposition 6.1 and
6.2]{foster2021statistical}, again noting that
\(\conv(\cM) = \cM\). Plugging \pref{eq:mab} into \pref{thm:upper_main} yields a
$\bigoh(\sqrt{dT\log{}\abs{\Pi}})$ upper bound on regret, and plugging into \pref{thm:lower_main} gives a $\bigomt(\sqrt{dT})$ lower
bound for \subc algorithms.

\begin{example}[Convex bandit]
  In the convex bandit problem, we have $\Pi\subseteq\bbR^{d}$ and
  \[
    \cF=\crl*{f:\Pi\to\brk*{0,1}\mid{}\text{$f$ is convex}},
  \]
  and take $\cM=\cMf$ as the induced model class. For this setting, it
  holds that for all $\gamma>0$, 
  \begin{align}
    \label{eq:convex_bandit}
\comp[\gamma](\conv(\cM)) \leq O \prn*{\frac{d^4}{\gamma} \cdot \text{polylog}(d, \diam(\Pi), \gamma)}. 
\end{align}
\end{example} 
This result follows from \citet[Proposition
6.3]{foster2021statistical} (which itself is a restatement of
\citet[Theorem 3]{lattimore2020bandit}), and by noting once more that \(\conv(\cM) = \cM\).

\begin{remark}
  The adversarial bandit literature \citep{auer2002non,audibert2009minimax,hazan2011better,dani2007price,abernethy2008competing,bubeck2012towards,kleinberg2004nearly,flaxman2005online,bubeck2017kernel,lattimore2020improved,kleinberg2004nearly,bubeck2011x}
  typically considers a slightly different formulation in which the adversary selects a
  deterministic reward function. This can be captured by restricting
  $\cM$ to deterministic models. It is clear that the upper bounds on
  $\comp(\conv(\cM))$ in the examples above lead to upper bounds for
  this formulation. The lower bounds in \pref{ex:mab,ex:linear} easily extend as well.
\end{remark}

\subsection{Reinforcement Learning}
We now consider an example in reinforcement learning. We begin by
recalling how to view the episodic reinforcement learning problem in the DMSO framework. 

\paragraph{Model class}
For episodic reinforcement learning, we fix a \emph{horizon} \(H\) and
let the model class \(\cM\) consist of a set of non-stationary Markov
Decision Processes (MDP). Each model \(M \in \cM\) is specified by \[M
  = \crl[\big]{\crl*{\cS_h}_{h=1}^{H+1}, \cA,
    \crl*{P\sups{M}_h}_{h=1}^H, \crl*{R_h\sups{M}}_{h=1}^H, d_1},\]
where \(\cS_h\) is the state space for layer \(h\), \(\cA\) is the
action space, \(P_h\sups{M}: \cS_h \times \cA \mapsto
\Delta(\cS_{h+1})\) is the probability transition kernel for layer
\(h\), \(R_h\sups{M}: \cS_h \times \cA \mapsto \Delta([0, 1])\) is the
reward distribution for layer \(h\) and \(d_1 \in \Delta(\cS_1)\) is
the initial state distribution. This formulation allows the reward
distribution and transition kernel to vary across models in \(\cM\),
but keeps the initial state distribution is fixed. We adopt the
convention that \(S_{H+1} = \crl{s_{H+1}}\) where \(s_{H+1}\) is a deterministic
terminal state.

For each episode, the learner selects a non-stationary
policy \(\pi = \prn*{\pi_1, \dots, \pi_H}\),  where \(\pi_h: \cS_h
\mapsto \cA\); we let \(\PiNS\) denote the set of
all such policies. For a given MDP $M\in\cM$, an episode proceeds by first sampling \(s_1 \sim d_1\), then for \(h = 1, \dots, H\): 
\begin{itemize}
\item \(a_h = \pi_h(s_h)\).
\item \(r_h \sim R\sups{M}_h(s_h, a_h)\) and \(s_{h+1} \sim P_h\sups{M}(\cdot \mid s_h, a_h)\). 
\end{itemize} 
The value of the policy \(\pi\) under $M$ is given by \(f\sups{M}(\pi) \ldef{} \En^{\sss{M}, \pi} \brk[\big]{\sum_{h=1}^H r_h}\), where \(\En^{\sss{M}, \pi} \brk*{\cdot}\) denotes expectation under the process above. 

\paragraph{Adversarial protocol} Within the adversarial \FrameworkShort framework, model classes above
give rise to the following adversarial reinforcement learning protocol. At
each time $t$, the learner plays selects a policy $\pi\in\PiNS$ and
the adversary chooses an MDP $M\ind{t} \in \cM$. The policy \(\pi
\ind{t}\) is executed in \(M\ind{t}\), resulting in a
trajectory \(\tau \ind{t} = (s_1 \ind{t}, r_1\ind{t}, r_1\ind{t}),
\dots, (s_H \ind{t}, r_H \ind{t}, r_H\ind{t})\). The learner then
observes feedback \((r\ind{t}, o\ind{t})\), where \(r\ind{t} \ldef
\sum_{h=1}^H r_H\ind{t}\) is the cumulative reward for the episode, and \(o\ind{t} = \tau\ind{t}\) is the trajectory. 

With this setting in mind, we give our main example.
\begin{example}[Tabular MDP] 
  \label{ex:tabular} 
Let \(\cM\) be the class of finite-state/action (tabular) MDPs with
horizon $H$, \(S \geq 2\) states, \(A \geq 2 \) actions, and \(\sum_{h=1}^H r_h \in [0, 1]\). Then, for any \(\gamma \geq  A^
{\min\crl{S - 1, H}}/6\),
\[\comp[\gamma, \veps_\gamma](\conv(\cM)) \geq \frac{A^
    {\min\crl{S - 1, H}}}{24 \gamma},\]
where \(\veps_\gamma \ldef {A^
{\min\crl{S - 1, H}}}/{24 \gamma}\).
\end{example}
Using this result with \pref{thm:lower_main} leads to a lower bound on regret
that scales with $\bigom(A^{\min{S-1,H}})$, which recovers existing
intractability results for this setting
\citep{kwon2021rl,liu2022learning}. Note that we have
$\comp(\cM)=\mathrm{poly}(S,A,H)/\gamma$ for this setting
\citep{foster2021statistical}, so this is a case where there is a large
separation between stochastic and adversarial decision making.

We briefly mention that the set $\conv(\cM)$ can be interpreted as the
set of \emph{latent MDPs} \citep{kwon2021rl}, which is also known to
be intractable. In the latent MDP
problem, each model is a mixture of MDPs. At the beginning of each
episode, the underlying MDP is drawn from the mixture (the identity is not
observed), and the learner's policy is executed in this MDP for the
duration of the episode.
 }

\arxiv{
\section{Related Work}
\label{sec:related}

Beyond \citet{foster2021statistical}, which was the starting point for
this work, our results build on a long line of research on partial
monitoring and the \irtext
\cite{russo2014learning,russo2018learning,lattimore2019information,lattimore2020improved,lattimore2021minimax,lattimore2021mirror,lattimore2022minimax,kirschner2021information, kirschner2020information,kirschner2021asymptotically};
most closely related are the works
  the works of
  \citet{lattimore2021mirror} and \citet{lattimore2022minimax}. Below we
  discuss and compare to these results in greater detail.

  \paragraph{Comparison to partial monitoring setting}
  \citet{lattimore2021mirror,lattimore2022minimax} and other works in
  this sequence consider a general partial monitoring setting in which
  each outcome $z\ind{t}$ is directly chosen by an adversary, and needs
  not contain the reward signal.
  In terms of reward signal, our setting is more restrictive because
    we assume that $r\ind{t}$ is observed. In fact, our upper bounds
    encompass the more general setting where rewards are not 
    observed by the learner (thus subsuming the partial monitoring
    problem) but our lower bounds that require that rewards are
    observed.
   In terms of data generation process, our setting is more general
    because we restrict to models in a known class in $\cM$. This
    setup recovers the case where $z\ind{t}$ is fully adversarial
    because by considering the special case where $\cM$ consists of point masses over $\cZ$. However, our framework also allows for semi-stochastic
    adversaries, and for problems like (structured) adversarial MDPs. For
    example, if we constrain all models in $\cM$ place $\veps$ probability mass on
    a particular outcome $z$, any adversary in our framework must place
    $\veps$ mass on this outcome as well.

  \paragraph{Upper bounds}
Our upper bounds build on the \exotext technique,
which was introduced in \citet{lattimore2020exploration} and
generalized significantly by \citet{lattimore2021mirror}. The latter
result shows that for a general family of mirror descent-based
\exotext algorithms parameterized by Bregman divergences, the regret
can be bounded by the generalized \irtext for the associated Bregman divergence (cf. \pref{app:structural}). This
approach yields bounds on pseudoregret with a similar form to
\pref{thm:upper_main} (with $\comp(\conv(\cM))$ replaced by the
generalized \irtext), but does appear to give regret bounds
for adaptive adversaries, or give high-probability guarantees.
Note that in general, in-expectation regret bounds do not imply
high-probability regret bounds (for example, even for multi-armed
bandits, the \expthree algorithm can experience linear regret with
constant probability \citep{lattimore2020bandit}). Likewise, bounds on
pseudoregret do not translate to bounds on (expected) true regret
unless one restricts to oblivious adversaries. To develop
high-probability regret bounds that complement our lower bounds, we
depart from the Bregman divergence-based framework and exploit
refined properties of Hellinger distance. We note that the work of \citet{lattimore2020exploration}
  also proposes a high-probability \exotext variant, but it is
  unclear whether this objective
  can be related to the information
  ratio or \CompText for general models.\footnote{This
    objective is based on a Bernstein-type tail bound, which scales
    with the range of the estimation functions. We avoid
    explicit dependence on the range using a more specialized tail
    bound based on \pref{lem:martingale_chernoff}.}

  \paragraph{Lower bounds}
  Our lower bounds refine the proof strategy from
  \citet{foster2021statistical}. Our most important technical result
  is a lower bound for the stochastic setting (\pref{thm:lower_general}), which improves upon Theorem 3.1 from
  \citet{foster2021statistical} by using a more refined change of
  measure argument. In particular, Theorem 3.1 of
  \citet{foster2021statistical} gives a lower bound based on the
  \CompShort that holds with low probability, and thus provides a
  meaningful converse only to algorithms with
  sub-Gaussian or sub-exponential tail behavior. Our lower bound gives a
  meaningful converse to any upper bound with sub-Chebychev tail
  behavior, which is a significantly weaker assumption. We note that
  while Theorem 3.2 of \citet{foster2021statistical} provides lower
  bounds on expected regret without assumptions on tail behavior, this result requires a stronger notion of
  localization than we consider here. \neurips{Of course, proving a lower bound on expected
    regret that matches our lower bound remains an interesting open problem.}

  Recent work of \citet{lattimore2022minimax},
  provides lower bounds on regret in a general partial
  monitoring setting based on a generalized \irtext
  (cf. \pref{app:structural}). This result is somewhat complementary
  to our lower bound (\pref{thm:lower_main}). 
  On the positive side, \citet{lattimore2022minimax} give a lower bound on \emph{expected
      regret} that is always tight in terms of dependence on $T$,
    while our result only leads to tight dependence on $T$ if one
    restricts to sub-Chebychev algorithms. In addition, the lower bound applies
  to the general partial monitoring setting, while our lower bound
  requires that rewards are observed. However, the lower bound of
  \citet{lattimore2022minimax} is loose in
    $\mathrm{poly}(\abs{\Pi})$ factors, and thus does not lead to meaningful
    dependence on problem parameters such as dimension in the same
    fashion as our results.
    An interesting question for future work is to investigate whether
    these techniques can be combined with our own to get the best of
    both worlds.

 }

\section{Discussion}
\label{sec:discussion}

We have shown that the convexified \CompText is necessary and
sufficient to achieve low regret for adversarial interactive
decision making, establishing that convexity governs the price of adversarial
outcomes. Our results elucidate
the relationship between the \CompShort, \exotext, and the \irtext, and
we hope they will find broader use.\loose

This work adds to a growing body of research which shows that online
    reinforcement learning with agnostic or adversarial outcomes can
    be statistically intractable
    \citep{sekhari2021agnostic,liu2022learning}. A promising future
    direction is to extend our techniques to natural
    semi-adversarial models in which reinforcement learning is
    tractable. Another interesting direction is to address the issue
    of computational efficiency for large decision spaces.\loose

\arxiv{
\subsection*{Acknowledgements}
We thank Zak Mhammedi for useful comments and feedback. AR
acknowledges support from the ONR through awards N00014-20-1-2336 and
N00014-20-1-2394, from the ARO through award W911NF-21-1-0328, and from the
DOE through award DE-SC0022199. Part of this work was completed
while DF was visiting the Simons Institute for the Theory of Computing. KS acknowledges support from NSF CAREER Award 1750575. 
}

\neurips{\clearpage}

\setlength{\bibsep}{6pt}
\bibliography{refs} 

\clearpage

\appendix

\neurips{
\section*{Checklist}

The checklist follows the references.  Please
read the checklist guidelines carefully for information on how to answer these
questions.  For each question, change the default \answerTODO{} to \answerYes{},
\answerNo{}, or \answerNA{}.  You are strongly encouraged to include a {\bf
justification to your answer}, either by referencing the appropriate section of
your paper or providing a brief inline description.  For example:
\begin{itemize}
  \item Did you include the license to the code and datasets? \answerYes{}
  \item Did you include the license to the code and datasets? \answerNo{The code and the data are proprietary.}
  \item Did you include the license to the code and datasets? \answerNA{}
\end{itemize}
Please do not modify the questions and only use the provided macros for your
answers.  Note that the Checklist section does not count towards the page
limit.  In your paper, please delete this instructions block and only keep the
Checklist section heading above along with the questions/answers below.

\begin{enumerate}

\item For all authors...
\begin{enumerate}
  \item Do the main claims made in the abstract and introduction accurately reflect the paper's contributions and scope?
    \answerYes{}
  \item Did you describe the limitations of your work?
    \answerYes{}
  \item Did you discuss any potential negative societal impacts of your work?
    \answerYes{}
  \item Have you read the ethics review guidelines and ensured that your paper conforms to them?
    \answerYes{}
\end{enumerate}

\item If you are including theoretical results...
\begin{enumerate}
  \item Did you state the full set of assumptions of all theoretical results?
    \answerYes{}
        \item Did you include complete proofs of all theoretical results?
    \answerYes{}
\end{enumerate}

\item If you ran experiments...
\begin{enumerate}
  \item Did you include the code, data, and instructions needed to reproduce the main experimental results (either in the supplemental material or as a URL)?
    \answerNA{}
  \item Did you specify all the training details (e.g., data splits, hyperparameters, how they were chosen)?
    \answerNA{}
        \item Did you report error bars (e.g., with respect to the random seed after running experiments multiple times)?
    \answerNA{}
        \item Did you include the total amount of compute and the type of resources used (e.g., type of GPUs, internal cluster, or cloud provider)?
    \answerNA{}
\end{enumerate}

\item If you are using existing assets (e.g., code, data, models) or curating/releasing new assets...
\begin{enumerate}
  \item If your work uses existing assets, did you cite the creators?
    \answerNA{}
  \item Did you mention the license of the assets?
    \answerNA{}
  \item Did you include any new assets either in the supplemental material or as a URL?
    \answerNA{}
  \item Did you discuss whether and how consent was obtained from people whose data you're using/curating?
    \answerNA{}
  \item Did you discuss whether the data you are using/curating contains personally identifiable information or offensive content?
    \answerNA{}
\end{enumerate}

\item If you used crowdsourcing or conducted research with human subjects...
\begin{enumerate}
  \item Did you include the full text of instructions given to participants and screenshots, if applicable?
    \answerNA{}
  \item Did you describe any potential participant risks, with links to Institutional Review Board (IRB) approvals, if applicable?
    \answerNA{}
  \item Did you include the estimated hourly wage paid to participants and the total amount spent on participant compensation?
    \answerNA{}
\end{enumerate}

\end{enumerate}

 \clearpage
}

 {
\renewcommand{\contentsname}{Contents of Appendix}
\tableofcontents
\addtocontents{toc}{\protect\setcounter{tocdepth}{2}}
}

\clearpage

\neurips{
\section{Detailed Discussion of Related Work}
\label{sec:related}

Beyond \citet{foster2021statistical}, which was the starting point for
this work, our results build on a long line of research on partial
monitoring and the \irtext
\cite{russo2014learning,russo2018learning,lattimore2019information,lattimore2020improved,lattimore2021minimax,lattimore2021mirror,lattimore2022minimax,kirschner2021information, kirschner2020information,kirschner2021asymptotically};
most closely related are the works
  the works of
  \citet{lattimore2021mirror} and \citet{lattimore2022minimax}. Below we
  discuss and compare to these results in greater detail.

  \paragraph{Comparison to partial monitoring setting}
  \citet{lattimore2021mirror,lattimore2022minimax} and other works in
  this sequence consider a general partial monitoring setting in which
  each outcome $z\ind{t}$ is directly chosen by an adversary, and needs
  not contain the reward signal.
  In terms of reward signal, our setting is more restrictive because
    we assume that $r\ind{t}$ is observed. In fact, our upper bounds
    encompass the more general setting where rewards are not 
    observed by the learner (thus subsuming the partial monitoring
    problem) but our lower bounds that require that rewards are
    observed.
   In terms of data generation process, our setting is more general
    because we restrict to models in a known class in $\cM$. This
    setup recovers the case where $z\ind{t}$ is fully adversarial
    because by considering the special case where $\cM$ consists of point masses over $\cZ$. However, our framework also allows for semi-stochastic
    adversaries, and for problems like (structured) adversarial MDPs. For
    example, if we constrain all models in $\cM$ place $\veps$ probability mass on
    a particular outcome $z$, any adversary in our framework must place
    $\veps$ mass on this outcome as well.

  \paragraph{Upper bounds}
Our upper bounds build on the \exotext technique,
which was introduced in \citet{lattimore2020exploration} and
generalized significantly by \citet{lattimore2021mirror}. The latter
result shows that for a general family of mirror descent-based
\exotext algorithms parameterized by Bregman divergences, the regret
can be bounded by the generalized \irtext for the associated Bregman divergence (cf. \pref{app:structural}). This
approach yields bounds on pseudoregret with a similar form to
\pref{thm:upper_main} (with $\comp(\conv(\cM))$ replaced by the
generalized \irtext), but does appear to give regret bounds
for adaptive adversaries, or give high-probability guarantees.
Note that in general, in-expectation regret bounds do not imply
high-probability regret bounds (for example, even for multi-armed
bandits, the \expthree algorithm can experience linear regret with
constant probability \citep{lattimore2020bandit}). Likewise, bounds on
pseudoregret do not translate to bounds on (expected) true regret
unless one restricts to oblivious adversaries. To develop
high-probability regret bounds that complement our lower bounds, we
depart from the Bregman divergence-based framework and exploit
refined properties of Hellinger distance. We note that the work of \citet{lattimore2020exploration}
  also proposes a high-probability \exotext variant, but it is
  unclear whether this objective
  can be related to the information
  ratio or \CompText for general models.\footnote{This
    objective is based on a Bernstein-type tail bound, which scales
    with the range of the estimation functions. We avoid
    explicit dependence on the range using a more specialized tail
    bound based on \pref{lem:martingale_chernoff}.}

  \paragraph{Lower bounds}
  Our lower bounds refine the proof strategy from
  \citet{foster2021statistical}. Our most important technical result
  is a lower bound for the stochastic setting (\pref{thm:lower_general}), which improves upon Theorem 3.1 from
  \citet{foster2021statistical} by using a more refined change of
  measure argument. In particular, Theorem 3.1 of
  \citet{foster2021statistical} gives a lower bound based on the
  \CompShort that holds with low probability, and thus provides a
  meaningful converse only to algorithms with
  sub-Gaussian or sub-exponential tail behavior. Our lower bound gives a
  meaningful converse to any upper bound with sub-Chebychev tail
  behavior, which is a significantly weaker assumption. We note that
  while Theorem 3.2 of \citet{foster2021statistical} provides lower
  bounds on expected regret without assumptions on tail behavior, this result requires a stronger notion of
  localization than we consider here. \neurips{Of course, proving a lower bound on expected
    regret that matches our lower bound remains an interesting open problem.}

  Recent work of \citet{lattimore2022minimax},
  provides lower bounds on regret in a general partial
  monitoring setting based on a generalized \irtext
  (cf. \pref{app:structural}). This result is somewhat complementary
  to our lower bound (\pref{thm:lower_main}). 
  On the positive side, \citet{lattimore2022minimax} give a lower bound on \emph{expected
      regret} that is always tight in terms of dependence on $T$,
    while our result only leads to tight dependence on $T$ if one
    restricts to sub-Chebychev algorithms. In addition, the lower bound applies
  to the general partial monitoring setting, while our lower bound
  requires that rewards are observed. However, the lower bound of
  \citet{lattimore2022minimax} is loose in
    $\mathrm{poly}(\abs{\Pi})$ factors, and thus does not lead to meaningful
    dependence on problem parameters such as dimension in the same
    fashion as our results.
    An interesting question for future work is to investigate whether
    these techniques can be combined with our own to get the best of
    both worlds.

 }

\section{Preliminaries}
\label{app:prelim}
\neurips{
  \paragraph{Basic notation} For a set $\cX$, we let
        $\Delta(\cX)$ denote the set of all Radon probability measures
        over $\cX$. We let $\conv(\cX)$ denote the set of all finitely
        supported convex combinations of elements in $\cX$. We use the
        shorthand $x\vee{}y=\max\crl{x,y}$ and $x\wedge{}y=\min\crl{x,y}$.

    We adopt non-asymptotic big-oh notation: For functions
	$f,g:\cX\to\bbR_{+}$, we write $f=\bigoh(g)$ (resp. $f=\bigom(g)$) if there exists a constant
	$C>0$ such that $f(x)\leq{}Cg(x)$ (resp. $f(x)\geq{}Cg(x)$)
        for all $x\in\cX$. We write $f=\bigoht(g)$ if
        $f=\bigoh(g\cdot\mathrm{polylog}(T))$, $f=\bigomt(g)$ if $f=\bigom(g/\polylog(T))$, and
        $f=\bigthetat(g)$ if $f=\bigoht(g)$ and $f=\bigomt(g)$. %
	We write $f\propto g$ if $f=\bigthetat(g)$.
        }

\paragraph{Probability spaces}
We formalize the probability
spaces for the \FrameworkShort framework in the same fashion as \citet{foster2021statistical}, which we briefly summarize here. decisions are
associated with a measurable space $(\Act,\Asig)$, rewards are
associated with the space $(\Rspace,\Rsig)$, and observations are
associated with the space $(\Obs,\Osig)$.  The history up to time $t$ is denoted by  $\hist\ind{t}
=(\act\ind{1},r\ind{1},\obs\ind{1}),\ldots,(\act\ind{t},r\ind{t},\obs\ind{t})$. We define
\[
\Hspace\ind{t}=\prod_{i=1}^{t}(\Act\times\Rspace\times\Obs),\mathand\Hsig\ind{t}=\bigotimes_{i=1}^{t}(\Asig\otimes{}\Rsig\otimes\Osig)
\]
so that $\hist\ind{t}$ is associated with the space $(\Hspace\ind{t},\Hsig\ind{t})$.

Formally, a model $M=M(\cdot,\cdot\mid\cdot)\in\cM$ is a probability
kernel from $(\Act,\Asig)$ to $(\Rspace\times\Obs,\Rsig\otimes\Osig)$;
we use the convention $M(\act) =  M(\cdot,\cdot\mid\act)$ throughout the paper.\footnote{For measurable spaces $(\cX,\scrX)$ and $(\cY,\scrY)$ a
probability kernel $P(\cdot\mid{}\cdot)$ from $(\cX,\scrX)$ to
$(\cY,\scrY)$ has the property that (i) For
all $x\in\cX$, $P(\cdot\mid{}x)$ is a probability measure, (ii) for
all $Y\in\scrY$, $x\mapsto{}P(Y\mid{}x)$ is measurable.} An
\emph{algorithm} for horizon $T$ is a sequence $p\ind{1},\ldots,p\ind{T}$, where $p\ind{t}(\cdot\mid\cdot)$ is a
probability kernel from $(\Hspace\ind{t-1},\Hsig\ind{t-1})$ to
$(\Act,\Asig)$.

\paragraph{Divergences}

  For probability distributions $\bbP$ and $\bbQ$ over a measurable space
  $(\Omega,\filt)$ with a common dominating measure, we define the total variation distance as
  \[
    \Dtv{\bbP}{\bbQ}=\sup_{A\in\filt}\abs{\bbP(A)-\bbQ(A)}
    = \frac{1}{2}\int\abs{d\bbP-d\bbQ}.
  \]
  Hellinger distance is defined as
  \[
    \Dhels{\bbP}{\bbQ}=\int\prn*{\sqrt{d\bbP}-\sqrt{d\bbQ}}^{2},
\]
and Kullback-Leibler divergence is defined as
  \[
    \Dkl{\bbP}{\bbQ} =\left\{
    \begin{array}{ll}
\int\log\prn[\big]{
      \frac{d\bbP}{d\bbQ}
      }d\bbP,\quad{}&\bbP\ll\bbQ,\\
      +\infty,\quad&\text{otherwise.}
    \end{array}\right.
\]
For a convex function $f:(0,\infty)\to\bbR$, the associated
$f$-divergence for measures $\bbP$ and $\bbQ$ with $\bbP\ll\bbQ$ is
given by
\begin{align}
  \label{eq:fdiv1}
  \Df{\bbP}{\bbQ} \ldef \En_{\bbQ}\brk*{f\prn*{\frac{d\bbP}{d\bbQ}}}
\end{align}
whenever $\bbP\ll\bbQ$. More generally, defining
$p=\frac{d\bbP}{d\nu}$ and $q=\frac{d\bbQ}{d\nu}$ for a common
dominating measure $\nu$, we have
\begin{align}
  \label{eq:fdiv2}
  \Df{\bbP}{\bbQ} \ldef \int_{q>0}qf\prn*{\frac{p}{q}}d\nu +
  \bbP(q=0)\cdot{}f'(\infty), 
\end{align}
where $f'(\infty)\ldef{}\lim_{x\to{}0^{+}}xf(1/x)$.

\section{Technical Tools}
\label{app:technical}

\subsection{Tail Bounds}

        \begin{lemma}[e.g., Lemma A.4 of \citet{foster2021statistical}]
    \label{lem:martingale_chernoff}
    For any sequence of real-valued random variables $\prn{X_t}_{t\leq{}T}$ adapted to a
    filtration $\prn{\filt_t}_{t\leq{}T}$, we have that with probability at least
    $1-\delta$,
    \begin{equation}
      \label{eq:martingale_chernoff}
      \sum_{t=1}^{T}X_t \leq
      \sum_{t=1}^{T}\log\prn*{\En\brk*{e^{X_t}\mid\filt_{t-1}}} + \log(\delta^{-1}).
    \end{equation}
  \end{lemma}

  \subsection{Minimax Theorem}

    \begin{lemma}[Sion's Minimax Theorem \citep{sion1958minimax}]
          \label{lem:sion}
          Let $\cX$ and $\cY$ be convex sets in linear topological
          spaces, and assume $\cX$ is compact. Let
          $F:\cX\times\cY\to\bbR$ be such that (i) $F(x, \cdot)$ is
          concave and upper semicontinuous over $\cY$ for all
          $x\in\cX$ and (ii) $F(\cdot,y)$ is convex and lower
          semicontinuous over $\cX$ for all $y\in\cY$. Then
          \begin{equation}
            \label{eq:minimax_theorem}
            \inf_{x\in\cX}\sup_{y\in\cY}F(x,y) = \sup_{y\in\cY}\inf_{x\in\cX}F(x,y).
          \end{equation}
        \end{lemma}
  
\subsection{Information Theory}

\subsubsection{Basic Results}

 \begin{restatable}{proposition}{hellingerconditional}
   \label{prop:hellinger_conditional}
   For any $f$-divergence $\Df{\cdot}{\cdot}$, one has that for any
   pair of random variables $(X,Y)$ with joint law $\bbP_{X,Y}$,
   \begin{align*}
     \En_{X\sim\bbP_{X}}\brk*{\Df{\bbP_{Y\mid{}X}}{\bbP_Y}}
     = \En_{Y\sim\bbP_{Y}}\brk*{\Df{\bbP_{X\mid{}Y}}{\bbP_X}}.
   \end{align*}
 \end{restatable}
 \begin{proof}[\pfref{prop:hellinger_conditional}]
   Recalling that
   $\Df{\bbP}{\bbQ}=\En_{\bbQ}\brk*{f\prn*{\frac{d\bbP}{d\bbQ}}}$ for
   $\bbP\ll\bbQ$, we have
   \begin{align*}
     \En_{X\sim\bbP_{X}}\brk*{\Df{\bbP_{Y\mid{}X}}{\bbP_Y}} 
     &=
       \textstyle\En_{X\sim\bbP_{X}}\En_{Y\sim\bbP_{Y}}\brk*{f\prn*{\frac{d\bbP_{Y\mid{}X}}{d\bbP_Y}}}
     \\
     &=
       \textstyle\En_{X\sim\bbP_{X}}\En_{Y\sim\bbP_{Y}}\brk*{f\prn*{\frac{d\bbP_{X,Y}}{d(\bbP_X\otimes\bbP_Y)}}}
     \\
     &=
       \textstyle\En_{Y\sim\bbP_{Y}}\En_{X\sim\bbP_{X}}\brk*{f\prn*{\frac{d\bbP_{X\mid{}Y}}{d\bbP_X}}}
     =      \En_{Y\sim\bbP_{Y}}\brk*{\Df{\bbP_{X\mid{}Y}}{\bbP_X}},
   \end{align*}
   where we have used that $\bbP_{Y\mid{}X}\ll\bbP_Y$,
   $\bbP_{X\mid{}Y}\ll\bbP_X$, and $\bbP_{X,Y}\ll\bbP_X\otimes\bbP_Y$.
 \end{proof}

  \subsubsection{Change of Measure}

 \begin{restatable}[Donsker-Varadhan (e.g., \citet{polyanskiy2014lecture})]{lemma}{donskervaradhan} 
   \label{lem:dv}
   Let $\bbP$ and $\bbQ$ be probability measures on $(\cX,\filt)$. Then
   \begin{align}
     \label{eq:dv}
     \Dkl{\bbP}{\bbQ}=\sup_{h:\cX\to\bbR}\crl*{
     \En_{\bbP}\brk*{h(X)} - \log\prn*{\En_{\bbQ}\brk*{\exp\prn*{h(X)}}}
     }.
   \end{align}
 \end{restatable}

     \begin{restatable}{lemma}{hellingercom}
    \label{lem:hellinger_com} 
        Let $\bbP$ and $\bbQ$ be probability distributions over a measurable space
    $(\cX,\filt)$. Then for all functions $h:\cX\to\bbR$, 
            \begin{equation}
      \label{eq:hellinger_com}
      \abs*{\En_{\bbP}\brk*{h(X)}
        - \En_{\bbQ}\brk*{h(X)}}
      \leq{} \sqrt{2^{-1}\prn*{\En_{\bbP}\brk*{h^2(X)} + \En_{\bbQ}\brk*{h^2(X)}}\cdot\Dhels{\bbP}{\bbQ}}.
    \end{equation}
  \end{restatable}
  \begin{proof}[\pfref{lem:hellinger_com}]
    From \citet{polyanskiy2014lecture}, it holds that for all functions
    $h:\cX\to\bbR$, if $\bbP\ll\bbQ$, 
    \begin{equation}
      \label{eq:chisquared_com}
      \abs*{\En_{\bbP}\brk*{h(X)}
        - \En_{\bbQ}\brk*{h(X)}}
      \leq{} \sqrt{\Var_{\bbQ}\brk*{h(X)}\cdot\Dchis{\bbP}{\bbQ}}
      \leq{} \sqrt{\En_{\bbQ}\brk*{h^2(X)}\cdot\Dchis{\bbP}{\bbQ}},
    \end{equation}
    where
    $\Dchis{\bbP}{\bbQ}\ldef\int\frac{(d\bbP-d\bbQ)^2}{d\bbQ}$ and
    $\Var_{\bbQ}$ denotes the variance under $\bbQ$. The
    result follows by using that $\Dchis{\bbP}{\tfrac{\bbP+\bbQ}{2}}\leq\Dhels{\bbP}{\bbQ}$.
  \end{proof}

\hellingerexp*
  \begin{proof}[\pfref{lem:hellinger_exp}]
    We first show that Hellinger distance is lower bounded by the
    quantity in \pref{eq:hellinger_exp}. Recall that Hellinger
    distance is the $f$-divergence associated with
    $f(x)=(1-\sqrt{x})^{2}$ (cf. \pref{eq:fdiv2}).
    Let
    $\fstar(y)\ldef\sup_{x\geq{}0}\crl*{xy - f(x)}$ be the Fenchel
    dual of $f$, which has the form
    \[
      \fstar(y) =
      \left\{\begin{array}{ll}
               \frac{y}{1-y},&\quad y < 1,\\
               \infty,&\quad y \geq{}1.
      \end{array}\right.
  \]
  Using Theorem 7.14 of \citet{polyanskiy2020lecture}, we express
  Hellinger distance as a variational problem based on
  the dual:
  \begin{align*}
    \Dhels{\bbP}{\bbQ}
    = \sup_{h:\cX\to(-\infty,1)}\crl*{\En_{\bbP}\brk*{h(X)} -
    \En_{\bbQ}\brk*{\fstar(h(X))}} 
    = \sup_{h:\cX\to(-\infty,1)}\crl*{\En_{\bbP}\brk*{h(X)} - \En_{\bbQ}\brk*{\frac{h(X)}{1-h(X)}}}.
  \end{align*}
  Reparameterizing via $h(X)=1-h'(X)$ for $h':\cX\to(0,\infty)$, this gives
  \begin{align*}
    \Dhels{\bbP}{\bbQ}= \sup_{h:\cX\to(0,\infty)}\crl*{2 -
    \En_{\bbP}\brk*{h(X)} - \En_{\bbQ}\brk*{\frac{1}{h(X)}}}.
  \end{align*}
  To conclude, observe that for any test function $g:\cX\to\bbR$,
  by setting $h(x)=e^{g(x)}\cdot\En_{\bbQ}\brk*{e^{-g}}$, we have
    \begin{align*}
      2 -
      \En_{\bbP}\brk*{h(X)} - \En_{\bbQ}\brk*{\frac{1}{h(X)}}
      &= 2 -
      \En_{\bbP}\brk[\big]{e^{g}}\cdot{}\En_{\bbQ}\brk[\big]{e^{-g}} -
      \En_{\bbQ}\brk[\big]{e^{-g}}/\En_{\bbQ}\brk[\big]{e^{-g}} \\
        &= 1 - \En_{\bbP}\brk[\big]{e^{g}}\cdot{}\En_{\bbQ}\brk[\big]{e^{-g}},
    \end{align*}
    so that
    \[
      \Dhels{\bbP}{\bbQ}
      \geq{} \sup_{g:\cX\to\bbR}\crl*{1 - \En_{\bbP}\brk[\big]{e^{g}}\cdot{}\En_{\bbQ}\brk[\big]{e^{-g}}}.
    \]
    We now prove the other direction of the inequality in
    \pref{eq:hellinger_exp}. Let $\nu$ be a common dominating measure
    for $\bbP$ and $\bbQ$, and set $p=\frac{d\bbP}{d\nu}$ and
    $q=\frac{d\bbQ}{d\nu}$. We first consider the case where $p,q>0$
    everywhere. Set $g(x)=\frac{1}{2}\log(q(x)/p(x))$. Then
    we have
    $\En_{\bbP}\brk[\big]{e^{g}}=\int\sqrt{pq}d\nu=1-\frac{1}{2}\Dhels{\bbP}{\bbQ}$,
    and likewise,
    $\En_{\bbQ}\brk[\big]{e^{-g}}=\int\sqrt{pq}d\nu=1-\frac{1}{2}\Dhels{\bbP}{\bbQ}$. As
    a result,
    \begin{align*}
      \sup_{g:\cX\to\bbR}\crl*{1-\En_{\bbP}\brk[\big]{e^{g}}\cdot\En_{\bbQ}\brk[\big]{e^{-g}}}
      \geq{} 1 - (1-\tfrac{1}{2}\Dhels{\bbP}{\bbQ})^2 \geq{} \frac{1}{2}\Dhels{\bbP}{\bbQ},
    \end{align*}
    where we have used that $\Dhels{\bbP}{\bbQ}\in\brk*{0,2}$. For the general case, we appeal to \pref{lem:hellinger_exp2}
    below and take $\veps\to{}0$.
  \end{proof}

  The following result is generalization of \pref{lem:hellinger_exp}
  which shows that up to small approximation error, the lower bound in
  \pref{eq:hellinger_exp} can be obtained using test functions with
  small magnitude.
  \begin{restatable}{lemma}{hellingerexp2}
    \label{lem:hellinger_exp2}
    Let $\bbP$ and $\bbQ$ be probability distributions over a measurable space
    $(\cX,\filt)$. Then for any $\alpha\geq{}1$,
    \begin{align}
      \label{eq:hellinger_exp2}
      \frac{1}{2}\Dhels{\bbP}{\bbQ}  \leq{}
      \sup_{g\in\cG_{\alpha}}\crl*{1-\En_{\bbP}\brk[\big]{e^{g}}\cdot\En_{\bbQ}\brk[\big]{e^{-g}}}
      + 4e^{-\alpha},
    \end{align}
    where $\cG_{\alpha}\ldef{}\crl*{g:\cX\to\bbR \mid{} \nrm*{g}_{\infty}\leq{}\alpha}$.
  \end{restatable}

  \begin{proof}[\pfref{lem:hellinger_exp2}]
    Fix $\alpha\geq{}1$ and let $\veps\ldef{}e^{-2\alpha}$. Note that
    $\veps\in(0,e^{-2})$. Given measures $\bbP$ and $\bbQ$, set
    $\bbP_{\veps}=(1-\veps)\bbP+\veps\bbQ$ and
    $\bbQ_{\veps}=(1-\veps)\bbQ+\veps\bbP$. Consider the test function
    $g=\frac{1}{2}\log(\tfrac{d\bbQ_{\veps}}{d\bbP_{\veps}})$, which
    has the following properties:
    \begin{itemize}
    \item $\nrm*{g}_{\infty}\leq{}
      \frac{1}{2}\log\prn*{\frac{1-\veps}{\veps}+\frac{\veps}{1-\veps}}\leq\frac{1}{2}\log(\veps^{-1})$,
      where we have used that $\veps\leq{}1/2$. This establishes that $g\in\cG_{\alpha}$.
      \item 
        $\En_{\bbP}\brk[\big]{e^{g}}\leq(1-\veps)^{-1/2}\int\sqrt{d\bbP{}d\bbQ_{\veps}}=(1-\veps)^{-1/2}\prn*{1-\frac{1}{2}\Dhels{\bbP}{\bbQ_{\veps}}}$.
      \item 
$\En_{\bbQ}\brk[\big]{e^{-g}}\leq(1-\veps)^{-1/2}\int\sqrt{d\bbP_{\veps}d\bbQ}=(1-\veps)^{-1/2}\prn*{1-\frac{1}{2}\Dhels{\bbP_{\veps}}{\bbQ}}$.
\end{itemize}
Using these bounds yields
    \begin{align*}
      \sup_{g:\cX\to\bbR}\crl*{1-\En_{\bbP}\brk[\big]{e^{g}}\cdot\En_{\bbQ}\brk[\big]{e^{-g}}}
      &\geq{} 1 -
        (1-\veps)^{-1}(1-\tfrac{1}{2}\Dhels{\bbP_{\veps}}{\bbQ})
        (1-\tfrac{1}{2}\Dhels{\bbP}{\bbQ_{\veps}}) \\
      &\geq{} 1 -
        (1-\veps)^{-1}(1-\tfrac{1}{2}\Dhels{\bbP_{\veps}}{\bbQ})\\
      &\geq (1-\veps)^{-1}\cdot\frac{1}{2}\Dhels{\bbP_{\veps}}{\bbQ} - 2\veps.
    \end{align*}
    Finally, we note that by the triangle inequality for Hellinger
    distance and convexity of squared Hellinger distance,
    \begin{align*}
      \Dhel{\bbP}{\bbQ}
      \leq{}\Dhel{\bbP_{\veps}}{\bbQ} + \Dhel{\bbP}{\bbP_{\veps}}
      \leq{}\Dhel{\bbP_{\veps}}{\bbQ} +
        \veps^{1/2}\Dhel{\bbP}{\bbQ}, 
    \end{align*}
    so that
    $\Dhels{\bbP_{\veps}}{\bbQ}\geq{}(1-\veps^{1/2})^2\Dhels{\bbP}{\bbQ}$,
    and
    \begin{align*}
      \sup_{g:\cX\to\bbR}\crl*{1-\En_{\bbP}\brk[\big]{e^{g}}\cdot\En_{\bbQ}\brk[\big]{e^{-g}}}
      \geq{}
        \frac{(1-\veps^{1/2})^2}{1-\veps}\frac{1}{2}\Dhels{\bbP}{\bbQ}-2\veps
      \geq{} \frac{1}{2}\Dhels{\bbP}{\bbQ}-4\veps^{1/2},
    \end{align*}
    where we have used that $\veps\in(0,1)$ and $\Dhels{\bbP}{\bbQ}\in\brk*{0,2}$.

  \end{proof}

  \subsection{Online Learning}

  \begin{lemma}[e.g., \citet{PLG}]
    \label{lem:exponential_weights}
    Let $\Pi$ be a finite set. Consider the exponential weights method with learning rate
    $\eta>0$ and initial point $q\ind{1}=\unif(\Pi)$, which
    has the update:
    \[
      q\ind{t+1}(\pi) = \frac{\exp(\eta{}\sum_{i\leq{}t}f\ind{i}(\pi))}{\sum_{\pi'}\exp(\eta{}\sum_{i\leq{}t}f\ind{t}(\pi'))},
    \]
    for an arbitrary (potentially
    adaptively selected) sequence of
    reward vectors $f\ind{1},\ldots,f\ind{T}$ in $\bbR^{\Pi}$. This strategy ensures that with
    probability $1$,
    \begin{align*}
  \sum_{t=1}^{T}\tri*{q - q\ind{t},f\ind{t}}
  &\leq{}
    \sum_{t=1}^{T}\tri*{q\ind{t+1}-q\ind{t},f\ind{t}}
        - \frac{1}{\eta}\sum_{t=1}^{T}\Dkl{q\ind{t+1}}{q\ind{t}} +
    \frac{\Dkl{q}{q\ind{1}}}{\eta},
    \end{align*}
    for all $q\in\Delta(\Pi)$.
  \end{lemma}

\neurips{
\section{Examples}
\label{sec:examples}
In this section, we instantiate our upper bounds for basic examples of interest.
\subsection{Structured Bandits}
We first consider adversarial (structured) bandit problems,
which correspond to the
special case of the adversarial \FrameworkShort setting in which there are no
observations (i.e., $\cO=\NullObs$). We consider three examples:
finite-armed bandits, linear bandits, and convex bandits. For each
example, we take $\cR=\brk*{0,1}$, fix a \emph{reward function class}
$\cF\subseteq(\Pi\to\brk*{0,1})$, and take
$\cMf=\crl*{M\mid{}\fm\in\cF}$ to be the induced model
class. The class $\cMf$ should be thought of as the set of all
reward distributions over $\brk*{0,1}$ with mean rewards in $\cF$.

\begin{example}[Finite-armed bandit] 
  \label{ex:mab}
In the finite-armed bandit problem, we take $\Pi=\crl*{1,\ldots,A}$ as
the decision space for $A\in\bbN$, let $\cF=\brk*{0,1}^{A}$, and
take $\cM=\cMf$ as the induced model class. For this setting, whenever
\(A \geq 2\), it holds that
\begin{align}
  \label{eq:mab}
  \comp[\gamma](\conv(\cM)) \leq \frac{A}{\gamma}\quad\forall{}\gamma>0,\mathand
  \comp[\gamma, \veps_\gamma](\conv(\cM)) \geq{} 2^{-6} \cdot \frac{A}{\gamma}\quad\forall\gamma\geq\frac{A}{3},
\end{align}
where $\veps_\gamma = \frac{A}{12 \gamma}$.
\end{example}
This result follows from \citet[Proposition 5.2 and
5.3]{foster2021statistical}, noting that \(\conv(\cM) =
\cM\). Plugging \pref{eq:mab} into \pref{thm:upper_main} yields a
$\bigoh(\sqrt{AT\log{}A})$ upper bound on regret, and plugging into \pref{thm:lower_main} gives a $\bigomt(\sqrt{AT})$ lower
bound for \subc algorithms.\footnote{For this example and \pref{ex:linear}, the lower bound on $\comp[\gamma,
  \veps_\gamma](\conv(\cM))$ in \citet{foster2021statistical} is witnessed by a subfamily
  $\cM'\subseteq\cM$ with $V(\cM')=\bigoh(1)$. As a result, we can
  take $\Ct=\bigoh(1)$ in \pref{thm:lower_main}.}

\begin{example}[Linear bandit]
  \label{ex:linear}
  In the linear bandit problem, we have $\Pi\subseteq\bbR^{d}$ and
  \[
    \cF=\crl*{f:\Pi\to\brk*{0,1}\mid{}\text{$f$ is linear}},
  \]
  and take $\cM=\cMf$ as the induced model class. For this setting, it
  holds that\footnote{The upper bound here holds for all $\Pi$, while
    the lower bound holds for a specific choice for $\Pi$.}
  \begin{align}
    \label{eq:linear_bandit}
 \comp[\gamma](\conv(\cM)) \leq
    \frac{d}{4\gamma}\quad\forall\gamma>0,\mathand
    \comp[\gamma, \veps_\gamma](\conv(\cM))\geq{}\frac{d}{12 \gamma}  \quad\forall\gamma\geq\frac{2d}{3},
  \end{align}
  where $\veps_\gamma \ldef \frac{d}{3 \gamma}$.
\end{example}

This result follows from \citet[Proposition 6.1 and
6.2]{foster2021statistical}, again noting that
\(\conv(\cM) = \cM\). Plugging \pref{eq:mab} into \pref{thm:upper_main} yields a
$\bigoh(\sqrt{dT\log{}\abs{\Pi}})$ upper bound on regret, and plugging into \pref{thm:lower_main} gives a $\bigomt(\sqrt{dT})$ lower
bound for \subc algorithms.

\begin{example}[Convex bandit]
  In the convex bandit problem, we have $\Pi\subseteq\bbR^{d}$ and
  \[
    \cF=\crl*{f:\Pi\to\brk*{0,1}\mid{}\text{$f$ is convex}},
  \]
  and take $\cM=\cMf$ as the induced model class. For this setting, it
  holds that for all $\gamma>0$, 
  \begin{align}
    \label{eq:convex_bandit}
\comp[\gamma](\conv(\cM)) \leq O \prn*{\frac{d^4}{\gamma} \cdot \text{polylog}(d, \diam(\Pi), \gamma)}. 
\end{align}
\end{example} 
This result follows from \citet[Proposition
6.3]{foster2021statistical} (which itself is a restatement of
\citet[Theorem 3]{lattimore2020bandit}), and by noting once more that \(\conv(\cM) = \cM\).

\begin{remark}
  The adversarial bandit literature \citep{auer2002non,audibert2009minimax,hazan2011better,dani2007price,abernethy2008competing,bubeck2012towards,kleinberg2004nearly,flaxman2005online,bubeck2017kernel,lattimore2020improved,kleinberg2004nearly,bubeck2011x}
  typically considers a slightly different formulation in which the adversary selects a
  deterministic reward function. This can be captured by restricting
  $\cM$ to deterministic models. It is clear that the upper bounds on
  $\comp(\conv(\cM))$ in the examples above lead to upper bounds for
  this formulation. The lower bounds in \pref{ex:mab,ex:linear} easily extend as well.
\end{remark}

\subsection{Reinforcement Learning}
We now consider an example in reinforcement learning. We begin by
recalling how to view the episodic reinforcement learning problem in the DMSO framework. 

\paragraph{Model class}
For episodic reinforcement learning, we fix a \emph{horizon} \(H\) and
let the model class \(\cM\) consist of a set of non-stationary Markov
Decision Processes (MDP). Each model \(M \in \cM\) is specified by \[M
  = \crl[\big]{\crl*{\cS_h}_{h=1}^{H+1}, \cA,
    \crl*{P\sups{M}_h}_{h=1}^H, \crl*{R_h\sups{M}}_{h=1}^H, d_1},\]
where \(\cS_h\) is the state space for layer \(h\), \(\cA\) is the
action space, \(P_h\sups{M}: \cS_h \times \cA \mapsto
\Delta(\cS_{h+1})\) is the probability transition kernel for layer
\(h\), \(R_h\sups{M}: \cS_h \times \cA \mapsto \Delta([0, 1])\) is the
reward distribution for layer \(h\) and \(d_1 \in \Delta(\cS_1)\) is
the initial state distribution. This formulation allows the reward
distribution and transition kernel to vary across models in \(\cM\),
but keeps the initial state distribution is fixed. We adopt the
convention that \(S_{H+1} = \crl{s_{H+1}}\) where \(s_{H+1}\) is a deterministic
terminal state.

For each episode, the learner selects a non-stationary
policy \(\pi = \prn*{\pi_1, \dots, \pi_H}\),  where \(\pi_h: \cS_h
\mapsto \cA\); we let \(\PiNS\) denote the set of
all such policies. For a given MDP $M\in\cM$, an episode proceeds by first sampling \(s_1 \sim d_1\), then for \(h = 1, \dots, H\): 
\begin{itemize}
\item \(a_h = \pi_h(s_h)\).
\item \(r_h \sim R\sups{M}_h(s_h, a_h)\) and \(s_{h+1} \sim P_h\sups{M}(\cdot \mid s_h, a_h)\). 
\end{itemize} 
The value of the policy \(\pi\) under $M$ is given by \(f\sups{M}(\pi) \ldef{} \En^{\sss{M}, \pi} \brk[\big]{\sum_{h=1}^H r_h}\), where \(\En^{\sss{M}, \pi} \brk*{\cdot}\) denotes expectation under the process above. 

\paragraph{Adversarial protocol} Within the adversarial \FrameworkShort framework, model classes above
give rise to the following adversarial reinforcement learning protocol. At
each time $t$, the learner plays selects a policy $\pi\in\PiNS$ and
the adversary chooses an MDP $M\ind{t} \in \cM$. The policy \(\pi
\ind{t}\) is executed in \(M\ind{t}\), resulting in a
trajectory \(\tau \ind{t} = (s_1 \ind{t}, r_1\ind{t}, r_1\ind{t}),
\dots, (s_H \ind{t}, r_H \ind{t}, r_H\ind{t})\). The learner then
observes feedback \((r\ind{t}, o\ind{t})\), where \(r\ind{t} \ldef
\sum_{h=1}^H r_H\ind{t}\) is the cumulative reward for the episode, and \(o\ind{t} = \tau\ind{t}\) is the trajectory. 

With this setting in mind, we give our main example.
\begin{example}[Tabular MDP] 
  \label{ex:tabular} 
Let \(\cM\) be the class of finite-state/action (tabular) MDPs with
horizon $H$, \(S \geq 2\) states, \(A \geq 2 \) actions, and \(\sum_{h=1}^H r_h \in [0, 1]\). Then, for any \(\gamma \geq  A^
{\min\crl{S - 1, H}}/6\),
\[\comp[\gamma, \veps_\gamma](\conv(\cM)) \geq \frac{A^
    {\min\crl{S - 1, H}}}{24 \gamma},\]
where \(\veps_\gamma \ldef {A^
{\min\crl{S - 1, H}}}/{24 \gamma}\).
\end{example}
Using this result with \pref{thm:lower_main} leads to a lower bound on regret
that scales with $\bigom(A^{\min{S-1,H}})$, which recovers existing
intractability results for this setting
\citep{kwon2021rl,liu2022learning}. Note that we have
$\comp(\cM)=\mathrm{poly}(S,A,H)/\gamma$ for this setting
\citep{foster2021statistical}, so this is a case where there is a large
separation between stochastic and adversarial decision making.

We briefly mention that the set $\conv(\cM)$ can be interpreted as the
set of \emph{latent MDPs} \citep{kwon2021rl}, which is also known to
be intractable. In the latent MDP
problem, each model is a mixture of MDPs. At the beginning of each
episode, the underlying MDP is drawn from the mixture (the identity is not
observed), and the learner's policy is executed in this MDP for the
duration of the episode.
\subsection{Proofs for Examples}
\label{app:examples}

\neurips{\subsubsection{Preliminaries}}
\arxiv{\paragraph{Preliminaries}}

Our lower bounds on the \CompText involve a constructing hard sub-family of models. Recall the following definition from \cite{foster2021statistical}. 
\begin{definition}[$(\alpha, \beta, \delta)$-family]
\label{def:hard_family_lb}
 A reference model \(\Mbar \in \cM\) and collection \(\crl{M_1, \dots, M_N}\) with \(N \geq 2\) are said to be an \((\alpha, \beta, \delta)\)-family if the following properties hold: 
\begin{itemize}
\item \textit{Regret property.} There exist functions \(u\sups{M} : \Pi \mapsto [0, 1]\), with \(\sum_{M \in \cM} u\sups{M}(\pi) \leq \frac{N}{2}\) for all \(\pi\) such that 
\begin{align*}
f\sups{M}(\pi\subs{M}) - f\sups{M}(\pi) \geq \alpha \cdot \prn*{1 - u\sups{M}(\pi)}
\end{align*}
for all \(M \in \cM\). 
\item \textit{Information property.} There exist functions \(v\sups{M}: \Pi \mapsto [0, 1]\), with \(\sum_{M \in \cM} v\sups{M}(\pi) \leq 1\) for all \(\pi\), such that 
\begin{align*}
\Dhels{M(\pi)}{\Mbar(\pi)}  \leq \beta \cdot v\sups{M}(\pi) + \delta. 
\end{align*}
\end{itemize} 
\end{definition}

Any $(\alpha, \beta, \delta)$-family leads to a difficult decision
making problem because a given decision can have low regret or large
information gain on (roughly) one model in the family. This is formalized through the following lemma. 
\begin{lemma}[Lemma 5.1, \cite{foster2021statistical}]
\label{lem:hard_family_lb}
 Let \(\cM = \crl{M_1, \dots, M_N}\) be an $(\alpha, \beta, \delta)$-family with respect to \(\Mbar\). Then, for all \(\gamma \geq 0\), 
\begin{align*}
  \comp(\cM,\Mbar) \geq \frac{\alpha}{2} - \gamma \prn*{\frac{\beta}{N} + \delta}. 
\end{align*}
\end{lemma}

The following technical lemma bounds Hellinger distance for Bernoulli distributions. 
\begin{lemma}[Lemma A.7, \citep{foster2021statistical}]
\label{lem:helg_bern} 
 For any \(\Delta \in (0, 1/2)\),  
\begin{align*}
\DhelsX{\Big}{\Ber\prn[\Big]{\frac{1}{2} + \Delta}}{\Ber\prn[\Big]{\frac{1}{2}}} \leq 3 \Delta^2. 
\end{align*}	
\end{lemma}

\neurips{\subsubsection{Proof for \preftitle{ex:tabular} (Tabular MDP)}}
\arxiv{\subsection{Proof for \preftitle{ex:tabular} (Tabular MDP)}}

In this section, we prove the lower bound in \pref{ex:tabular}. We
first derive an intermediate result which gives a lower bound on the
\CompText when the model class $\cM$ consists of \emph{mixtures of $K$
  MDPs}; this is equivalent to the subset of $\conv(\cM)$ where we
restrict to support size $K$, as well as the so-called latent MDP
setting \citep{kwon2021rl}.

\begin{lemma}
\label{lem:tabular}
Let \(K \geq 1\) be given. Let \(\cM\) be the class of \emph{mixtures
  of \(K\) MDPs} with horizon \(H\), \(S \geq 2\) states, \(A\geq{}2\) actions, and \(\sum_{h=1}^H r_h \in [0, 1]\). Then there exists \(\Mbar \in \cM\) such that for all \(\gamma \geq  A^
{\min\crl{S - 1, H, K}}/6\),
\[\comp(\cM_{\veps_\gamma}\prn*{\Mbar}, \Mbar) \geq \frac{A^
    {\min\crl{S - 1, H, K}}}{24 \gamma},
  \]
where \(\veps_\gamma \ldef \frac{A^
{\min\crl{S - 1, H, K}}}{24 \gamma}\).
\end{lemma}%
\newcommand{\ba}{\mb{a}}%
\renewcommand{\bK}{\wb{K}}%
 The proof of this result proceeds by constructing a hard sub-family of
 models and appealing to \pref{lem:hard_family_lb}. Our construction
 is based of the lower bound for latent MDPs in \citet{kwon2021rl}.

 \begin{proof}[\pfref{lem:tabular}]  Let \(\cS\) and \(\cA\) be
   arbitrary sets with $\abs{\cS}=S$ and $\abs{\cA}=A$. Let \(\Delta
   \in (0, 1/2)\) be a parameter to be chosen later, and define \(\bK
   \ldef \min\crl{S-1, K, H}\). Partition the state space \(\cS\) into
   sets \(\cS'\) and \(\cS \setminus \cS'\) such that \(\abs{\cS'} =
   \bK + 1\), and label the states in \(\cS'\) as $\crl{s\ind{1},
     \dots, s\ind{\bK + 1}}$. Additionally, define sets via \(\cS_h =
   \crl{s\ind{h}, s\ind{\bK + 1}}\) for \(h \leq \bK\) and \(\cS_{h} =
   \crl{s\ind{\bK+1}} \cup \prn{S \setminus S'}\) for \(\bK < h \leq
   H+1\). Recall that the decision space \(\PiNS\) is the set of all deterministic non-stationary policies \(\pi = \prn*{\pi_1, \dots, \pi_H}\) where \(\pi_h: \cS_h
   \mapsto \cA\).
   
We construct a class \(\cM'\subseteq\cM\) in which each model \(M \in \cM'\) is specified by 
 \begin{align*}
 M = \crl[\big]{\crl*{\cS_h}_{h=1}^{H+1}, \cA, \crl{\bbM\sups{M}_k}_{k=1}^{\bK}, \crl{a\sups{M}_k}_{k=1}^K},
\end{align*}
where for each \(k \in [\bK]\), \(a_k\sups{M} \in \cA\),  and where \(\bbM_k\sups{M}\) is a tabular MDP specified by 
\begin{align*}
\bbM\sups{M}_k = \crl[\big]{\crl*{\cS_h}_{h=1}^{H+1}, \cA,
    \crl*{P\sups{M}_{h, k}}_{h=1}^H, \crl*{R_{h, k}\sups{M}}_{h=1}^H, \delta_{s\ind{1}}}.
\end{align*}
Here, $d_1=\delta_{s\ind{1}}$, so that the initial state \(s_1\) is
\(s\ind{1}\) deterministically. The transitions \(P\sups{M}_{h, k}\)
and rewards \(R\sups{M}_{h, k}\) are constructed as follows.
\begin{enumerate}[label=\(\bullet\)]
\item Construction of \(\bbM\sups{M}_{1}\).
\begin{enumerate}[label=(\roman*)]
\item For all \(h \leq H\), the dynamics \(P\sups{M}_{h, k}\) are
  deterministic. For an action \(a_h\) in the state \(s_h\), the next state \(s_{h+1}\) is  
\begin{align*}
s_{h+1}  = \begin{cases} s\ind{h+1}, & \text{if $h \leq \bK$, $s_h = s
    \ind{h}$, and $a_h = a\sups{M}_i$}, \\ 
s\ind{\bK+1},\quad & \text{if $h \leq \bK$, $s_h = s \ind{h}$, and $a_h \neq a\sups{M}_i$}, \\ 
s_h,\quad &\text{otherwise}.
\end{cases}
\end{align*}

\item The reward distribution is given by 
\begin{align*}
R\sups{M}_{h, k}(s_h, a_h) = \begin{cases}
		\Ber\prn[\big]{\tfrac{1}{2} + \Delta}, & \text{if} \quad h = \bK, s_h = s\ind{\bK}, ~\text{and}~ a_h = a\sups{M}_{\bK}, \\
		\Ber\prn[\big]{\tfrac{1}{2}}, & \text{if} \quad h = \bK, s_h = s\ind{\bK}, ~\text{and}~ a_h \neq a\sups{M}_{\bK}, \\
		0, & \text{otherwise}.
\end{cases}
\end{align*} 
\end{enumerate} 
\item Construction of \(\bbM\sups{M}_{j}\) for \(2 \leq j \leq \bK\). 
\begin{enumerate}[label=(\roman*)]
\item For each \(h \leq H\), the dynamics \(P\sups{M}_{h, k}\) are
  deterministic. For action \(a_h\) in state \(s_h\), the next state \(s_{h+1}\) is   
\begin{align*}
s_{h+1}  = \begin{cases} 
s\ind{h + 1} &\text{if} \quad s_h = s\ind{h}~ \text{and}~ h < j  \\ 
s\ind{\bK+1} &\text{if} \quad s_h = s\ind{h}, h = j ~\text{and}~ a_h = a\sups{M}_{h} \\ 
s\ind{h+1} & \text{if} \quad s_h = s\ind{h},  h = j ~ \text{and}~ a_h \neq a\sups{M}_{h} \\ 
s\ind{h+1} &\text{if} \quad s_h = s\ind{h}, h > j ~ \text{and} ~ a_h = a\sups{M}_{h} \\ 
s\ind{\bK+1} &\text{if} \quad  s_h = s\ind{h}, h > j ~ \text{and} ~ a_h \neq a\sups{M}_{h} \\ 
s\ind{\bK+1} &\text{if} \quad h = \bK-1 ~\text{or}~ h = \bK \\
s_h & \text{otherwise} 
\end{cases} .
\end{align*} 

\item The reward distribution is given by 
\begin{align*}
R\sups{M}_{h, k}(s_h, a_h) = \begin{cases}
		\Ber\prn[\big]{\tfrac{1}{2}}, & \text{if} \quad h = \bK,\\	
		0, & \text{otherwise}.
\end{cases}  
\end{align*} 
\end{enumerate} 
\end{enumerate} 
Each model \(M\in\cM'\) is a uniform mixture of \(\bK\) MDPs
$\crl{\bbM\sups{M}_{1}, \dots, \bbM\sups{M}_{\bK}}$ as described above,
parameterized by the action sequence \(a\sups{M}_{1:\bK}\). The model
class \(\cM'\) is defined as the set of all such mixture models (one
for each sequence in $\cA^{\bK}$, so that \(\abs{\cM'} = A^{\bK}\). 

At the start of each episode, an MDP \(\bbM\sups{M}_{z}\) is chosen by sampling \(z \sim \text{Unif}([\bK])\). The trajectory is then drawn by setting \(s_1 = s\ind{1}\), and for \(h = 1, \dots, H\): 
\begin{itemize}
\item \(a_h = \pi_h(s_h)\).
\item \(r_h \sim R\sups{M}_{h, z}(s_h, a_h)\) and \(s_{h+1} \sim P_{h, z}\sups{M}(\cdot \mid s_h, a_h)\). 
\end{itemize} 
Note that rewards can be non-zero only at layer $h=\bK$. We receive a reward from \(\Ber\prn[\big]{\tfrac{1}{2} + \Delta}\) only when \(z = 1\) and the first \(\bK\) actions match \(a\sups{M}_{1:\bK}\), i.e. \(a_{1:\bK} = a\sups{M}_{1:\bK}\). For every other action sequence, the reward is sampled from  \(\Ber\prn[\big]{\tfrac{1}{2}}\). %
Thus, for any policy \(\pi\),  
\begin{align*}
f\sups{M}(\pi) = \tfrac{1}{2} + \Delta \indic\crl{\pi(s_{1:\bK}) = a\sups{M}_{1:\bK}},
\end{align*}
which implies that
\begin{align}
f\sups{M}(\pi\subs{M}) - f\sups{M}(\pi) &= \Delta(1 - \indic\crl{\pi(s_{1:\bK}) = a\sups{M}_{1:\bK}}). \label{eq:lb_mixture_tab3}
\end{align}

Finally, we define the reference model \(\Mbar\). The model \(\Mbar\)
is specified by \(\crl[\big]{\crl*{\cS_h}_{h=1}^{H+1}, \cA,
  \bbM\sups{\Mbar}}\) where \(\bbM\sups{\Mbar}\) is a tabular MDP
given by
\begin{align*}
\bbM\sups{\Mbar} = \crl[\big]{\crl*{\cS_h}_{h=1}^{H+1}, \cA,
    P\sups{\Mbar}_{h}, R_{h}\sups{\Mbar}, \delta_{s\ind{1}}}.
\end{align*}
Here, the initial state \(s_1\) is \(s\ind{1}\) deterministically, and
the transitions \(P\sups{\Mbar}_{h, k}\) and rewards
\(R\sups{\Mbar}_{h, k}\) are as follows:
\begin{enumerate}[label=(\roman*)]
\item Transitions are stochastic and independent of the chosen
  action. In particular, for each \(h \leq H\), the dynamics
  \(P\sups{\Mbar}_{h}\) are given by
\begin{align*}
 P\sups{\Mbar}_{h}\prn{s_{h+1} \mid s_{h}, a_h} = \begin{cases}
	 \frac{\bK - h}{\bK - h + 1} &\text{if} \quad h \leq \bK, s_h = s\ind{h}~\text{and}~s_{h+1} = s\ind{h+1} \\
	  \frac{1}{\bK - h + 1} & \text{if} \quad h \leq \bK, s_h = s\ind{h}~\text{and}~s_{h+1} = s\ind{\bK + 1} \\
	  1 & \text{if} \quad h \leq \bK, s_h \neq s\ind{h}~\text{and}~s_h = s_{h+1} \\
	  1 & \text{if} \quad h > \bK~\text{and}~s_h = s_{h+1} \\
	  0 &\text{otherwise} 
\end{cases}. 
\end{align*}
\item The reward distribution is given by 
\begin{align*}
R\sups{\Mbar}_{h}(s_h, a_h) = \begin{cases}
		\Ber\prn[\big]{\tfrac{1}{2}}, & \text{if} \quad h = \bK, \\	
		0, & \text{otherwise}.
\end{cases}.  
\end{align*} 
\end{enumerate}
Note that \(\Mbar\) can be thought of as a mixture of \(\bK\)
identical tabular MDPs each given by \(\bbM\sups{\Mbar}\). Note that for any policy \(\pi\), the rewards for any trajectory in \(\Mbar\) are sampled from \(\Ber\prn[\big]{\tfrac{1}{2}}\), and thus \(f\sups{\Mbar}(\pi) =  \frac{1}{2}\) which implies that  
\begin{align}
f\sups{\Mbar}(\pi\subs{\Mbar}) - f\sups{\Mbar}(\pi) &= 0. \label{eq:lb_mixture_tab4}
\end{align}

We define \(\cM'' = \cM' \cup \crl{\Mbar}\subseteq\cM\), and note that for any policy \(\pi\), the distribution over the trajectories is identical in all mixture models in \(\cM''\). However, as mentioned before, the rewards in \(\Mbar\) are sampled from  $\Ber\prn*{\frac{1}{2}}$ and for any \(M \in \cM'\), the rewards in \(M\) are sampled from $\Ber\prn[\big]{\frac{1}{2} + \frac{\Delta}{M} \indic\crl*{\pi(s_{1:\bK}) = a\sups{M}_{1:\bK}}}$. Thus, for any policy \(\pi\) and \(M \in \cM'\), 
\begin{align}
\Dhels{M(\pi)}{\Mbar(\pi)} &= \Dhels{\Ber\prn[\Big]{\tfrac{1}{2} 
+ \tfrac{\Delta}{\bK} \indic\crl*{\pi(s_{1:\bK}) = a\sups{M}_{1:\bK}}}}{\Ber\prn[\Big]{\tfrac{1}{2}}}\notag  \\  &\leq 3\frac{\Delta^2}{{\bK}^2}\cdot\indic\crl{\pi(s_{1:\bK}) = a\sups{M}_{1:\bK}}, \label{eq:lb_mixture_tab1}
\end{align} where the last line uses \pref{lem:helg_bern}.

The bounds in \pref{eq:lb_mixture_tab3}, \pref{eq:lb_mixture_tab4} and
\pref{eq:lb_mixture_tab1} together imply that the model class \(\cM''\)
is a $(\frac{\Delta}{\bK}, 3\frac{\Delta^2}{{\bK}^2}, 0)$-family in
the sense of \pref{def:hard_family_lb}, where for each \(\pi \in \Pi\)
and\(M \in \cM''\) we take
\begin{align*} 
u\sups{M}(\pi) \ldef \indic\crl{\pi(s_{1:\bK})=a\sups{M}_{1:\bK}} \qquad \text{and} \qquad v\sups{M}(\pi) \ldef  \indic\crl{\pi(s_{1:\bK})=a\sups{M}_{1:\bK}}, 
\end{align*}
with \(u\sups{\Mbar}(\pi) \ldef 1\) and \(v\sups{\Mbar}(\pi) \ldef 0\). As a result, \pref{lem:hard_family_lb} implies that 
\begin{align*}
  \comp(\cM,\Mbar) \geq \frac{\Delta}{2\bK} - \frac{3 \gamma \Delta^2}{{\bK}^2N}, 
\end{align*}
for \(N \ldef A^{\bK} + 1\). Setting \(\Delta = \frac{\bK N }{12
  \gamma}\) leads to the lower bound \(  \comp(\cM,\Mbar) \geq
\frac{N}{24 \gamma}\). We conclude by noting that all $M\in\cM''$ have
\(M \in \cM_{\veps_\gamma}\prn*{\Mbar}\) with \(\veps_\gamma =
\frac{N}{24 \gamma}\), and thus the lower bound on the \CompShort also
applies to the class \(\cM_{\veps_\gamma}\prn*{\Mbar}\). 
 \end{proof}

\begin{proof}[Proof for \preftitle{ex:tabular}] let \(\cM\) be the class of
  all tabular MDPs, and let \(\cM\ind{K}\) denote the set of all
  mixture models in which each \(M \in \cM\ind{K}
\) is a mixture of \(K\) MDPs from \(\cM\). Additionally, define \(\wt{\cM} = \conv(\cM)\), and note that \(\cM\ind{K} \subseteq \wt{\cM}\) for all \(K \geq 1\). For any \(\veps > 0\) and \(\Mbar \in \cM\ind{K}\), we have that \(\cM\ind{k}_{\veps}(\Mbar) \subseteq \wt{\cM}_{\veps}(\Mbar)\), which implies that
\begin{align*}
\comp\prn{\wt{\cM}_{\veps}(\Mbar), \Mbar} \geq \comp(\cM\ind{K}_{\veps}(\Mbar), \Mbar),
\end{align*} because
\(\comp(\cdot, \Mbar)\) is a non-decreasing function with respect to
inclusion. Using \pref{lem:tabular}, it holds that for any \(K \geq 1\)
and \(\gamma \geq  A^
{\min\crl{S - 1, H, K}}/6\), with \(\veps_\gamma \ldef {A^
{\min\crl{S - 1, H, K}}}/{24 \gamma}\), 
\begin{align*}
\comp\prn{\wt{\cM}_{\veps}(\Mbar), \Mbar} \geq \comp(\cM\ind{K}_{\veps}(\Mbar), \Mbar) \geq \frac{A^
{\min\crl{S - 1, H, K}}}{24 \gamma}.
\end{align*}
Setting \(K = S\) above gives the desired lower bound. 
\end{proof}

 }

\section{Structural Results}
\label{app:structural}

This section is organized as follows.
\begin{itemize}
\item In \pref{app:structural_background}, we recall the existing variants
  of Information Ratio  and state some basic properties. 
\item In \pref{app:structural_ir_dec}, we prove equivalence of the
  \CompText and the \irtextp with Hellinger distance
  (\pref{thm:ir_dec_basic}), as well as a generalization of this
  result (\pref{thm:dec_info_ratio}).
  \item In \pref{app:structural_eo_dec}, we prove equivalence of the
    \irtextp with Hellinger distance and the high-probability exploration-by-optimization objective.
\end{itemize}

\subsection{Background on Complexity Measures}
\label{app:structural_background}
Throughout this section, we refer to any non-negative function
$\Dgen{\cdot}{\cdot}$ that is defined over $\Delta(\cX)\times\Delta(\cX)$ for all measurable spaces
$(\cX, \filt)$ as a \emph{\dlike} function.

\paragraph{Generalized Information Ratio}
Below we recall two notions of \emph{generalized
  Information Ratio} introduced by \citet{lattimore2021mirror} and
\citet{lattimore2022minimax}, which extend the original definition
of \citet{russo2014learning,russo2018learning}.

For a given prior $\mu\in\Delta(\cM\times\Pi)$ and distribution
$p\in\Delta(\Pi)$, let $\bbP$ denote the law of the process
  $(M,\pistar)\sim\mu,\pi\sim{}p,z\sim{}M(\pi)$, and define $\ppr(\pi')\ldef\bbP(\pistar=\pi')$ and
$\ppo(\pi';\pi,z)\ldef\bbP(\pistar=\pi'\mid{}(\pi,z))$.
\begin{enumerate}
\item \citet{lattimore2021mirror} define a class $\cM$ to have generalized Information Ratio 
  $(\alpha,\beta,\lambda)$ (where $\alpha,\beta\geq{}0$ and $\lambda>1$) if for each prior
  $\mu\in\Delta(\cM\times\Pi)$, there exists a distribution
  $p\in\Delta(\Pi)$ such that
  \begin{align}
    \label{eq:info_ratio_basic}
    \En_{(M,\pistar)\sim\mu}\En_{\pi\sim{}p}\brk*{\fm(\pistar)-\fm(\pi)}
    \leq \alpha +
    \beta^{1-1/\lambda}\prn*{\En_{\pi\sim{}p}\En_{z\mid{}\pi}\brk*{\Dgen{\ppo(\cdot;\pi,z)}{\ppr}}}^{1/\lambda}.
  \end{align}
\item \citet{lattimore2022minimax} defines the generalized \irtext for a class $\cM$ (for $\lambda>1$) via
  \begin{align}
    \label{eq:info_ratio_lattimore2022}
\Psi_{\lambda}(\cM) = \sup_{\mu\in\Delta(\cM\times\Pi)}\inf_{p\in\Delta(\Pi)}\crl*{
    \frac{(\En_{(M,\pistar)\sim\mu}\En_{\pi\sim{}p}\brk*{\fm(\pistar)-\fm(\pi)})^{\lambda}}
    {\En_{\pi\sim{}p}\En_{z\mid{}\pi}\brk*{\Dgen{\ppo(\cdot;\pi,z)}{\ppr}}}
    }.
  \end{align}
\end{enumerate}
As mentioned in \pref{sec:connections}, the formulations above
slightly generalize the original versions in
\citet{lattimore2021mirror,lattimore2022minimax} by incorporating models $M\in\cM$ and considering general
distances.

The following proposition shows that boundedness of the generalized
Information Ratio  implies boundedness of the \irtextp (\pref{def:info_ratio_p}).
\begin{proposition}
  \label{prop:generalized_ir}
  Fix $\alpha,\beta\geq{}0$ and $\lambda>1$. If a class $\cM$ has generalized
  Information Ratio  $(\alpha,\beta,\lambda)$ in the sense of
  \pref{eq:info_ratio_basic}, then
  \[
    \irgens_{\gamma}(\cM) \leq \alpha + {\beta}/{\gamma^{\frac{1}{\lambda-1}}}\quad\forall{}\gamma>0.
  \]
  Likewise, the generalized Information Ratio  in
  \pref{eq:info_ratio_lattimore2022} satisfies
      \[
    \irgens_{\gamma}(\cM) \leq \prn*{\Psi_{\lambda}(\cM)/\gamma}^{\frac{1}{\lambda-1}}\quad\forall{}\gamma>0.
  \]
\end{proposition}
\begin{proof}[\pfref{prop:generalized_ir}] Suppose $\cM$ has
  generalized Information Ratio  $(\alpha,\beta,\lambda)$. Then there
  exists $p\in\Delta(\Pi)$ such that for all $\mu\in\Delta(\cM\times\Pi)$,
  \begin{align*}
    \En_{(M,\pistar)\sim\mu}\En_{\pi\sim{}p}\brk*{\fm(\pistar)-\fm(\pi)}        &\leq \alpha +
    \beta^{1-1/\lambda}\prn*{\En_{\pi\sim{}p}\En_{z\mid{}\pi}\brk*{\Dgen{\ppo(\cdot;\pi,z)}{\ppr}}}^{1/\lambda}
    \\
    &\leq \alpha +
      \frac{\beta}{\gamma^{\frac{1}{\lambda-1}}} + \gamma\cdot\En_{\pi\sim{}p}\En_{z\mid{}\pi}\brk*{\Dgen{\ppo(\cdot;\pi,z)}{\ppr}},
  \end{align*}
  where we have applied Young's inequality, which gives that
  $xy\leq{}\frac{\lambda-1}{\lambda}x^{\frac{\lambda}{\lambda-1}} +
  \frac{1}{\lambda}y^{\lambda}$ for $x,y\geq{}0$.

  For the second result, we use that the definition of
  $\Psi_{\lambda}(\cM)$ implies generalized Information Ratio  $(0,(\Psi_{\lambda}(\cM))^{\frac{1}{\lambda-1}},\lambda)$.
\end{proof}

This results show that an upper bound in terms of the \irtextp in
\pref{def:info_ratio_p} implies an upper bound in terms of either
version of the generalized Information Ratio. It is also
straightforward to see that generalized Information Ratio 
$(0,\beta,\lambda)$ in \pref{eq:info_ratio_basic} implies that
$\Psi_{\lambda}(\cM)\leq\beta^{\lambda-1}$ and vice-versa. Note that $\alpha=0$ is the most interesting regime, as the regret
bounds in \citet{lattimore2021mirror} scale with $\alpha\cdot{}T$ when
$\alpha>0$.

Another important property
of the \irtextp (as well as the generalized Information Ratio) is that
it is invariant under convexification. 
\inforatioconvex*
\begin{proof}[\pfref{prop:info_ratio_convex}]%
  \newcommand{\mutil}{\wt{\mu}}
  Fix $\mu\in\Delta(\conv(\cM)\times\Pi)$. We can represent any
  $\Mbar\in\conv(\cM)$ as a mixture $\nu\in\Delta(\cM)$, so that
  $\Mbar=\En_{M\sim\nu}\brk*{M}$. Let
  $\mutil\in\Delta(\Delta(\cM)\times\Pi)$ be such that the process
  $(\nu,\pistar)\sim\mutil$, $\Mbar=\En_{M\sim\nu}\brk*{M}$ has the
  same law as $(\Mbar,\pistar)\sim\mutil$. Finally, let
  $\mu'\in\Delta(\cM\times\Pi)$ be the law of $(M,\pistar)$ induced by sampling
  $(\nu,\pistar)\sim\mutil$ and $M\sim\nu$.

  We observe that for any distribution $p\in\Delta(\Pi)$, 
  \begin{align*}
    \En_{(\Mbar,\pistar)\sim\mu}\En_{\pi\sim{}p}\brk*{\fmbar(\pistar)-\fmbar(\pi)}
     &=\En_{(\nu,\pistar)\sim\mutil}\En_{\pi\sim{}p}\En_{M\sim\nu}\brk*{\fm(\pistar)-\fm(\pi)}
    \\
    &=\En_{(M,\pistar)\sim\mu'}\En_{\pi\sim{}p}\brk*{\fm(\pistar)-\fm(\pi)}.
  \end{align*}
  Next, observe that $(\pi,\pistar,z)$ are identically distributed under the
  processes $\pi\sim{}p$, $(\Mbar,\pistar)\sim\mu$, $z\sim\Mbar(\pi)$
  and $\pi\sim{}p$, $(M,\pistar)\sim\mu'$, $z\sim{}M(\pi)$. As a result,
  we have $\mupr=\mupr'$ and $\mupo=\mupo'$, so
  \begin{align*}
    \En_{\pi\sim{}p}\En_{z\mid{}\pi}\brk*{\Dgen{\ppo(\cdot;\pi,z)}{\ppr}}
    = \En_{\pi\sim{}p}\En_{z\mid{}\pi}\brk*{\Dgen{\ppo'(\cdot;\pi,z)}{\ppr'}}.
  \end{align*}
This establishes that $\irgens_{\gamma}(\conv(\cM))\leq\irgens_{\gamma}(\cM)$; the other
direction is trivial.
\end{proof}

\subsection{\CompText and Information Ratio}
\label{app:structural_ir_dec}

\irdecbasic*

\pref{thm:ir_dec_basic} is a special case of the following theorem,
which concerns general \dlike functions.

\begin{theorem}
  \label{thm:dec_info_ratio}
  Let $\Dgen{\cdot}{\cdot}$ be any \dlike
  function defined over $\Delta(\Pi)\times\Delta(\Pi)$ and
  $\Delta(\cZ)\times\Delta(\cZ)$, for which there exists constants $c_1,c_2\geq{}1$ such that:
  \begin{enumerate}
  \item For all $\bbQ\in\Delta(\cZ)$, $\bbP\mapsto{}\Dgen{\bbP}{\bbQ}$ is convex.
  \item For all pairs of random variables $(\pi,z)\in\Pi\times\cZ$,
    \[
      \En_{\pi\sim\bbP_{\pi}}\brk*{\Dgen{\bbP_{z\mid{}\pi}}{\bbP_z}}
    \leq{} c_1\cdot{}\En_{z\sim\bbP_{z}}\brk*{\Dgen{\bbP_{\pi\mid{}z}}{\bbP_\pi}}
  \]
  and
      \[
        \En_{z\sim\bbP_{z}}\brk*{\Dgen{\bbP_{\pi\mid{}z}}{\bbP_\pi}}
        \leq c_1\cdot \En_{\pi\sim\bbP_{\pi}}\brk*{\Dgen{\bbP_{z\mid{}\pi}}{\bbP_z}}.
  \]
\item For all pairs of random variables $(\pi,z)\in\Pi\times\cZ$,
  \[
    \En_{\pi\sim\bbP_{\pi}}\brk*{\Dgen{\bbP_{z\mid{}\pi}}{\bbP_z}} \leq{}
    c_2\cdot{}\inf_{\bbQ}\En_{\pi\sim\bbP_{\pi}}\brk*{\Dgen{\bbP_{z\mid{}\pi}}{\bbQ}}.
  \] 
\item For all $\veps>0$ sufficiently small, and all $\bbQ\in\Delta(\cZ)$, there exists $\bbQ'\in\Delta(\cZ)$ such that
    $\Dgen{\bbP}{\bbQ}\geq{}\Dgen{\bbP}{\bbQ'}-\veps$ and $\sup_{\bbP\in\Delta(\cZ)}\Dgen{\bbP}{\bbQ'}<\infty$.
\end{enumerate}
Then we have
\begin{align}
  \label{eq:dec_info_ratio}
\irgens_{c_1\gamma}(\cM) \leq   \compgen(\conv(\cM)) \leq{}\irgens_{\gamma/c_1c_2}(\cM).
\end{align}
\end{theorem}
All $f$-divergences satisfy Property 2 with $c_1=1$, but may not
satisfy Property 3. On the other hand, Bregman
divergences\footnote{Recall that for a convex set $\cX$ and
  regularizer $\cR:\cX\to\bbR$, $\Dbreg{x}{y}\ldef\cR(x) - \cR(y) -
  \tri*{\grad\cR(y),x-y}$ is the associated Bregman divergence.}
satisfy Property 3 with $c_2=1$, but may not satisfy Property
2 (consider squared euclidean distance). KL-divergence, being both an $f$-divergence and a Bregman
divergence, satisfies both properties with $c_1=c_2=1$ (this fact has
been used tacitly in many prior works). Squared Hellinger distance is
an $f$-divergence but not a Bregman divergence, yet satisfies Property
3 with $c_2=4$ as a consequence of the triangle inequality.

\begin{proof}[\pfref{thm:dec_info_ratio}] We first bound the
  \CompShort by the Information Ratio, then proceed to bound the
  Information Ratio  by the \CompShort.
  
  \paragraph{Bounding the \CompShort by the Information Ratio}
  Fix any reference model $M'\in\cM$, $\veps>0$ and let $M''$ be such that
  $\Dhels{\cdot}{M'(\pi)}\geq{}\Dhels{\cdot}{M''(\pi)}-\veps$ and
  $\Dhels{\cdot}{M''(\pi)}<\infty$ ($M''$ can be obtained by applying property 4 uniformly for all $\pi$). Using the minimax theorem (\pref{lem:sion}), we have
  \begin{align*}
    \compgen(\cM,M')
        &\leq{}
          \inf_{p\in\Delta(\Pi)}\sup_{M\in\cM}\En_{\pi\sim{}p}\brk*{
    \fm(\pim) - \fm(\pi)
    - \gamma\cdot{}\Dgen{M(\pi)}{M''(\pi)}
          } + \gamma\veps \\
              &=
                \inf_{p\in\Delta(\Pi)}\sup_{\nu\in\Delta(\cM)}\En_{\pi\sim{}p}\En_{M\sim\nu}\brk*{
    \fm(\pim) - \fm(\pi)
    - \gamma\cdot{}\Dgen{M(\pi)}{M''(\pi)}
    } + \gamma\veps \\
    &= 
    \sup_{\nu\in\Delta(\cM)}\inf_{p\in\Delta(\Pi)}\En_{\pi\sim{}p}\En_{M\sim\nu}\brk*{
    \fm(\pim) - \fm(\pi)
    - \gamma\cdot{}\Dgen{M(\pi)}{M''(\pi)}
    } + \gamma\veps.
  \end{align*}
Note that the application of the minimax theorem is admissible here,
since $\Delta(\Pi)$ is compact (a consequence of finiteness of $\Pi$)
and the objective value is bounded (a consequence of the choice of
$M''$ and the fact that $\fm\in\brk*{0,1}$).
  
Fix any $\nu\in\Delta(\cM)$, and define $\mu\in\Delta(\cM\times\Pi)$ be the induced law of $(M,\pim)$.
  Define $\Mbar_{\pi'}(\pi)=\En_{M\sim\nu}\brk*{M(\pi)\mid{}\pim=\pi'}$
  and $\Mbar(\pi) = \En_{M\sim\mu}\brk*{M(\pi)} = \En_{\pistar\sim\mu}\brk*{M_{\pistar}(\pi)}$.
  Then, for any
  distribution $p\in\Delta(\Pi)$,
  \begin{align*}
    \hspace{1.5in}  & \hspace{-1.5in} \En_{M\sim\nu}\En_{\pi\sim{}p}\brk*{
    \fm(\pim) - \fm(\pi)
    - \gamma\cdot{}\Dgen{M(\pi)}{M''(\pi)}
    } \\
    &=
      \En_{(M,\pistar)\sim\mu}\En_{\pi\sim{}p}\brk*{
      \fm(\pistar) - \fm(\pi)
    - \gamma\cdot{}\Dgen{M(\pi)}{M''(\pi)}
      }\\
        &\leq{}
          \En_{(M,\pistar)\sim\mu}\En_{\pi\sim{}p}\brk*{
          \fm(\pistar) - \fm(\pi)
    - \gamma\cdot{}\Dgen{\Mbar_{\pistar}(\pi)}{M''(\pi)}
          } \\
            &\leq{}
              \En_{(M,\pistar)\sim\mu}\En_{\pi\sim{}p}\brk*{
                        \fm(\pistar) - \fm(\pi)
              - \frac{\gamma{}}{c_2} \cdot{}\Dgen{\Mbar_{\pistar}(\pi)}{\Mbar(\pi)}
    },
  \end{align*}
  where the first inequality uses convexity of \(D\) in the first argument (Property 1),
  and the second inequality uses Property 3. 
  To proceed, let $\bbP$ be the law of the process $\pi\sim{}p$,
  $(M,\pistar)\sim\mu$, $z\sim{}M(\pi)$, and define $\ppr(\pi')\ldef\bbP(\pistar=\pi')$ and
  $\ppo(\pi';\pi,z)\ldef\bbP(\pistar=\pi'\mid{}(\pi,z))$. Observe that $\Mbar_{\pistar}(\pi)=\bbP_{z\mid{}\pi,\pistar}$ and
  $\Mbar(\pi)=\bbP_{z\mid{}\pi}$. Hence, using Property 2, it holds
  that for all $\pi$,
  \begin{align*}
    \En_{\pistar\sim\nu}\brk*{\Dgen{\Mbar_{\pistar}(\pi)}{\Mbar(\pi)}}
    \geq{}
    \frac{1}{c_1} \En_{z\mid\pi}\brk*{\Dgen{\bbP_{\pistar\mid{}\pi,z}}{\bbP_{\pistar\mid{}\pi}}}
    =
    \frac{1}{c_1} \En_{z\mid\pi}\brk*{\Dgen{\bbP_{\pistar\mid{}\pi,z}}{\bbP_{\pistar}}}, 
  \end{align*}
  where the last equality holds because $\pi$ and $\pistar$ are
  independent (marginally). Since
  $\Dgen{\bbP_{\pistar\mid{}\pi,z}}{\bbP_{\pistar}}=\Dgen{\mupo(\cdot;\pi,z)}{\mupr}$,
  if we choose $p$ to attain the minimum in \pref{eq:info_ratio_p} for
  $\mu$ we are guaranteed that
  \begin{align*}
    &\En_{M\sim\nu}\En_{\pi\sim{}p}\brk*{
    \fm(\pim) - \fm(\pi)
    - \gamma\cdot{}\Dgen{M(\pi)}{M'(\pi)}
      } \\
    &\leq{} 
      \En_{(M,\pistar)\sim\mu}\En_{\pi\sim{}p}\brk*{
                        \fm(\pistar) - \fm(\pi)}
              - \frac{\gamma{}}{c_1c_2} \cdot{}\En_{\pi\sim{}p}\En_{z\mid\pi}\brk*{\Dgen{\mupo(\cdot;\pi,z)}{\mupr}
      } +\gamma\veps  \\
    &\leq{}
      \irgens_{\gamma/c_1c_2}(\cM)+\gamma\veps. 
  \end{align*}
  Taking $\veps\to{}0$, we conclude that $\compgen(\cM) \leq{}\irgens_{\gamma/c_1c_2}(\cM)$. By \pref{prop:info_ratio_convex},
  $\irgen(\cM)=\irgen(\conv(\cM))$, so applying the result to
  $\conv(\cM)$ yields
  \[
    \compgen(\conv(\cM)) \leq{}\irgens_{\gamma/c_1c_2}(\cM).
  \]

  \paragraph{Bounding the Information Ratio  by the \CompShort}
  We now consider the opposite direction. Fix a prior
  $\mu\in\Delta(\cM\times\Pi)$ and consider the value for the \irtextp:
  \[
    \En_{(M,\pistar)\sim\mu}\En_{\pi\sim{}p}\brk*{
                        \fm(\pistar) - \fm(\pi)}
              - \gamma{}\cdot{}\En_{\pi\sim{}p}\En_{z\mid{}\pi}\brk*{\Dgen{\mupo(\cdot;\pi,z)}{\mupr}
            },
          \]
          with $\mupr$ and $\mupo$ defined as before. Define
          $\Mbar_{\pi'}(\pi)\ldef\En_{\mu}\brk*{M(\pi)\mid\pistar=\pi'}$
          and $\Mbar(\pi)=\En_{M\sim\mu}\brk*{M(\pi)}$.
Using that $\pistar$ and $\pi$ are independent (marginally), along with Property 2, we
have
\begin{align*}
  \En_{z\mid{}\pi}\brk*{\Dgen{\mupo(\cdot;\pi,z)}{\mupr}} 
  &=
    \En_{z\mid\pi}\brk*{\Dgen{\bbP_{\pistar\mid{}\pi,z}}{\bbP_{\pistar}}}\\
  &=\En_{z\mid\pi}\brk*{\Dgen{\bbP_{\pistar\mid{}\pi,z}}{\bbP_{\pistar\mid{}\pi}}} 
  \geq{}
    \frac{1}{c_1} \En_{\pistar\sim\mu}\brk*{\Dgen{\Mbar_{\pistar}(\pi)}{\Mbar(\pi)}}.
\end{align*}
Next, observe that
\begin{align*}
  \En_{(M,\pistar)\sim\mu}\En_{\pi\sim{}p}\brk*{\fm(\pistar) -
  \fm(\pi)}
  &= \En_{\pi\sim{}p}\En_{\pistar\sim\mu}\En\brk*{\fm(\pistar) -
    \fm(\pi)\mid\pistar}\\
  &=
    \En_{\pi\sim{}p}\En_{\pistar\sim\mu}\brk*{f\sups{\Mbar_{\pistar}}(\pistar)
    - f\sups{\Mbar_{\pistar}}(\pi)}\\
  &\leq{}
    \En_{\pistar\sim\mu}\En_{\pi\sim{}p}\brk*{\max_{\pi'}f\sups{\Mbar_{\pistar}}(\pi')
    - f\sups{\Mbar_{\pistar}}(\pi)}.
\end{align*}
  Recall that the
  definition of $\comp(\conv(\cM))$ implies the following: For any $\kappa\in\Delta(\cM)$ there exists a
  distribution $p\in\Delta(\Pi)$ such that for all
  $\nu\in\Delta(\cM)$, defining
  $\Mbar_{\kappa}(\pi)\ldef\En_{M\sim\kappa}\brk*{M(\pi)}$ and $\Mbar_{\nu}(\pi)\ldef\En_{M\sim\nu}\brk*{M(\pi)}$, we have
    \begin{align}
    \label{eq:dec_general}
    \En_{\pi\sim{}p}\brk*{
    \max_{\pi'}f\sups{\Mbar_{\nu}}(\pi')
    - f\sups{\Mbar_{\nu}}(\pi)
      - \gamma{\cdot\Dgen{\Mbarnu(\pi)}{\Mbarkap(\pi)}}
    } \leq \comp(\conv(\cM)).
    \end{align}
    By invoking \pref{eq:dec_general} with $\Mbarkap=\Mbar$ and
$\Mbarnu=\Mbar_{\pistar}$, we are guaranteed that for every draw of $\pistar$, 
\[ 
\En_{\pi\sim{}p}\brk*{\max_{\pi'}f\sups{\Mbar_{\pistar}}(\pi') - 
  f\sups{\Mbar_{\pistar}}(\pi)}
\leq{}\frac{\gamma{}}{c_1} \cdot \En_{\pi\sim{}p}\brk*{\Dgen{\Mbar_{\pistar}(\pi)}{\Mbar(\pi)}}
+ \comp[\gamma/c_1](\conv(\cM)).
\] 
Taking the expectation over $\pistar\sim\mu$, we conclude that
\[
\irgens_{\gamma}(\cM) \leq  \comp[\gamma/c_1](\conv(\cM)).
\]

\end{proof}

\subsection{High-Probability Exploration-By-Optimization and
  Information Ratio}
\label{app:structural_eo_dec}

\eotodec*

\begin{proof}[\pfref{thm:eo_to_dec}]
We first prove the upper bound. We divide the proof into two parts as follows: 
  \paragraph{Upper bound: Minimax theorem}
  We first use the minimax theorem to move to a Bayesian counterpart
  to the \exotext objective. This requires some care to ensure
  boundedness and compactness, but otherwise is conceptually straightforward.
  To begin, observe that we can write the \exotext objective as
  \begin{align*}
    \eostar_{\eta}(\cM) &=
\sup_{q\in\Delta(\Pi)}\inf_{p\in\Delta(\Pi),g\in\cG}\sup_{M\in\cM,\pistar\in\Pi}\brk*{\exoval_{q,\eta}(p,g\midsem{}\pistar,M)}
    \\
    &= \sup_{q\in\Delta(\Pi)}\inf_{p\in\Delta(\Pi),g\in\cG}\sup_{\mu\in\Delta(\cM\times\Pi)}\En_{(M,\pistar)\sim\mu}\brk*{\exoval_{q,\eta}(p,g\midsem{}\pistar,M)}.
  \end{align*}
  Fix $\alpha\geq{}1\vee\eta^{-1}$ and $\veps\in(0,1)$, and define
  \[\cG_{\alpha}=\crl*{g\in\cG\mid{}\nrm*{g}_{\infty}\leq{}\alpha},\mathand \cP_{\veps}=\crl*{p\in\Delta(\Pi)\mid{}p(\pi)\geq\veps\abs{\Pi}^{-1}\;\;\forall\pi}.\]
Then by restricting to these classes, we have\footnote{Restricting to
  these sets allows us to enforce boundedness and continuity of
  the \exotext objective, which is necessary to appeal to the minimax
  theorem. The parameters $\alpha$ and $\veps$ will not enter
the final bound quantitatively.}
  \begin{align*}
    \eostar_{\eta}(\cM)
    \leq{} \sup_{q\in\Delta(\Pi)}\inf_{p\in\cP_{\veps},g\in\cG_{\alpha}}\sup_{\mu\in\Delta(\cM\times\Pi)}\En_{(M,\pistar)\sim\mu}\brk*{\exoval_{q,\eta}(p,g\midsem{}\pistar,M)}.
  \end{align*}
  We now verify that the conditions required to apply the minimax theorem
  are satisfied.
  \begin{itemize} 
  \item The map
    $\mu\mapsto{}\En_{(M,\pistar)\sim\mu}\brk*{\exoval_{q,\eta}(p,g\midsem{}\pistar,M)}$
    is linear. Furthermore, by \pref{prop:eo_convex} (given below), the map 
    $(p,g)\mapsto{}\En_{(M,\pistar)\sim\mu}\brk*{\exoval_{q,\eta}(p,g\midsem{}\pistar,M)}$
    is jointly convex.
  \item 
    Since we have
    restricted to $p\in\cP_{\veps}$ and $g\in\cG_{\alpha}$, the value
    $\Gamma_{q,\eta}(p,g;\pistar,M)$ is uniformly bounded, as well as
    continuous with respect to $p$ and $g$ (so long as
    $\veps>0$ and $\alpha<\infty$).
  \item The set $\Delta(\cM\times\Pi)$ is convex. Since
    $\abs{\Pi}<\infty$, the set $\cP_{\veps}\times\cG_{\alpha}$ is
    convex and compact (for $\cP_{\veps}$ equipped with the usual
    topology and $\cG_{\alpha}$ equipped with the product topology;
    see \citet{lattimore2021mirror} for details).
  \end{itemize}
Hence, using \pref{lem:sion} we can bound by the value of the Bayesian
game as follows:
\begin{align}
  \label{eq:exo_bayes}
  \eostar_{\eta}(\cM) \leq
\sup_{q\in\Delta(\Pi)}\sup_{\mu\in\Delta(\cM\times\Pi)}\inf_{p\in\cP_{\veps},g\in\cG_{\alpha}}\En_{(M,\pistar)\sim\mu}\brk*{\exoval_{q,\eta}(p,g\midsem{}
  \pistar,M)}.
\end{align}
\paragraph{Upper bound: Moving to Hellinger distance}
For any $q\in\Delta(\Pi)$, $\mu\in\Delta(\cM\times\Pi)$, and
$p\in\cP_{\veps}$, the value of the game in \pref{eq:exo_bayes} is %
\begin{align*}
  &\En_{(M,\pistar)\sim\mu}\En_{\pi\sim{}p}\brk*{\fm(\pistar)-\fm(\pi)}\\
  &~~~~+ \frac{1}{\eta} \inf_{g\in\cG_{\alpha}}\En_{(M,\pistar)\sim\mu}\brk*{
    \En_{\pi\sim{}p,z\sim{}M(\pi)}
\En_{\pi'\sim{}q}\exp\prn*{
    \frac{\eta}{p(\pi)}\prn*{
    g(\pi'\midsem\pi,z)
  - g(\pistar\midsem\pi,z)
  }
  }-1
  }.
\end{align*}
Let $\bbP$ be the law of the process
  $(M,\pistar)\sim\mu,\pi\sim{}p,z\sim{}M(\pi)$, and define $\ppr(\pi')\ldef\bbP(\pistar=\pi')$ and
  $\ppo(\pi';\pi,z)\ldef\bbP(\pistar=\pi'\mid{}(\pi,z))$. Using Bayes' rule, we can rewrite the second term above as
\begin{align*}
    \inf_{g\in\cG_{\alpha}}
  \En_{\pi\sim{}p}\En_{z\mid\pi}\brk*{
  \En_{\pi'\sim{}q}\brk*{\exp\prn*{
    \eta
    \frac{g(\pi'\midsem{}\pi,z)}{p(\pi)}
    }}
    \cdot \En_{\pistar\sim\mupo(\cdot\midsem{}\pi,z)}\brk*{\exp\prn*{
    -\eta \frac{g(\pistar\midsem\pi,z)}{p(\pi)}
  }
    }
    -1
    }
\end{align*}
By reparameterizing via
$g(\pi'\midsem{}\pi,z)\gets{}\frac{p(\pi)}{\eta}g(\pi'\midsem{}\pi,z)$,
the value is upper bounded by
\begin{align*}
      \inf_{g\in\cG_{\alpha\eta}}
  \En_{\pi\sim{}p}\En_{z\mid\pi}\brk*{
  \En_{\pi'\sim{}q}\brk*{\exp\prn*{
  g(\pi'\midsem{}\pi,z)
    }}
    \cdot \En_{\pistar\sim\mupo(\cdot\midsem{}\pi,z)}\brk*{\exp\prn*{
    -g(\pistar\midsem\pi,z)
  }
    }
    -1
    }.
\end{align*}
Furthermore, by skolemizing, we can rewrite this as
\begin{align*}
V(p,q,\mu) \ldef{} \En_{\pi\sim{}p}\En_{z\mid\pi}  \inf_{g:\Pi\to\bbR,\nrm*{g}_{\infty}\leq\alpha\eta}\crl*{
  \En_{\pi'\sim{}q}\brk*{\exp\prn*{
  g(\pi')
    }}
    \cdot \En_{\pistar\sim\mupo(\cdot\midsem\pi,z)}\brk*{\exp\prn*{
    -g(\pistar)
  }
    }
    -1
    }.
\end{align*}
We now appeal to \pref{lem:hellinger_exp2}, which grants that
\begin{align}
  \label{eq:hellinger_exp_exo}
V(p,q,\mu) \leq{}
  -\frac{1}{2}\En_{\pi\sim{}p}\En_{z\mid\pi}\brk*{\Dhels{\mupo(\cdot\midsem{}\pi,z)}{q}}
  + 4e^{-\alpha\eta}.
\end{align}
Using \pref{eq:hellinger_exp_exo}, we can upper bound the value as
\begin{align*}
  \exostar_{\eta}(\cM) &\leq{}
  \sup_{q\in\Delta(\Pi)}\sup_{\mu\in\Delta(\cM\times\Pi)}\inf_{p\in\cP_{\veps}}\crl*{
  \En_{(M,\pistar)\sim\mu}\En_{\pi\sim{}p}\brk*{\fm(\pistar)-\fm(\pi)}
  - \frac{1}{2\eta}\En_{\pi\sim{}p}\En_{z\mid\pi}\brk*{\Dhels{\mupo(\cdot\midsem{}\pi,z)}{q}}
  } 
    + \frac{4}{\eta} e^{-\alpha\eta}.
\end{align*}
In addition, since $\fm\in\brk*{0,1}$ and
$\Dhels{\cdot}{\cdot}\in\brk*{0,2}$, we can further bound by
\begin{align*}
  &\sup_{q\in\Delta(\Pi)}\sup_{\mu\in\Delta(\cM\times\Pi)}\inf_{p\in\Delta(\Pi)}\crl*{
  \En_{(M,\pistar)\sim\mu}\En_{\pi\sim{}p}\brk*{\fm(\pistar)-\fm(\pi)}
  - \frac{1}{2\eta}\En_{\pi\sim{}p}\En_{z\mid\pi}\brk*{\Dhels{\mupo(\cdot\midsem{}\pi,z)}{q}}
    } \\
  &\hspace{3.0in} + \bigoh(\eta^{-1}e^{-\alpha\eta} + \veps\cdot{}(1+\eta^{-1})).
\end{align*}
Since this expression only depends on $\alpha$ and $\veps$ through the
additive approximation terms, taking the limit as $\alpha\to\infty$
and $\veps\to{}0$ yields
\begin{align*}
  \exostar_{\eta}(\cM) 
  \leq{}
  \sup_{q\in\Delta(\Pi)}\sup_{\mu\in\Delta(\cM\times\Pi)}\inf_{p\in\Delta(\Pi)}\crl*{
  \En_{(M,\pistar)\sim\mu}\En_{\pi\sim{}p}\brk*{\fm(\pistar)-\fm(\pi)}
  - \frac{1}{2\eta}\En_{\pi\sim{}p}\En_{z\mid\pi}\brk*{\Dhels{\mupo(\cdot\midsem{}\pi,z)}{q}}
  }.
\end{align*}
Finally, recall that since Hellinger distance satisfies the triangle
inequality,
\[
  \En_{\pi\sim{}p}\En_{z\mid{}\pi}\brk*{\Dhels{\ppo(\cdot\midsem{}\pi,z)}{\ppr}}
  \leq{} 2\En_{\pi\sim{}p}\En_{z\mid{}\pi}\brk*{\Dhels{\ppo(\cdot\midsem{}\pi,z)}{q}}
   + 2\Dhels{\ppr}{q}.
 \]
 Using that
 $\ppr(\pi')=\En_{\pi\sim{}p}\En_{z\mid{}\pi}\brk*{\ppo(\pi'\midsem{}\pi,z)}$
 and that squared Hellinger distance is convex, we have $\Dhels{\ppr}{q}\leq{}\En_{\pi\sim{}p}\En_{z\mid{}\pi}\brk*{\Dhels{\ppo(\cdot\midsem{}\pi,z)}{q}}$,
 and so
 \[
   \En_{\pi\sim{}p}\En_{z\mid{}\pi}\brk*{\Dhels{\ppo(\cdot\midsem{}\pi,z)}{\ppr}}
   \leq{} 4\cdot{}\En_{\pi\sim{}p}\En_{z\mid{}\pi}\brk*{\Dhels{\ppo(\cdot\midsem{}\pi,z)}{q}}.
 \]
 It follows that
 \begin{align*}
   \exostar_{\eta}(\cM) 
   &\leq{}
  \sup_{\mu\in\Delta(\cM\times\Pi)}\inf_{p\in\Delta(\Pi)}\crl*{
     \En_{(M,\pistar)\sim\mu}\En_{\pi\sim{}p}\brk*{\fm(\pistar)-\fm(\pi)}
  - \frac{1}{8\eta}\En_{\pi\sim{}p}\En_{z\mid\pi}\brk*{\Dhels{\mupo(\cdot\midsem{}\pi,z)}{\mupr}}
   } \\
   &= \irHels_{1/8\eta}(\cM).
 \end{align*}
 \paragraph{Lower bound}
It is immediate (without invoking the minimax theorem) that
\begin{align*}
    \eostar_{\eta}(\cM) &= \sup_{q\in\Delta(\Pi)}\inf_{p\in\Delta(\Pi),g\in\cG}\sup_{\mu\in\Delta(\cM\times\Pi)}\En_{(M,\pistar)\sim\mu}\brk*{\exoval_{q,\eta}(p,g\midsem{}
  \pistar,M)}\\
&\geq{}  \sup_{q\in\Delta(\Pi)}\sup_{\mu\in\Delta(\cM\times\Pi)}\inf_{p\in\Delta(\Pi),g\in\cG}\En_{(M,\pistar)\sim\mu}\brk*{\exoval_{q,\eta}(p,g\midsem{}
  \pistar,M)}.
\end{align*}
Performing the same sequence of calculations as in the upper bound, it
holds that for any $q\in\Delta(\Pi)$, $\mu\in\Delta(\cM\times\Pi)$, and
$p\in\Delta(\Pi)$,
\begin{align*}
  &\inf_{g\in\cG}\En_{(M,\pistar)\sim\mu}\brk*{\exoval_{q,\eta}(p,g\midsem{}
  \pistar,M)} \\
  &=\En_{(M,\pistar)\sim\mu}\En_{\pi\sim{}p}\brk*{\fm(\pistar)-\fm(\pi)}\\
  &~~~~+ \eta^{-1}\inf_{g\in\cG}\En_{(M,\pistar)\sim\mu}\brk*{
    \En_{\pi\sim{}p,z\sim{}M(\pi)}
\En_{\pi'\sim{}q}\exp\prn*{
    \frac{\eta}{p(\pi)}\prn*{
    g(\pi'\midsem\pi,z)
  - g(\pistar\midsem\pi,z)
  }
  }-1
    } \\
    &=\En_{(M,\pistar)\sim\mu}\En_{\pi\sim{}p}\brk*{\fm(\pistar)-\fm(\pi)}
  + \eta^{-1}\En_{\pi\sim{}p}\En_{z\mid\pi}  \inf_{g\in\cG}\crl*{
  \En_{\pi'\sim{}q}\brk*{\exp\prn*{
  g(\pi')
    }}
    \cdot \En_{\pistar\sim\mupo(\cdot\midsem\pi,z)}\brk*{\exp\prn*{
    -g(\pistar)
  }
    }
    -1
    },
\end{align*}
where we define $\mupr$ and $\mupo$ as in the prequel. Using \pref{lem:hellinger_exp} yields
\begin{align*}
\En_{\pi\sim{}p}\En_{z\mid\pi}  \inf_{g\in\cG}\crl*{
  \En_{\pi'\sim{}q}\brk*{\exp\prn*{
  g(\pi')
    }}
    \cdot \En_{\pistar\sim\mupo(\cdot\midsem\pi,z)}\brk*{\exp\prn*{
    -g(\pistar)
  }
    }
    -1
  }
  \geq{} -\En_{\pi\sim{}p}\En_{z\mid\pi}\brk*{\Dhels{\mupo(\cdot\midsem{}\pi,z)}{q}}.
\end{align*}
We conclude that
 \begin{align*}
     \exostar_{\eta}(\cM)
   &\geq
  \sup_{q\in\Delta(\Pi)}\sup_{\mu\in\Delta(\cM\times\Pi)}\inf_{p\in\Delta(\Pi)}\crl*{
     \En_{(M,\pistar)\sim\mu}\En_{\pi\sim{}p}\brk*{\fm(\pistar)-\fm(\pi)}
  - \frac{1}{\eta}\En_{\pi\sim{}p}\En_{z\mid\pi}\brk*{\Dhels{\mupo(\cdot\midsem{}\pi,z)}{q}}
     } \\
      &\geq
        \sup_{\mu\in\Delta(\cM\times\Pi)}\inf_{p\in\Delta(\Pi)}\crl*{
        \En_{(M,\pistar)\sim\mu}\En_{\pi\sim{}p}\brk*{\fm(\pistar)-\fm(\pi)}
  - \frac{1}{\eta}\En_{\pi\sim{}p}\En_{z\mid\pi}\brk*{\Dhels{\mupo(\cdot\midsem{}\pi,z)}{\mupr}}
        }\\
   &= \irHels_{1/\eta}(\cM).
 \end{align*}

\end{proof}

  \begin{lemma}
    \label{prop:eo_convex}
    For any fixed $M\in\cM$ and $\pistar\in\Pi$, the map
    $(p,g)\mapsto\exoval_{q,\eta}(p,g\midsem{}\pistar,M)$ is jointly convex with respect
    to $(p,g)\in\Delta(\Pi)\times\cG$, where $\cG\ldef(\Pi\times\Pi\times\cZ\to\bbR)$.
  \end{lemma}

  \begin{proof}[\pfref{prop:eo_convex}]
    Let $M\in\cM$ and $\pistar\in\Pi$ be fixed. The map
    $p\mapsto{}\En_{\pi\sim{}p}\brk*{\fm(\pim)-\fm(\pi)}$ is linear, so our main task
    is to show that the function
    \begin{align*}
      (p,g)\mapsto{} \sum_{\pi}p(\pi)\En_{z\sim{}M(\pi)}\brk*{
  \sum_{\pi'}q(\pi')\exp\prn*{
      \frac{\eta}{p(\pi)}\prn*{
      g(\pi'\midsem\pi,z)
      - g(\pistar\midsem\pi,z)
  }
  }
  }
    \end{align*}
    is jointly convex. We can rewrite this as
    \begin{align*}
      \sum_{\pi}q(\pi') \sum_{\pi}p(\pi)\En_{z\sim{}M(\pi)}\brk*{\exp\prn*{
      \frac{\eta}{p(\pi)}\prn*{
      g(\pi'\midsem\pi,z)
  - g(\pistar\midsem\pi,z)
  }
  }
  }.
    \end{align*}
    Since convexity is preserved under summation with non-negative weights, it suffices to show that
    for any fixed $(\pi,\pi')$, the map
    \begin{align}
      \label{eq:eo_convex0}
      (p(\pi),g) \mapsto p(\pi)\En_{z\sim{}M(\pi)}\brk*{\exp\prn*{
      \frac{\eta}{p(\pi)}\prn*{
      g(\pi'\midsem\pi,z)
  - g(\pistar\midsem\pi,z)
  }
  }
  }
    \end{align}
    is convex. Since the function $g\mapsto{}\En_{z\sim{}M(\pi)}\brk*{\exp\prn*{
  \eta\prn*{
    g(\pi'\midsem\pi,z)
  - g(\pistar\midsem\pi,z)
  }
  }
}$ is convex over $\cG$, convexity for \pref{eq:eo_convex0} follows
  from the following standard result.%
  \end{proof}
  
      \begin{proposition}[Convexity of perspective transformation]
      \label{prop:perspective}
      Let $f:\bbR^{d}\to(-\infty,\infty)$ be a convex function. Then
      the function
      \[
        (x,t)\mapsto{} t\cdot{}f(x/t)
      \]
      is convex over $\bbR^{d}\times\bbR_{+}$.
    \end{proposition}

\section{Proofs for Main Results (\preftitle{sec:main})}
\label{app:main}

\subsection{Proof of \preftitle{thm:upper_main}}

\uppermain*

\begin{proof}[\pfref{thm:upper_main}]
  Let us adopt convention $\tri*{p,f}=\sum_{\pi}p(\pi)\cdot{}f(\pi)$ and let
  $e_{\pi}$ denote the $\pi$th standard basis vector in $\bbR^{\Pi}$. 
  For each $\pistar\in\Pi$, we write regret as
\begin{align*}
  \RegDM(\pistar) =
  \sum_{t=1}^{T}\En_{\pi\sim{}p\ind{t}}\brk*{\fmt(\pistar)-\fmt(\pi)}
  =  \sum_{t=1}^{T}\tri[\big]{e_{\pistar}-p\ind{t},\fmt}.
\end{align*}
Adding and subtracting
$\sum_{t=1}^{T}\tri[\big]{e_{\pistar}-q\ind{t},\fhat\ind{t}}$, we
rewrite this as
\begin{align}
  \sum_{t=1}^{T}\tri[\big]{e_{\pistar}-p\ind{t},\fmt}
  = \sum_{t=1}^{T}\tri[\big]{e_{\pistar}-p\ind{t},\fmt} +
  \sum_{t=1}^{T}\tri[\big]{e_{\pistar}-q\ind{t},\fhat\ind{t}}
   - \sum_{t=1}^{T}\tri[\big]{e_{\pistar}-q\ind{t},\fhat\ind{t}}.\label{eq:reg0}
\end{align}
The exponential weights update ensures
(\pref{lem:exponential_weights}) that with probability $1$,
\begin{align}
  \sum_{t=1}^{T}\tri[\big]{e_{\pistar} - q\ind{t},\fhat\ind{t}}
  &\leq{}
    \sum_{t=1}^{T}\tri[\big]{q\ind{t+1}-q\ind{t},\fhat\ind{t}}
        - \frac{1}{\eta}\sum_{t=1}^{T}\Dkl{q\ind{t+1}}{q\ind{t}} +
    \frac{\Dkl{e_{\pistar}}{q\ind{1}}}{\eta} \notag\\
    &\leq{}
    \sum_{t=1}^{T}\tri[\big]{q\ind{t+1}-q\ind{t},\fhat\ind{t}}
        - \frac{1}{\eta}\sum_{t=1}^{T}\Dkl{q\ind{t+1}}{q\ind{t}} +
      \frac{\log\abs{\Pi}}{\eta}.\notag
\end{align}
In addition, using \pref{lem:dv}, we have that for all $t$,
\[
\tri[\big]{q\ind{t+1},\fhat\ind{t}} -
\frac{1}{\eta}\Dkl{q\ind{t+1}}{q\ind{t}}
\leq{} \frac{1}{\eta}\log\prn*{\sum_{\pi}q\ind{t}(\pi)
    \exp\prn*{
      \eta\cdot\fhat\ind{t}(\pi)}
    }.
\]
Hence, combining this with \pref{eq:reg0}, we have
\begin{align*}
  \RegDM(\pistar)
  \leq{} \sum_{t=1}^{T}\tri[\big]{e_{\pistar}-p\ind{t},\fmt}
  - \tri[\big]{e_{\pistar},\fhat\ind{t}}
  + \frac{1}{\eta}\sum_{t=1}^{T}\log\prn*{\sum_{\pi}q\ind{t}(\pi)
    \exp\prn*{
  \eta\cdot\fhat\ind{t}(\pi)}
  } + \frac{\log\abs{\Pi}}{\eta}.
\end{align*}
Let
$\filt_{t}\ldef\sigma(\pi\ind{1},z\ind{1},\ldots,\pi\ind{t},z\ind{t})$
be a filtration, and let $\En_{t}\brk{\cdot}\ldef\En\brk*{\cdot\mid\filt_t}$.
For each $\pi\in\Pi$, define a sequence of random variables
$\crl*{X_t(\pi)}_{t=1}^{T}$ via \[X_t(\pi) = \frac{1}{\eta}\log\prn*{\sum_{\pi'}q\ind{t}(\pi')
    \exp\prn*{
  \eta\cdot\fhat\ind{t}(\pi')}
  }- \tri[\big]{e_{\pistar},\fhat\ind{t}}.\] Using
  \pref{lem:martingale_chernoff} and a union bound, we have that for any $\eta>0$,
  with probability at least $1-\delta$, for all $\pi\in\Pi$
  \begin{align*}
    \sum_{t=1}^{T}X_t(\pi)
    \leq{} \frac{1}{\eta}\sum_{t=1}^{T}
    \log\prn*{\En_{t-1}\brk*{\exp\prn*{
  \eta{}X_t(\pi)
    }}}
  + \frac{\log(\abs{\Pi}/\delta)}{\eta}.
  \end{align*}
Since this bounded holds uniformly for all $\pi$, we are guaranteed that with probability at least
$1-\delta$, for all $\pistar\in\Pi$,
\begin{align*}
  \RegDM(\pistar)
  \leq{} \sum_{t=1}^{T}\tri[\big]{e_{\pistar}-p\ind{t},\fmt}
+\frac{1}{\eta}\sum_{t=1}^{T}
    \log\prn*{\En_{t-1}\brk*{\exp\prn*{
  \eta{}X_t(\pistar)
    }}}  
  + 2 \frac{\log(\abs{\Pi}/\delta)}{\eta}.
\end{align*}
  We compute that for any $\pistar\in\Pi$,
  \begin{align*}
    &\log\prn*{\En_{t-1}\brk*{\exp\prn*{
  \eta{}X_t(\pistar)
    }}}\\
    &= 
      \log\prn*{\En_{\pi\sim{}p\ind{t}}\En_{z\sim{}M\ind{t}(\pi)}
      \En_{\pi'\sim{}q\ind{t}}\brk*{
    \exp\prn*{
      \frac{\eta}{p\ind{t}(\pi)}\cdot\prn[\big]{g\ind{t}(\pi'\midsem{}\pi,z)-g\ind{t}(\pistar\midsem{}\pi,z)}}
            }
      } \\
    &\leq{}
      \En_{\pi\sim{}p\ind{t}}\En_{z\sim{}M\ind{t}(\pi)}
      \En_{\pi'\sim{}q\ind{t}}\brk*{
    \exp\prn*{
      \frac{\eta}{p\ind{t}(\pi)}\cdot\prn[\big]{g\ind{t}(\pi'\midsem{}\pi,z)-g\ind{t}(\pistar\midsem{}\pi,z)}}
            }
      -1,
  \end{align*}
  where we have used that $\log(x)\leq{}x-1$ for $x>0$. Hence, with
  probability at least $1-\delta$, for all $\pistar\in\Pi$,
  \begin{align*}
    \RegDM(\pistar)  
    &\leq{}
    \sum_{t=1}^{T}\tri[\big]{e_{\pistar}-p\ind{t},\fmt} + 2\frac{\log(\abs{\Pi}/\delta)}{\eta}\\
    &~~~~~+ \frac{1}{\eta}\prn*{
      \En_{\pi\sim{}p\ind{t}}\En_{z\sim{}M\ind{t}(\pi)}
      \En_{\pi'\sim{}q\ind{t}}\brk*{
    \exp\prn*{
      \frac{\eta}{p\ind{t}(\pi)}\cdot\prn[\big]{g\ind{t}(\pi'\midsem{}\pi,z)-g\ind{t}(\pistar\midsem{}\pi,z)}}
            }
      -1} \\
    &=
      \sum_{t=1}^{T}\exoval_{q\ind{t},\eta}(p\ind{t},g\ind{t}\midsem{}\pistar,M\ind{t})
      + 2\frac{\log(\abs{\Pi}/\delta)}{\eta} \\
    &\leq{}  
                  \exostar_{\eta}(\cM)\cdot{}T
      + 2\frac{\log(\abs{\Pi}/\delta)}{\eta},
  \end{align*}
  where the last line uses that $(p\ind{t},g\ind{t})$ are chosen to
  minimize the Exploration-By-Optimization objective. Finally, using
  \pref{thm:equivalence}, $\exostar_{\eta}(\cM)\leq{}\comp[(8\eta)^{-1}](\conv(\cM))$.

\end{proof}

\subsection{Proof of \preftitle{thm:lower_main}}

In this section we prove \pref{thm:lower_main}. Most of the work
consists of proving an improved lower bound for the \emph{stochastic}
setting in which $M\ind{t}=\Mstar$ is fixed across $t$
(\pref{thm:lower_general}). We then appeal to this stochastic lower
bound with the class $\conv(\cM)$. Since $\conv(\cM)$ is equivalent to
the set of mixtures of models in $\cM$, this establishes existence of
distribution $\mu\in\Delta(\cM)$ and mixture model $M_{\mu}=\En_{M\sim\mu}\brk*{M}$ for which regret in the
stochastic setting must
scale with $\comploc{\vepslowg}(\conv(\cM))$. The proof concludes by
arguing that this yields a lower bound for the adversarial setting
when we sample $M\ind{t}\sim\mu$.

Throughout this section, we define the \emph{one-sided variance} for a random
variable $Z$ as \[\varplus\brk{Z}
\ldef \En\brk*{(Z-\En\brk*{Z})_{+}^2}.\]

\lowermain*

We also have the following slight variant of \pref{thm:lower_main}.
\begin{thmmod}{thm:lower_main}{a}
  \label{thm:lower_var}
            Let $\Ct \ldef c\cdot\log(T\wedge{}V(\cM))$ for a
          sufficiently large numerical constant $c>0$. Set
          $\vepslowg\ldef\frac{\gamma}{4\Ct{}T}$. For any algorithm,
          there exists an oblivious adversary for which
          $\En\brk*{\RegDM}\geq{}0$ and
          \begin{align}
            \label{eq:lower_var}
            \En\brk*{\RegDM} + \sqrt{\En\brk*{\RegDM}\cdot{}T} \geq\bigom(1)\cdot\sup_{\gamma>\sqrt{2\Ct{}T}}\comploc{\vepslowg}(\conv(\cM))\cdot{}T,
          \end{align}%
\end{thmmod}

\begin{proof}[\pfref{thm:lower_main}]%
  \newcommand{\RegDMTil}{\wt{\Reg}_{\mathsf{DM}}}%
  \newcommand{\RegDMCheck}{\wh{\Reg}_{\mathsf{DM}}}%
  \newcommand{\pimu}{\pi_{\mu}}%
We invoke \pref{thm:lower_general} with the model class $\conv(\cM)$,
which implies
  that there exists a distribution $\mu\in\Delta(\cM)$ for which
  \begin{align*}
    \En\brk[\big]{\RegDMTil} + \sqrt{\Varplus\brk[\big]{\RegDMTil}} \geq{}L
    \ldef{} 8^{-1}\cdot\sup_{\gamma>\sqrt{2\Ct{}T}}\comploc{\vepslowg}(\conv(\cM))\cdot{}T,
  \end{align*}
  where
  \[
    \RegDMTil\ldef{}\sum_{t=1}^{T}\En_{\pi\ind{t}\sim{}p\ind{t}}\En_{M\sim\mu}\brk*{\fm(\pimu)-\fm(\pi\ind{t})},
  \]
  and $\pimu\ldef\argmax_{\pi\in\Pi}\En_{M\sim\mu}\brk*{\fm(\pi)}$,
  with
  the data generating process is (for each $t=1,\ldots,T$):
  \begin{itemize}
  \item The learner samples $\pi\ind{t}\sim{}p\ind{t}$.
  \item Nature samples $z\ind{t}\sim \En_{M\sim\mu}\brk*{M(\pi\ind{t})}$.
  \end{itemize}
Observe that this is equivalent in law to the following data-generating
process, which constitutes an admissible adversary (with
$M\ind{t}\in\cM$):
  \begin{itemize}
  \item The learner samples $\pi\ind{t}\sim{}p\ind{t}$.
  \item Nature samples $M\ind{t}\sim\mu$ and 
    $z\ind{t}\sim M\ind{t}(\pi\ind{t})$.
  \end{itemize}
  Likewise, we can equivalently write
    \[
    \RegDMTil=\sum_{t=1}^{T}\En_{M\ind{t}\sim\mu}\En_{\pi\ind{t}\sim{}p\ind{t}}\brk*{\fmt(\pimu)-\fmt(\pi\ind{t})}.
  \]
Hence, all that remains is to relate the quantity $\RegDMTil$ to
  the realized regret $\RegDM$ for the sequence $M\ind{1},\ldots{}M\ind{T}$, which entails removing the conditional
  expectation over $M\ind{t}\sim\mu$. To this end, we first observe that
  \begin{align*}
    \En\brk[\big]{\RegDMTil}
    &=
      \En\brk*{\sum_{t=1}^{T}\En_{\pi\ind{t}\sim{}p\ind{t}}\brk*{\fmt(\pimu)-\fmt(\pi\ind{t})}}
    \\
    &\leq
      \En\brk*{\max_{\pistar\in\Pi}\sum_{t=1}^{T}\En_{\pi\ind{t}\sim{}p\ind{t}}\brk*{\fmt(\pistar)-\fmt(\pi\ind{t})}}
      = \En\brk[\big]{\RegDM}.
  \end{align*}
Next, note that since $\RegDMTil$ is non-negative,
$\Varplus\brk[\big]{\RegDMTil}\leq{}\En\brk[\big]{(\RegDMTil)_{+}^2}$. Define
\begin{align*}
    \RegDMCheck\ldef{}\sum_{t=1}^{T}\En_{\pi\ind{t}\sim{}p\ind{t}}\brk*{\fmt(\pimu)-\fmt(\pi\ind{t})}.
\end{align*}
Then we have
\begin{align*}
  \En\brk[\big]{(\RegDMTil)_{+}^2}
  &\leq{} 2\En\brk[\big]{(\RegDMCheck)_{+}^2}
    + 2\En\brk[\big]{(\RegDMTil-\RegDMCheck)^2}\\
    &\leq{} 2\En\brk[\big]{(\RegDM)_{+}^2}
      + 2\En\brk[\big]{(\RegDMTil-\RegDMCheck)^2}\\
      &\leq{} 2\En\brk[\big]{(\RegDM)_{+}^2}
        + 2T,
\end{align*}
where the first inequality uses that $\RegDMCheck\leq\RegDM$ almost
surely, and the second inequality uses (i) $\fm\in\brk*{0,1}$, and
(ii) for any sequence
of random variables $(Z_t)_{t=1}^{T}$ with
$\En\brk[\big]{Z_t\mid{}Z_1,\ldots,Z_{t-1}}=0$,
$\En\brk*{\prn{\sum_{t=1}^{T}Z_t}^2}=\sum_{t=1}^{T}\En\brk*{Z_t^2}$. Putting
everything together, we conclude that
\[
    \En\brk[\big]{\RegDM} +
    \sqrt{2\En\brk[\big]{\prn[\big]{\RegDM}_{+}^2}} \geq{}L - \sqrt{2T}.
  \]
This proves \pref{thm:lower_main}. To prove \pref{thm:lower_var}, we use that since
$\RegDMTil\in\brk*{0,T}$,
\[
  \Varplus\brk[\big]{\RegDMTil}
  \leq{} T\cdot{}\En\brk[\big]{\RegDMTil}\leq{}T\cdot\En\brk*{\RegDM}.
\]
  
\end{proof}

The following result concerns the \emph{stochastic setting} in
\citet{foster2021statistical}. Here, there is a (unknown) underlying model
$\Mstar\in\cM$. For $t=1,\ldots,T$, data is generated through the
process:
\begin{itemize}
\item Learner samples $\pi\ind{t}\sim{}p\ind{t}$.
\item Nature samples $z\ind{t}\sim\Mstar(\pi\ind{t})$.
\end{itemize}
In addition, regret simplifies to
\begin{align}
  \label{eq:regret_stochastic}
  \RegDM =
  \sum_{t=1}^{T}\En_{\pi\ind{t}\sim{}p\ind{t}}\brk[\big]{\fmstar(\pimstar)
  - \fmstar(\pi\ind{t})}
\end{align}
For a fixed algorithm, let $\bbP\sups{M}$ denote the law of
$\cH\ind{T}$ when $\Mstar=M$, and let $\Enm\brk*{\cdot}$ and $\Varplus\sup{M}\brk*{\cdot}$ denote the
corresponding expectation non-negative variance.  Recall that 
$\abscont=\sup_{M,M'\in\cM}\sup_{\act\in\Act}\sup_{A\in{}\Rsig\otimes\Osig}\crl[\big]{\tfrac{M(A\mid\act)}{M'(A\mid{}\act)}}\vee{}e$. Our main lower bound for the stochastic
setting is as follows.
\begin{restatable}{theorem}{lowergeneral}
  \label{thm:lower_general}
      Let $\Ct \ldef 2^{9}\log(T\wedge{}V(\cM))$, and set
      $\vepslowg=\frac{\gamma}{4\Ct{}T}$.
  For any algorithm, there exists a model in $\cM$ for which
  \begin{align*}
    \Enm\brk*{\RegDM} + \sqrt{\Varplus\sups{M}\brk{\RegDM}} \geq{}8^{-1}\cdot\sup_{\gamma\geq{}4\sqrt{\Ct{}T}}\sup_{\Mbar\in\cM}\comp(\cMloc[\vepslowg](\Mbar),\Mbar)\cdot{}T.
  \end{align*}
\end{restatable}

The general structure of the lower bound follows that of Theorem 3.1
in \citet{foster2021statistical}, with the main difference being that
we use a more refined change-of-measure argument to move from a
``reference'' model $\Mbar\in\cM$ to a worst-case
alternative. Specifically, we replace Lemma A.11 in
\citet{foster2021statistical}, which requires an almost sure bound on
the random variables under consideration (in our case, regret), with
\pref{lem:hellinger_com}, which requires only boundedness of the
second moment. Combining this with a self-bounding argument
that takes advantage of the localized model class yields the result.

\begin{proof}[\pfref{thm:lower_general}]%
Throughout this proof we will use that $\RegDM$ is
non-negative in the stochastic setting, which can be seen by
inspecting \pref{eq:regret_stochastic} (in the general adversarial setting, it
is possible for $\RegDM$ to be negative).

Let us introduce  additional notation. For $M\in\cM$, define
$\gm(\pi)=\fm(\pim)-\fm(\pi)$, and for $p\in\Delta(\Pi)$, let $\gm(p) \ldef
  \En_{\pi\sim{}p}\brk*{\gm(\pi)}$. Let $\phat \ldef
  \frac{1}{T}\sum_{t=1}^{T}p\ind{t}$, and
  $\pm\ldef\Enm\brk*{\frac{1}{T}\sum_{t=1}^{T}p\ind{t}}$.

  To begin, fix
    $\Mbar\in\cM$, $\gamma>0$, and $\veps>0$, and set
    \[
      M = \argmax_{M\in\cMloc(\Mbar)}
      \En_{\pi\sim{}\pmbar}\brk*{\fm(\pim) - \fm(\pi)  - \gamma\cdot\Dhels{M(\pi)}{\Mbar(\pi)}}.
  \]
Abbreviate $\comp\equiv\comp(\cMloc[\veps](\Mbar),\Mbar)$. The definition of the \CompShort implies that
  \begin{align}
    \comp \leq \En_{\pmbar}\brk*{\gm(\pi)} - \gamma\cdot\En_{\pmbar}\brk*{\Dhels{M(\pi)}{\Mbar(\pi)}}
          = \Enmbar\brk*{\gm(\phat)} -
    \gamma\cdot\En_{\pmbar}\brk*{\Dhels{M(\pi)}{\Mbar(\pi)}}.
    \label{eq:h_step0}
  \end{align}

\paragraph{Change of measure} To proceed, we write
  \begin{align}
    \Enmbar\brk*{\gm(\phat)}
    &= \Enmbar\brk*{\gm(\phat)-\gmbar(\phat)-\Enm\brk*{\gm(\phat)}} +
    \Enmbar\brk*{\gmbar(\phat)}
      + \Enm\brk*{\gm(\phat)}\notag\\
    &\leq{} \Enmbar\brk*{(\gm(\phat)-\gmbar(\phat)-\Enm\brk*{\gm(\phat)})_{+}} +
    \Enmbar\brk*{\gmbar(\phat)}
    + \Enm\brk*{\gm(\phat)}.\label{eq:h_step1}
  \end{align}

  We recall the following technical lemma.
  \hellingercom*
  Defining $h(\phat) =
(\gm(\phat)-\gmbar(\phat)-\Enm\brk*{\gm(\phat)})_{+}$,
\pref{lem:hellinger_com} implies that
\begin{align}
  &\Enmbar\brk*{(\gm(\phat)-\gmbar(\phat)-\Enm\brk*{\gm(\phat)})_{+}}\notag\\
  &\leq{}
  \Enm\brk*{(\gm(\phat)-\gmbar(\phat)-\Enm\brk*{\gm(\phat)})_{+}}
  + \sqrt{\prn*{\Enm\brk*{h(\phat)^2} +
    \Enmbar\brk*{h(\phat)^2}}\cdot\Dhels{\bbPm}{\bbPmbar}}\notag\\
  &\leq{}
  \Enm\brk*{\gm(\phat)}
    + \sqrt{\prn*{\Enm\brk*{h(\phat)^2} +
    \Enmbar\brk*{h(\phat)^2}}\cdot\Dhels{\bbPm}{\bbPmbar}},
    \label{eq:com_step0}
\end{align}
where we have used that $\gm,\gmbar\geq{}0$. We proceed to bound the
second moment terms. For the first such term, we bound by the
non-negative variance as follows:
\begin{align}
  \Enm\brk*{h(\phat)^2}
  & = 
    \Enm\brk*{(\gm(\phat)-\gmbar(\phat)-\Enm\brk*{\gm(\phat)})_{+}^2}\notag
  \\
  & \leq{} \Enm\brk*{(\gm(\phat)-\Enm\brk*{\gm(\phat)})_{+}^2} \notag\\
  & = \Varm_{+}\brk*{\gm(\phat)}. \label{eq:var_m}
\end{align}
where the first inequality uses that $\gmbar\geq{}0$. For the second
variance term, we begin with the bound
\begin{align*}
  \Enmbar\brk*{h(\phat)^2}
   =
     \Enmbar\brk*{(\gm(\phat)-\gmbar(\phat)-\Enm\brk*{\gm(\phat)})_{+}^2}
    \leq{}
     \Enmbar\brk*{(\gm(\phat)-\gmbar(\phat))_{+}^2}.
\end{align*}
We further bound this quantity by
\begin{align*}
  \hspace{0.5in} & \hspace{-0.5in} \Enmbar\brk*{(\gm(\phat)-\gmbar(\phat))_{+}^2} \\
  &=   \Enmbar\brk*{(\gm(\phat)-\gmbar(\phat))_{+}(\fm(\pim) - \fmbar(\pimbar) +
    \fmbar(\phat)-\fm(\phat))_{+}}\\
  &\leq \Enmbar\brk*{(\gm(\phat)-\gmbar(\phat))_{+}(\fm(\pim) - \fmbar(\pimbar))_{+}} +
    \Enmbar\brk*{(\gm(\phat)-\gmbar(\phat))_{+}(\fmbar(\phat)-\fm(\phat))_{+}}.
\end{align*}
For the first term above, we have
\begin{align*}
  \Enmbar\brk*{(\gm(\phat)-\gmbar(\phat))_{+}(\fm(\pim) -
    \fmbar(\pimbar)_{+}} 
  \leq{} \veps \cdot{}\Enmbar\brk*{(\gm(\phat)-\gmbar(\phat))_{+}}
  \leq{} \veps\cdot\Enmbar\brk*{\gm(\phat)},
\end{align*}
where we have used the localization property and the fact that
$\gm,\gmbar\geq{}0$. For the second term, using the AM-GM inequality yields
\begin{align*}
\Enmbar\brk*{(\gm(\phat)-\gmbar(\phat))_{+}(\fmbar(\phat)-\fm(\phat))_{+}}  &\leq{}   \frac{1}{2}\Enmbar\brk*{(\gm(\phat)-\gmbar(\phat))^2_{+}}
    + \frac{1}{2}\Enmbar\brk*{(\fmbar(\phat)-\fm(\phat))_{+}^2} \\
    &\leq{}   \frac{1}{2}\Enmbar\brk*{(\gm(\phat)-\gmbar(\phat))^2_{+}}
      + \frac{1}{2}\En_{\pi\sim\pmbar}\brk*{(\fm(\pi)-\fmbar(\pi))^2} \\
      &\leq{}   \frac{1}{2}\Enmbar\brk*{(\gm(\phat)-\gmbar(\phat))^2_{+}}
        +
        \frac{1}{2}\En_{\pi\sim\pmbar}\brk*{\Dhels{M(\pi)}{\Mbar(\pi)}}, 
\end{align*}
where the second line uses Jensen's inequality, and the last line uses that rewards are observed and
bounded in $\brk*{0,1}$. After combining these results and
rearranging, we have
\begin{align}
    \Enmbar\brk*{h(\phat)^2}
   \leq\Enmbar\brk*{(\gm(\phat)-\gmbar(\phat))_{+}^2} 
   \leq 2 \veps\cdot\Enmbar\brk*{\gm(\phat)} + \En_{\pi\sim\pmbar}\brk*{\Dhels{M(\pi)}{\Mbar(\pi)}}.\label{eq:var_mbar}
\end{align}

To proceed, we recall Lemma A.13 from \citet{foster2021statistical},
which states that
\begin{equation}
  \label{eq:hellinger_bound}
    \Dhels{\bbPm}{\bbPmbar}
    \leq \Ct\cdot{}T\cdot\En_{\pi\sim\pmbar}\brk*{\Dhels{M(\pi)}{\Mbar(\pi)}},
  \end{equation}
  where $\Ct \leq 2^{8}\cdot\log(T\wedge{}\abscont)$. Combining this
  with the variance bounds above and \pref{eq:com_step0}, we have
\begin{align*}
 \hspace{0.5in}&\hspace{-0.5in} \Enmbar\brk*{(\gm(\phat)-\gmbar(\phat)-\Enm\brk*{\gm(\phat)})_{+}}\\
  &\leq \Enm\brk*{\gm(\phat)}
  + \sqrt{\prn*{\Varm_{+}\brk*{\gm(\phat)} +
    2\veps\cdot\Enmbar\brk*{\gm(\phat)} + \En_{\pi\sim\pmbar}\brk*{\Dhels{M(\pi)}{\Mbar(\pi)}}}
    \cdot\Dhels{\bbPm}{\bbPmbar}} \\
  &\leq{} \Enm\brk*{\gm(\phat)} + \sqrt{2\Varm_{+}\brk*{\gm(\phat)}}
    + \sqrt{\prn*{
    2\veps\cdot\Enmbar\brk*{\gm(\phat)} + \En_{\pi\sim\pmbar}\brk*{\Dhels{M(\pi)}{\Mbar(\pi)}}}
    \cdot\Dhels{\bbPm}{\bbPmbar}} \\
  &\leq{} \Enm\brk*{\gm(\phat)} + \sqrt{2\Varm_{+}\brk*{\gm(\phat)}}
        + \sqrt{C(T)T}\cdot
        \En_{\pi\sim\pmbar}\brk*{\Dhels{M(\pi)}{\Mbar(\pi)}}\\
  &~~~~+ \sqrt{2\veps\Enmbar\brk*{\gm(\phat)}\cdot{}C(T)T \En_{\pi\sim\pmbar}\brk*{\Dhels{M(\pi)}{\Mbar(\pi)}}},
\end{align*}
where the second inequality uses that $\Dhels{\cdot}{\cdot}\leq{}2$
and the last inequality uses \pref{eq:hellinger_bound}.

We observe that under the restriction
$\veps\leq{}\frac{\gamma}{4TC(T)}$, 
\begin{align*}
  \sqrt{2\veps\cdot\Enmbar\brk*{\gm(\phat)}\cdot{}C(T)T
  \En_{\pi\sim\pmbar}\brk*{\Dhels{M(\pi)}{\Mbar(\pi)}}}
 &\leq{} \sqrt{\Enmbar\brk*{\gm(\phat)}\cdot{}\frac{\gamma}{2}\cdot
   \En_{\pi\sim\pmbar}\brk*{\Dhels{M(\pi)}{\Mbar(\pi)}}}\\
&\leq{} \frac{1}{2}\Enmbar\brk*{\gm(\phat)} + \frac{\gamma}{4}\cdot \En_{\pi\sim\pmbar}\brk*{\Dhels{M(\pi)}{\Mbar(\pi)}},
\end{align*} where the last line uses AM-GM inequality. The above yields that 
\begin{align*}
  \hspace{0.5in}&\hspace{-0.5in} \Enmbar\brk*{(\gm(\phat)-\gmbar(\phat)-\Enm\brk*{\gm(\phat)})_{+}} \\
  &\leq{} \Enm\brk*{\gm(\phat)} + \sqrt{2\Varm_{+}\brk*{\gm(\phat)}}
        + (\sqrt{C(T)T}+\gamma/4)\cdot
        \En_{\pi\sim\pmbar}\brk*{\Dhels{M(\pi)}{\Mbar(\pi)}}
+ \frac{1}{2}\Enmbar\brk*{\gm(\phat)}.
\end{align*}
Using \pref{eq:h_step1}, this implies that
\begin{align*}
  \Enmbar\brk*{\gm(\phat)}
  &\leq{}
  2\Enm\brk*{\gm(\phat)} + \Enmbar\brk*{\gmbar(\phat)} + \sqrt{2\Varm_{+}\brk*{\gm(\phat)}}\\
  &\hspace{1in} + (\sqrt{C(T)T}+\gamma/4)\cdot
        \En_{\pi\sim\pmbar}\brk*{\Dhels{M(\pi)}{\Mbar(\pi)}}
+ \frac{1}{2}\Enmbar\brk*{\gm(\phat)},
\end{align*}
and after rearranging,
\begin{align}
  \Enmbar\brk*{\gm(\phat)}
  \leq{}
  4\Enm\brk*{\gm(\phat)} + 2\Enmbar\brk*{\gmbar(\phat)} + \sqrt{8\Varm_{+}\brk*{\gm(\phat)}}
        + 2(\sqrt{C(T)T}+\gamma/4)\cdot
        \En_{\pi\sim\pmbar}\brk*{\Dhels{M(\pi)}{\Mbar(\pi)}}.\label{eq:com_final}
\end{align}
\paragraph{Completing the proof}
Combining \pref{eq:com_final} with \pref{eq:h_step0} yields the bound
\begin{align*}
  \comp \leq 4\Enm\brk*{\gm(\phat)} + 2\Enmbar\brk*{\gmbar(\phat)} + \sqrt{8\Varm_{+}\brk*{\gm(\phat)}}
        + \prn*{2(\sqrt{C(T)T}+\gamma/4)-\gamma}\cdot
        \En_{\pi\sim\pmbar}\brk*{\Dhels{M(\pi)}{\Mbar(\pi)}}.
\end{align*}
In particular, whenever $\gamma\geq4\sqrt{C(T)T}$,
this implies that there exists an instance $M'\in\crl{M,\Mbar}$ for which
\[
  \En\sups{M'}\brk*{g\sups{M'}(\phat)}
  + \sqrt{\Var\sups{M'}_{+}\brk*{g\sups{M'}(\phat)}}
   \geq{} 8^{-1}\cdot\comp.
 \]
 Finally, we observe that $g\sups{M'}(\phat)$ is identical in law to
 $\RegDM$ under $\bbP\sups{M'}$.

\end{proof}

\subsection{Proof of \preftitle{thm:learnability}}

{\renewcommand\footnote[1]{}\learnability*}

\begin{proof}[\pfref{thm:learnability}]%
  \newcommand{\RegDMT}{\RegDM(T)}%
  This proof closely follows that of Theorem 3.5 in
  \citet{foster2021statistical}.
  
  \paragraph{Upper bound}
  Assume that
  $\lim_{\gamma\to\infty}\comp(\conv(\cM))\cdot{}\gamma^{\rho}=0$ for some
  $\rho>0$, and that $\log\abs{\Pi_T}=\bigoht(T^{q})$ for
  some $q<1$. Using \pref{thm:upper_main} with $\delta=1/T$, there
  exists an algorithm such that for each $T$, for all adversaries, 
  \begin{align*}
    \En\brk*{\RegDMT}
    \leq{} \bigoht\prn*{
    \comp(\conv(\cM))\cdot{}T + \gamma\cdot\log\abs{\Pi_T}
    }
     \leq{} \bigoht\prn*{
    \comp(\conv(\cM))\cdot{}T + \gamma\cdot{}T^{q}
    },
  \end{align*}
  with $\bigoht(\cdot)$ hiding factors logarithmic in $T$. 
  For each $T$, we set $\gamma=\gamma_T\ldef{}T^{\frac{1-q}{1+\rho}}$;
  recall that $1-q>0$. The assumption that
  $\lim_{\gamma\to\infty}\comp(\conv(\cM))\cdot{}\gamma^{\rho}=0$, implies
  that for all $\veps>0$, there exists $\gamma'>0$ such that
  $\comp(\conv(\cM))\leq{}\veps/\gamma^{\rho}$ for all
  $\gamma\geq{}\gamma'$. For $T$ sufficiently large, this implies that
  for all adversaries
  \begin{align*}
    \En\brk*{\RegDM}
     \leq{} \bigoht\prn*{
    \frac{T}{\gamma_T^{\rho}}+ \gamma_T\cdot{}T^{q}
    } = \bigoht(T^{\frac{1+\rho{}q}{1+\rho}}).
  \end{align*}
  Defining $p'\ldef\frac{1}{2}(p+1)<1$, this establishes that
  \[
    \lim_{T\to\infty}\frac{\MinimaxReg}{T^{p'}}=0.
  \]
  \paragraph{Lower bound}
  Assume that
$\lim_{\gamma\to\infty}\comp(\conv(\cM))\cdot\gamma^{\rho}=\infty$ for all
$\rho>0$ (this is equivalent to assuming that $\lim_{\gamma\to\infty}\comp(\conv(\cM))\cdot\gamma^{\rho}>0$
  for all $\rho>0$, as in the theorem statement). Let $\rho\in(0,1/2)$
  be fixed. Using \pref{thm:lower_var}, we are guaranteed that for any
  algorithm, there exists an adversary for which
  $\En\brk*{\RegDM}\geq{}0$ and
\begin{align*}
  \En\brk*{\RegDM} + \sqrt{\En\brk*{\RegDM}\cdot{}T}
  &=\bigomt\prn*{
  \compbasic_{\gamma,\veps(\gamma,T)}(\conv(\cM))\cdot{}T
    },
\end{align*}
for all $\gamma=\omega(\sqrt{T\log(T)})$, where $\veps(\gamma,T)\ldef{}c\cdot\frac{\gamma}{T\log(T)}$ for a
sufficiently small numerical constant $c\leq{}1$. Since there exists
$M_0\in\cM$ such that the function $\fmnot$ is constant, Lemma B.1 of
\citet{foster2021statistical} further implies that
\begin{align*}
    \En\brk*{\RegDM} + \sqrt{\En\brk*{\RegDM}\cdot{}T}
  &=\bigomt\prn*{
\veps(\gamma,T) \cdot{}\comp(\conv(\cM))\cdot{}T
    }.
\end{align*}
For each $T$, set $\gamma=\gamma_T\ldef{}T$. By the assumption that
$\lim_{\gamma\to\infty}\comp(\conv(\cM))\cdot\gamma^{\rho}=\infty$, it
holds that for $T$
sufficiently large,
$\comp[\gamma_T](\conv(\cM))\geq\gamma_T^{-\rho}$, which implies that
and
\[
      \En\brk*{\RegDM} + \sqrt{\En\brk*{\RegDM}\cdot{}T}
  =\bigomt\prn*{
  \frac{T}{\gamma_T^{\rho}}
},
\]
where we have used that
$\veps(\gamma_T,T)\propto\frac{1}{\log(T)}$. Rearranging, this implies
that
\[
  \En\brk*{\RegDM}
  =\bigomt\prn*{
T^{1-2\rho}
}.
\]
Hence, for any $p\in(0,1)$, by setting
$\rho=\frac{1-p}{2}\in(0,1/2)$, we have
\[
  \En\brk*{\RegDM} = \bigomt(T^{p}).
\]
Applying this argument with
$p'=\frac{1}{2}(p+1)\in(1/2,1)$ yields
  \[
    \lim_{T\to\infty}\frac{\MinimaxReg}{T^{p}}=\infty.
  \]
\end{proof}

\subsection{Sub-Chebychev Algorithms}
\label{sec:subc}

\begin{proposition}
  \label{prop:subc_variance}
  Any random variable with $\En\brk*{X_+^2}\leq{}R$
  has
  \[
    \bbP(X_{+}>t) \leq \frac{R^2}{t^2},\quad\forall{}t>0.
  \]
  Conversely, if $X\in(-\infty,B)$ and has
  $\bbP(X_{+}>t) \leq \frac{R^2}{t^2}\;\forall{}t>0$, then
  \[
    \En\brk*{X_{+}^2} \leq{} R^{2}(\log(B/R)+1).
  \]
  
\end{proposition}
\begin{proof}[\pfref{prop:subc_variance}]
  For the first direction, note that if $\En\brk*{X_{+}^2}\leq{}R$,
  Chebychev's inequality implies that for all $t>0$, 
\begin{equation}
  \label{eq:chebychev}
  \bbP\prn*{X_{+}^2>t} \leq \frac{R^2}{t^2}.
\end{equation}
For the other direction, since $X_{+}\in\brk*{0,B}$ almost
surely, we can bound by
\[
  \En\brk*{X_{+}^2}=\int_{0}^{B}\bbP(X_{+}>t)tdt
  \leq{} R^2 + \int_{R}^{B}\bbP(X_{+}>t)tdt
  \leq{} R^2+R^2 \int_{R}^{B}\frac{1}{t}dt
  \leq{} R^2+R^{2}\log(B/R). 
\]

\end{proof}

\begin{proposition}
  \label{prop:subc_high_prob}
  Suppose that for any $\delta>0$, an algorithm (with $\delta$ as
  a parameter) ensures that with probability at least $1-\delta$,
  \[
    \RegDM \leq{} R\log^{\rho}(\delta^{-1})
  \]
  for some $R\geq{}1$ and $\rho>0$. Then the algorithm, when invoked with parameter
  $\delta=1/T^2$, is \subc with parameter $5^{1/2}R\log^{\rho}(T)$.
\end{proposition}
\begin{proof}[\pfref{prop:subc_high_prob}]
  Set $\delta=1/T^2$. Then, since $\abs{\RegDM}\leq{}T$, the law of total
  expectation implies that
  \[
    \En\brk*{(\RegDM)_{+}^2} \leq{} R^2\log^{2\rho}(T^2) + T^2/T^2 \leq{}
    5R^2\log^{2\rho}(T), 
  \]
  where we have used that $R\geq{}1$. Chebychev's inequality now
  implies that for all $t>0$
  \[
    \bbP((\RegDM)_{+}\geq{}t) \leq{} \frac{\En\brk*{(\RegDM)_{+}^2}}{t^2}\leq{}\frac{5R^2\log^{2\rho}(T)}{t}.
  \]
\end{proof}

\subcregret*

\begin{proof}[\pfref{cor:subc}]
This result immediately follows from \pref{prop:subc_variance}, \pref{prop:subc_high_prob}, and \pref{thm:lower_main}.
\end{proof}

\arxiv{
\section{Proofs for Examples (\preftitle{sec:examples})}
\label{app:examples}

\neurips{\subsubsection{Preliminaries}}
\arxiv{\paragraph{Preliminaries}}

Our lower bounds on the \CompText involve a constructing hard sub-family of models. Recall the following definition from \cite{foster2021statistical}. 
\begin{definition}[$(\alpha, \beta, \delta)$-family]
\label{def:hard_family_lb}
 A reference model \(\Mbar \in \cM\) and collection \(\crl{M_1, \dots, M_N}\) with \(N \geq 2\) are said to be an \((\alpha, \beta, \delta)\)-family if the following properties hold: 
\begin{itemize}
\item \textit{Regret property.} There exist functions \(u\sups{M} : \Pi \mapsto [0, 1]\), with \(\sum_{M \in \cM} u\sups{M}(\pi) \leq \frac{N}{2}\) for all \(\pi\) such that 
\begin{align*}
f\sups{M}(\pi\subs{M}) - f\sups{M}(\pi) \geq \alpha \cdot \prn*{1 - u\sups{M}(\pi)}
\end{align*}
for all \(M \in \cM\). 
\item \textit{Information property.} There exist functions \(v\sups{M}: \Pi \mapsto [0, 1]\), with \(\sum_{M \in \cM} v\sups{M}(\pi) \leq 1\) for all \(\pi\), such that 
\begin{align*}
\Dhels{M(\pi)}{\Mbar(\pi)}  \leq \beta \cdot v\sups{M}(\pi) + \delta. 
\end{align*}
\end{itemize} 
\end{definition}

Any $(\alpha, \beta, \delta)$-family leads to a difficult decision
making problem because a given decision can have low regret or large
information gain on (roughly) one model in the family. This is formalized through the following lemma. 
\begin{lemma}[Lemma 5.1, \cite{foster2021statistical}]
\label{lem:hard_family_lb}
 Let \(\cM = \crl{M_1, \dots, M_N}\) be an $(\alpha, \beta, \delta)$-family with respect to \(\Mbar\). Then, for all \(\gamma \geq 0\), 
\begin{align*}
  \comp(\cM,\Mbar) \geq \frac{\alpha}{2} - \gamma \prn*{\frac{\beta}{N} + \delta}. 
\end{align*}
\end{lemma}

The following technical lemma bounds Hellinger distance for Bernoulli distributions. 
\begin{lemma}[Lemma A.7, \citep{foster2021statistical}]
\label{lem:helg_bern} 
 For any \(\Delta \in (0, 1/2)\),  
\begin{align*}
\DhelsX{\Big}{\Ber\prn[\Big]{\frac{1}{2} + \Delta}}{\Ber\prn[\Big]{\frac{1}{2}}} \leq 3 \Delta^2. 
\end{align*}	
\end{lemma}

\neurips{\subsubsection{Proof for \preftitle{ex:tabular} (Tabular MDP)}}
\arxiv{\subsection{Proof for \preftitle{ex:tabular} (Tabular MDP)}}

In this section, we prove the lower bound in \pref{ex:tabular}. We
first derive an intermediate result which gives a lower bound on the
\CompText when the model class $\cM$ consists of \emph{mixtures of $K$
  MDPs}; this is equivalent to the subset of $\conv(\cM)$ where we
restrict to support size $K$, as well as the so-called latent MDP
setting \citep{kwon2021rl}.

\begin{lemma}
\label{lem:tabular}
Let \(K \geq 1\) be given. Let \(\cM\) be the class of \emph{mixtures
  of \(K\) MDPs} with horizon \(H\), \(S \geq 2\) states, \(A\geq{}2\) actions, and \(\sum_{h=1}^H r_h \in [0, 1]\). Then there exists \(\Mbar \in \cM\) such that for all \(\gamma \geq  A^
{\min\crl{S - 1, H, K}}/6\),
\[\comp(\cM_{\veps_\gamma}\prn*{\Mbar}, \Mbar) \geq \frac{A^
    {\min\crl{S - 1, H, K}}}{24 \gamma},
  \]
where \(\veps_\gamma \ldef \frac{A^
{\min\crl{S - 1, H, K}}}{24 \gamma}\).
\end{lemma}%
\newcommand{\ba}{\mb{a}}%
\renewcommand{\bK}{\wb{K}}%
 The proof of this result proceeds by constructing a hard sub-family of
 models and appealing to \pref{lem:hard_family_lb}. Our construction
 is based of the lower bound for latent MDPs in \citet{kwon2021rl}.

 \begin{proof}[\pfref{lem:tabular}]  Let \(\cS\) and \(\cA\) be
   arbitrary sets with $\abs{\cS}=S$ and $\abs{\cA}=A$. Let \(\Delta
   \in (0, 1/2)\) be a parameter to be chosen later, and define \(\bK
   \ldef \min\crl{S-1, K, H}\). Partition the state space \(\cS\) into
   sets \(\cS'\) and \(\cS \setminus \cS'\) such that \(\abs{\cS'} =
   \bK + 1\), and label the states in \(\cS'\) as $\crl{s\ind{1},
     \dots, s\ind{\bK + 1}}$. Additionally, define sets via \(\cS_h =
   \crl{s\ind{h}, s\ind{\bK + 1}}\) for \(h \leq \bK\) and \(\cS_{h} =
   \crl{s\ind{\bK+1}} \cup \prn{S \setminus S'}\) for \(\bK < h \leq
   H+1\). Recall that the decision space \(\PiNS\) is the set of all deterministic non-stationary policies \(\pi = \prn*{\pi_1, \dots, \pi_H}\) where \(\pi_h: \cS_h
   \mapsto \cA\).
   
We construct a class \(\cM'\subseteq\cM\) in which each model \(M \in \cM'\) is specified by 
 \begin{align*}
 M = \crl[\big]{\crl*{\cS_h}_{h=1}^{H+1}, \cA, \crl{\bbM\sups{M}_k}_{k=1}^{\bK}, \crl{a\sups{M}_k}_{k=1}^K},
\end{align*}
where for each \(k \in [\bK]\), \(a_k\sups{M} \in \cA\),  and where \(\bbM_k\sups{M}\) is a tabular MDP specified by 
\begin{align*}
\bbM\sups{M}_k = \crl[\big]{\crl*{\cS_h}_{h=1}^{H+1}, \cA,
    \crl*{P\sups{M}_{h, k}}_{h=1}^H, \crl*{R_{h, k}\sups{M}}_{h=1}^H, \delta_{s\ind{1}}}.
\end{align*}
Here, $d_1=\delta_{s\ind{1}}$, so that the initial state \(s_1\) is
\(s\ind{1}\) deterministically. The transitions \(P\sups{M}_{h, k}\)
and rewards \(R\sups{M}_{h, k}\) are constructed as follows.
\begin{enumerate}[label=\(\bullet\)]
\item Construction of \(\bbM\sups{M}_{1}\).
\begin{enumerate}[label=(\roman*)]
\item For all \(h \leq H\), the dynamics \(P\sups{M}_{h, k}\) are
  deterministic. For an action \(a_h\) in the state \(s_h\), the next state \(s_{h+1}\) is  
\begin{align*}
s_{h+1}  = \begin{cases} s\ind{h+1}, & \text{if $h \leq \bK$, $s_h = s
    \ind{h}$, and $a_h = a\sups{M}_i$}, \\ 
s\ind{\bK+1},\quad & \text{if $h \leq \bK$, $s_h = s \ind{h}$, and $a_h \neq a\sups{M}_i$}, \\ 
s_h,\quad &\text{otherwise}.
\end{cases}
\end{align*}

\item The reward distribution is given by 
\begin{align*}
R\sups{M}_{h, k}(s_h, a_h) = \begin{cases}
		\Ber\prn[\big]{\tfrac{1}{2} + \Delta}, & \text{if} \quad h = \bK, s_h = s\ind{\bK}, ~\text{and}~ a_h = a\sups{M}_{\bK}, \\
		\Ber\prn[\big]{\tfrac{1}{2}}, & \text{if} \quad h = \bK, s_h = s\ind{\bK}, ~\text{and}~ a_h \neq a\sups{M}_{\bK}, \\
		0, & \text{otherwise}.
\end{cases}
\end{align*} 
\end{enumerate} 
\item Construction of \(\bbM\sups{M}_{j}\) for \(2 \leq j \leq \bK\). 
\begin{enumerate}[label=(\roman*)]
\item For each \(h \leq H\), the dynamics \(P\sups{M}_{h, k}\) are
  deterministic. For action \(a_h\) in state \(s_h\), the next state \(s_{h+1}\) is   
\begin{align*}
s_{h+1}  = \begin{cases} 
s\ind{h + 1} &\text{if} \quad s_h = s\ind{h}~ \text{and}~ h < j  \\ 
s\ind{\bK+1} &\text{if} \quad s_h = s\ind{h}, h = j ~\text{and}~ a_h = a\sups{M}_{h} \\ 
s\ind{h+1} & \text{if} \quad s_h = s\ind{h},  h = j ~ \text{and}~ a_h \neq a\sups{M}_{h} \\ 
s\ind{h+1} &\text{if} \quad s_h = s\ind{h}, h > j ~ \text{and} ~ a_h = a\sups{M}_{h} \\ 
s\ind{\bK+1} &\text{if} \quad  s_h = s\ind{h}, h > j ~ \text{and} ~ a_h \neq a\sups{M}_{h} \\ 
s\ind{\bK+1} &\text{if} \quad h = \bK-1 ~\text{or}~ h = \bK \\
s_h & \text{otherwise} 
\end{cases} .
\end{align*} 

\item The reward distribution is given by 
\begin{align*}
R\sups{M}_{h, k}(s_h, a_h) = \begin{cases}
		\Ber\prn[\big]{\tfrac{1}{2}}, & \text{if} \quad h = \bK,\\	
		0, & \text{otherwise}.
\end{cases}  
\end{align*} 
\end{enumerate} 
\end{enumerate} 
Each model \(M\in\cM'\) is a uniform mixture of \(\bK\) MDPs
$\crl{\bbM\sups{M}_{1}, \dots, \bbM\sups{M}_{\bK}}$ as described above,
parameterized by the action sequence \(a\sups{M}_{1:\bK}\). The model
class \(\cM'\) is defined as the set of all such mixture models (one
for each sequence in $\cA^{\bK}$, so that \(\abs{\cM'} = A^{\bK}\). 

At the start of each episode, an MDP \(\bbM\sups{M}_{z}\) is chosen by sampling \(z \sim \text{Unif}([\bK])\). The trajectory is then drawn by setting \(s_1 = s\ind{1}\), and for \(h = 1, \dots, H\): 
\begin{itemize}
\item \(a_h = \pi_h(s_h)\).
\item \(r_h \sim R\sups{M}_{h, z}(s_h, a_h)\) and \(s_{h+1} \sim P_{h, z}\sups{M}(\cdot \mid s_h, a_h)\). 
\end{itemize} 
Note that rewards can be non-zero only at layer $h=\bK$. We receive a reward from \(\Ber\prn[\big]{\tfrac{1}{2} + \Delta}\) only when \(z = 1\) and the first \(\bK\) actions match \(a\sups{M}_{1:\bK}\), i.e. \(a_{1:\bK} = a\sups{M}_{1:\bK}\). For every other action sequence, the reward is sampled from  \(\Ber\prn[\big]{\tfrac{1}{2}}\). %
Thus, for any policy \(\pi\),  
\begin{align*}
f\sups{M}(\pi) = \tfrac{1}{2} + \Delta \indic\crl{\pi(s_{1:\bK}) = a\sups{M}_{1:\bK}},
\end{align*}
which implies that
\begin{align}
f\sups{M}(\pi\subs{M}) - f\sups{M}(\pi) &= \Delta(1 - \indic\crl{\pi(s_{1:\bK}) = a\sups{M}_{1:\bK}}). \label{eq:lb_mixture_tab3}
\end{align}

Finally, we define the reference model \(\Mbar\). The model \(\Mbar\)
is specified by \(\crl[\big]{\crl*{\cS_h}_{h=1}^{H+1}, \cA,
  \bbM\sups{\Mbar}}\) where \(\bbM\sups{\Mbar}\) is a tabular MDP
given by
\begin{align*}
\bbM\sups{\Mbar} = \crl[\big]{\crl*{\cS_h}_{h=1}^{H+1}, \cA,
    P\sups{\Mbar}_{h}, R_{h}\sups{\Mbar}, \delta_{s\ind{1}}}.
\end{align*}
Here, the initial state \(s_1\) is \(s\ind{1}\) deterministically, and
the transitions \(P\sups{\Mbar}_{h, k}\) and rewards
\(R\sups{\Mbar}_{h, k}\) are as follows:
\begin{enumerate}[label=(\roman*)]
\item Transitions are stochastic and independent of the chosen
  action. In particular, for each \(h \leq H\), the dynamics
  \(P\sups{\Mbar}_{h}\) are given by
\begin{align*}
 P\sups{\Mbar}_{h}\prn{s_{h+1} \mid s_{h}, a_h} = \begin{cases}
	 \frac{\bK - h}{\bK - h + 1} &\text{if} \quad h \leq \bK, s_h = s\ind{h}~\text{and}~s_{h+1} = s\ind{h+1} \\
	  \frac{1}{\bK - h + 1} & \text{if} \quad h \leq \bK, s_h = s\ind{h}~\text{and}~s_{h+1} = s\ind{\bK + 1} \\
	  1 & \text{if} \quad h \leq \bK, s_h \neq s\ind{h}~\text{and}~s_h = s_{h+1} \\
	  1 & \text{if} \quad h > \bK~\text{and}~s_h = s_{h+1} \\
	  0 &\text{otherwise} 
\end{cases}. 
\end{align*}
\item The reward distribution is given by 
\begin{align*}
R\sups{\Mbar}_{h}(s_h, a_h) = \begin{cases}
		\Ber\prn[\big]{\tfrac{1}{2}}, & \text{if} \quad h = \bK, \\	
		0, & \text{otherwise}.
\end{cases}.  
\end{align*} 
\end{enumerate}
Note that \(\Mbar\) can be thought of as a mixture of \(\bK\)
identical tabular MDPs each given by \(\bbM\sups{\Mbar}\). Note that for any policy \(\pi\), the rewards for any trajectory in \(\Mbar\) are sampled from \(\Ber\prn[\big]{\tfrac{1}{2}}\), and thus \(f\sups{\Mbar}(\pi) =  \frac{1}{2}\) which implies that  
\begin{align}
f\sups{\Mbar}(\pi\subs{\Mbar}) - f\sups{\Mbar}(\pi) &= 0. \label{eq:lb_mixture_tab4}
\end{align}

We define \(\cM'' = \cM' \cup \crl{\Mbar}\subseteq\cM\), and note that for any policy \(\pi\), the distribution over the trajectories is identical in all mixture models in \(\cM''\). However, as mentioned before, the rewards in \(\Mbar\) are sampled from  $\Ber\prn*{\frac{1}{2}}$ and for any \(M \in \cM'\), the rewards in \(M\) are sampled from $\Ber\prn[\big]{\frac{1}{2} + \frac{\Delta}{M} \indic\crl*{\pi(s_{1:\bK}) = a\sups{M}_{1:\bK}}}$. Thus, for any policy \(\pi\) and \(M \in \cM'\), 
\begin{align}
\Dhels{M(\pi)}{\Mbar(\pi)} &= \Dhels{\Ber\prn[\Big]{\tfrac{1}{2} 
+ \tfrac{\Delta}{\bK} \indic\crl*{\pi(s_{1:\bK}) = a\sups{M}_{1:\bK}}}}{\Ber\prn[\Big]{\tfrac{1}{2}}}\notag  \\  &\leq 3\frac{\Delta^2}{{\bK}^2}\cdot\indic\crl{\pi(s_{1:\bK}) = a\sups{M}_{1:\bK}}, \label{eq:lb_mixture_tab1}
\end{align} where the last line uses \pref{lem:helg_bern}.

The bounds in \pref{eq:lb_mixture_tab3}, \pref{eq:lb_mixture_tab4} and
\pref{eq:lb_mixture_tab1} together imply that the model class \(\cM''\)
is a $(\frac{\Delta}{\bK}, 3\frac{\Delta^2}{{\bK}^2}, 0)$-family in
the sense of \pref{def:hard_family_lb}, where for each \(\pi \in \Pi\)
and\(M \in \cM''\) we take
\begin{align*} 
u\sups{M}(\pi) \ldef \indic\crl{\pi(s_{1:\bK})=a\sups{M}_{1:\bK}} \qquad \text{and} \qquad v\sups{M}(\pi) \ldef  \indic\crl{\pi(s_{1:\bK})=a\sups{M}_{1:\bK}}, 
\end{align*}
with \(u\sups{\Mbar}(\pi) \ldef 1\) and \(v\sups{\Mbar}(\pi) \ldef 0\). As a result, \pref{lem:hard_family_lb} implies that 
\begin{align*}
  \comp(\cM,\Mbar) \geq \frac{\Delta}{2\bK} - \frac{3 \gamma \Delta^2}{{\bK}^2N}, 
\end{align*}
for \(N \ldef A^{\bK} + 1\). Setting \(\Delta = \frac{\bK N }{12
  \gamma}\) leads to the lower bound \(  \comp(\cM,\Mbar) \geq
\frac{N}{24 \gamma}\). We conclude by noting that all $M\in\cM''$ have
\(M \in \cM_{\veps_\gamma}\prn*{\Mbar}\) with \(\veps_\gamma =
\frac{N}{24 \gamma}\), and thus the lower bound on the \CompShort also
applies to the class \(\cM_{\veps_\gamma}\prn*{\Mbar}\). 
 \end{proof}

\begin{proof}[Proof for \preftitle{ex:tabular}] let \(\cM\) be the class of
  all tabular MDPs, and let \(\cM\ind{K}\) denote the set of all
  mixture models in which each \(M \in \cM\ind{K}
\) is a mixture of \(K\) MDPs from \(\cM\). Additionally, define \(\wt{\cM} = \conv(\cM)\), and note that \(\cM\ind{K} \subseteq \wt{\cM}\) for all \(K \geq 1\). For any \(\veps > 0\) and \(\Mbar \in \cM\ind{K}\), we have that \(\cM\ind{k}_{\veps}(\Mbar) \subseteq \wt{\cM}_{\veps}(\Mbar)\), which implies that
\begin{align*}
\comp\prn{\wt{\cM}_{\veps}(\Mbar), \Mbar} \geq \comp(\cM\ind{K}_{\veps}(\Mbar), \Mbar),
\end{align*} because
\(\comp(\cdot, \Mbar)\) is a non-decreasing function with respect to
inclusion. Using \pref{lem:tabular}, it holds that for any \(K \geq 1\)
and \(\gamma \geq  A^
{\min\crl{S - 1, H, K}}/6\), with \(\veps_\gamma \ldef {A^
{\min\crl{S - 1, H, K}}}/{24 \gamma}\), 
\begin{align*}
\comp\prn{\wt{\cM}_{\veps}(\Mbar), \Mbar} \geq \comp(\cM\ind{K}_{\veps}(\Mbar), \Mbar) \geq \frac{A^
{\min\crl{S - 1, H, K}}}{24 \gamma}.
\end{align*}
Setting \(K = S\) above gives the desired lower bound. 
\end{proof}

 }

\end{document}